\documentclass[a4paper,fleqn]{cas-sc}

\usepackage[authoryear,longnamesfirst]{natbib}

\usepackage{soul}
\usepackage{url}
\usepackage[utf8]{inputenc}
\usepackage[small]{caption}
\usepackage{graphicx}
\usepackage{amsmath}
\usepackage{amssymb}
\usepackage{amsthm}
\usepackage{booktabs}
\usepackage{algorithm}
\usepackage{aliascnt}
\usepackage{xspace}
\usepackage{cleveref}
\usepackage{xcolor}
\usepackage{esvect}
\usepackage[switch]{lineno}
\usepackage[appendix=inline,bibliography=common]{apxproof}
\usepackage{tikz}
\usetikzlibrary{arrows.meta}
\usepackage{multirow}

\usepackage[noend]{algpseudocode}

\allowdisplaybreaks
\newcommand{\FARBE}[1]{#1}
\newcommand{\ol}[1]{\overline{#1}}

\renewcommand{\leq}{\leqslant}

\renewcommand{\geq}{\geqslant}

\newcommand{\cspR}{\ensuremath{\mathit{CR}(\R)}}

\newcommand{\cspOf}[1]{\ensuremath{\mathit{CR}(#1)}}

\newcommand{\cspMIof}[2]{\ensuremath{\mathit{CR}^{#2}(#1)}} %
\newcommand{\solutionsR}{\ensuremath{\solutionsOfhelp{\cspR}}}
\newcommand{\solutionsOf}[1]{\ensuremath{\mathit{Sol}(#1)}}
\newcommand{\solutionsOfhelp}[1]{\ensuremath{\mathit{Sol}(#1)}}

\newcommand{\OCFsolutionsOf}[1]{\ensuremath{\mathit{Sol_{OCF}}(#1)}}

\newcommand{\OCFsolutionsRnMI}[1]{\ensuremath{\OCFsolutionsOf{\cspMIof{\R_n}{n-1}}}}

\newcommand{\induzierteOCF}[1]{\ensuremath{\kappa_{\!#1}}} %
\newcommand{\syntheticKB}[1]{\texttt{kb\_synth<\(n\)>\_c<\(2n\!\!-\!\!1\)>.pl}}

\newcommand{\cL}{\ensuremath{\mathcal L}}
\newcommand{\R}{\ensuremath{\mathcal R}}

\newcommand{\condAB}{(B | A)}

\newcommand{\notA}{\overline{A}}
\newcommand{\AnotB}{A \overline{B}}
\newcommand{\notB}{\overline{B}}
\newcommand{\condL}{\mbox{$(\cL \mid \cL)$}}

\newcommand{\naturals}{\mathbb{N}}

\newcommand{\Mod}{\mbox{\it Mod}\,}

\newcommand{\kappaiminus}[1]{\eta_{#1}}

\newcommand{\kappaiminusVektor}{\ensuremath{\vv{\eta}}}

\makeatletter
\ifx\centernot\undefined
\newcommand*{\centernot}{%
	\mathpalette\@centernot
}
\fi
\def\@centernot#1#2{%
	\mathrel{%
		\rlap{%
			\settowidth\dimen@{$\m@th#1{#2}$}%
			\kern.5\dimen@
			\settowidth\dimen@{$\m@th#1=$}%
			\kern-.5\dimen@
			$\m@th#1\not$%
		}%
		{#2}%
	}%
}
\makeatother
\DeclareRobustCommand\nmableitSymb{\mathrel|\mkern-.5mu\joinrel\sim} %
\newcommand{\nmableit}{\ensuremath{\mbox{$\,\nmableitSymb\,$}}} %
\newcommand{\notnmableit}{\ensuremath{\mbox{$\,\centernot\nmableitSymb\,$}}} %

\newcommand{\beweisendezeichen}%
{\penalty50\hspace*{0pt plus 1fil}\parfillskip=0pt\mbox{$\Box$}}

\newcommand{\fussnoteOhneMarkierung}[1]%
{%
\footnote{#1}%
\addtocounter{footnote}{-1}%
}

\newcommand{\satzCL}[2]{\ensuremath{(#1|#2)}}

\newlength{\abstand}
\newcommand{\symbolForInfSkept}{\ensuremath{\mathit{sk}}}

\newcommand{\symbolInd}{\ensuremath{\mathit{\mathit{O}}}}

\newcommand{\kbableitModMinName}[3]{\ensuremath{\nmableit^{\!\!\!\!#2,#3}_{\!#1}}}

\newcommand{\indinf}[3]{\ensuremath{\kbableitModMinName{\R}{\symbolForInfSkept}{\symbolInd}}}

\newcommand{\kb}{\ensuremath{\R}}

\newcommand{\perpendicular}%
{\begin{sideways}$\models$\end{sideways}}

\newcommand{\verify}{\ensuremath{\mathit{v}}}
\newcommand{\falsif}{\ensuremath{\mathit{f}}}
\newcommand{\neutra}{\ensuremath{\mathit{-}}}

\newcommand{\condLabel}[1]{\ensuremath{#1\!\!:}}

\newcommand{\druckeWeltVier}[4]{\ensuremath{#1\,#2\,#3\,#4}}

\newcommand{\omegaAvier}[4]{\ensuremath{\druckeWeltVier{#1}{#2}{#3}{#4}}}
\newcommand{\omegaBvier}[4]{\ensuremath{\druckeWeltVier{#1}{#2}{#3}{\ol{#4}}}}
\newcommand{\omegaCvier}[4]{\ensuremath{\druckeWeltVier{#1}{#2}{\ol{#3}}{#4}}}
\newcommand{\omegaDvier}[4]{\ensuremath{\druckeWeltVier{#1}{#2}{\ol{#3}}{\ol{#4}}}}
\newcommand{\omegaEvier}[4]{\ensuremath{\druckeWeltVier{#1}{\ol{#2}}{#3}{#4}}}
\newcommand{\omegaFvier}[4]{\ensuremath{\druckeWeltVier{#1}{\ol{#2}}{#3}{\ol{#4}}}}
\newcommand{\omegaGvier}[4]{\ensuremath{\druckeWeltVier{#1}{\ol{#2}}{\ol{#3}}{#4}}}
\newcommand{\omegaHvier}[4]{\ensuremath{\druckeWeltVier{#1}{\ol{#2}}{\ol{#3}}{\ol{#4}}}}

\newcommand{\omegaIvier}[4]{\ensuremath{\druckeWeltVier{\ol{#1}}{#2}{#3}{#4}}}
\newcommand{\omegaJvier}[4]{\ensuremath{\druckeWeltVier{\ol{#1}}{#2}{#3}{\ol{#4}}}}
\newcommand{\omegaKvier}[4]{\ensuremath{\druckeWeltVier{\ol{#1}}{#2}{\ol{#3}}{#4}}}
\newcommand{\omegaLvier}[4]{\ensuremath{\druckeWeltVier{\ol{#1}}{#2}{\ol{#3}}{\ol{#4}}}}
\newcommand{\omegaMvier}[4]{\ensuremath{\druckeWeltVier{\ol{#1}}{\ol{#2}}{#3}{#4}}}
\newcommand{\omegaNvier}[4]{\ensuremath{\druckeWeltVier{\ol{#1}}{\ol{#2}}{#3}{\ol{#4}}}}
\newcommand{\omegaOvier}[4]{\ensuremath{\druckeWeltVier{\ol{#1}}{\ol{#2}}{\ol{#3}}{#4}}}
\newcommand{\omegaPvier}[4]{\ensuremath{\druckeWeltVier{\ol{#1}}{\ol{#2}}{\ol{#3}}{\ol{#4}}}}

\newcommand{\omegaAbpfw}{\ensuremath{\omegaAvier{b}{p}{f}{w}}}
\newcommand{\omegaBbpfw}{\ensuremath{\omegaBvier{b}{p}{f}{w}}}
\newcommand{\omegaCbpfw}{\ensuremath{\omegaCvier{b}{p}{f}{w}}}
\newcommand{\omegaDbpfw}{\ensuremath{\omegaDvier{b}{p}{f}{w}}}
\newcommand{\omegaEbpfw}{\ensuremath{\omegaEvier{b}{p}{f}{w}}}
\newcommand{\omegaFbpfw}{\ensuremath{\omegaFvier{b}{p}{f}{w}}}
\newcommand{\omegaGbpfw}{\ensuremath{\omegaGvier{b}{p}{f}{w}}}
\newcommand{\omegaHbpfw}{\ensuremath{\omegaHvier{b}{p}{f}{w}}}

\newcommand{\omegaIbpfw}{\ensuremath{\omegaIvier{b}{p}{f}{w}}}
\newcommand{\omegaJbpfw}{\ensuremath{\omegaJvier{b}{p}{f}{w}}}
\newcommand{\omegaKbpfw}{\ensuremath{\omegaKvier{b}{p}{f}{w}}}
\newcommand{\omegaLbpfw}{\ensuremath{\omegaLvier{b}{p}{f}{w}}}
\newcommand{\omegaMbpfw}{\ensuremath{\omegaMvier{b}{p}{f}{w}}}
\newcommand{\omegaNbpfw}{\ensuremath{\omegaNvier{b}{p}{f}{w}}}
\newcommand{\omegaObpfw}{\ensuremath{\omegaOvier{b}{p}{f}{w}}}
\newcommand{\omegaPbpfw}{\ensuremath{\omegaPvier{b}{p}{f}{w}}}

\newlength{\spalteAbst}
\newlength{\spalteAbstGr}
\newlength{\spalteAbstGGr}
\newlength{\zeileAbst}
\newcommand{\ALvariableE}{\ensuremath{a}}
\newcommand{\ALvariableZ}{\ensuremath{b}}
\newcommand{\ALvariableD}{\ensuremath{c}}
\newcommand{\ALvariableV}{\ensuremath{d}}
\newcommand{\ALvariableF}{\ensuremath{e}}

\newlength{\AbstandZwischenLiteralen}
\newcommand{\setzeAbstandZwischenLiteralen}{\hspace{\AbstandZwischenLiteralen}}

\newcommand{\omegaFuenf}[1]{%
    \IfEqCase{#1}{%
        {01}{\ensuremath{\omegaAafuenf{\ALvariableE}{\ALvariableZ}{\ALvariableD}{\ALvariableV}{\ALvariableF}}}%
        {02}{\ensuremath{\omegaAbfuenf{\ALvariableE}{\ALvariableZ}{\ALvariableD}{\ALvariableV}{\ALvariableF}}}%
        {03}{\ensuremath{\omegaBafuenf{\ALvariableE}{\ALvariableZ}{\ALvariableD}{\ALvariableV}{\ALvariableF}}}%
        {04}{\ensuremath{\omegaBbfuenf{\ALvariableE}{\ALvariableZ}{\ALvariableD}{\ALvariableV}{\ALvariableF}}}%
        {05}{\ensuremath{\omegaCafuenf{\ALvariableE}{\ALvariableZ}{\ALvariableD}{\ALvariableV}{\ALvariableF}}}%
        {06}{\ensuremath{\omegaCbfuenf{\ALvariableE}{\ALvariableZ}{\ALvariableD}{\ALvariableV}{\ALvariableF}}}%
        {07}{\ensuremath{\omegaDafuenf{\ALvariableE}{\ALvariableZ}{\ALvariableD}{\ALvariableV}{\ALvariableF}}}%
        {08}{\ensuremath{\omegaDbfuenf{\ALvariableE}{\ALvariableZ}{\ALvariableD}{\ALvariableV}{\ALvariableF}}}%
        {09}{\ensuremath{\omegaEafuenf{\ALvariableE}{\ALvariableZ}{\ALvariableD}{\ALvariableV}{\ALvariableF}}}%
        {10}{\ensuremath{\omegaEbfuenf{\ALvariableE}{\ALvariableZ}{\ALvariableD}{\ALvariableV}{\ALvariableF}}}%
        {11}{\ensuremath{\omegaFafuenf{\ALvariableE}{\ALvariableZ}{\ALvariableD}{\ALvariableV}{\ALvariableF}}}%
        {12}{\ensuremath{\omegaFbfuenf{\ALvariableE}{\ALvariableZ}{\ALvariableD}{\ALvariableV}{\ALvariableF}}}%
        {13}{\ensuremath{\omegaGafuenf{\ALvariableE}{\ALvariableZ}{\ALvariableD}{\ALvariableV}{\ALvariableF}}}%
        {14}{\ensuremath{\omegaGbfuenf{\ALvariableE}{\ALvariableZ}{\ALvariableD}{\ALvariableV}{\ALvariableF}}}%
        {15}{\ensuremath{\omegaHafuenf{\ALvariableE}{\ALvariableZ}{\ALvariableD}{\ALvariableV}{\ALvariableF}}}%
        {16}{\ensuremath{\omegaHbfuenf{\ALvariableE}{\ALvariableZ}{\ALvariableD}{\ALvariableV}{\ALvariableF}}}%
        {17}{\ensuremath{\omegaIafuenf{\ALvariableE}{\ALvariableZ}{\ALvariableD}{\ALvariableV}{\ALvariableF}}}%
        {18}{\ensuremath{\omegaIbfuenf{\ALvariableE}{\ALvariableZ}{\ALvariableD}{\ALvariableV}{\ALvariableF}}}%
        {19}{\ensuremath{\omegaJafuenf{\ALvariableE}{\ALvariableZ}{\ALvariableD}{\ALvariableV}{\ALvariableF}}}%
        {20}{\ensuremath{\omegaJbfuenf{\ALvariableE}{\ALvariableZ}{\ALvariableD}{\ALvariableV}{\ALvariableF}}}%
        {21}{\ensuremath{\omegaKafuenf{\ALvariableE}{\ALvariableZ}{\ALvariableD}{\ALvariableV}{\ALvariableF}}}%
        {22}{\ensuremath{\omegaKbfuenf{\ALvariableE}{\ALvariableZ}{\ALvariableD}{\ALvariableV}{\ALvariableF}}}%
        {23}{\ensuremath{\omegaLafuenf{\ALvariableE}{\ALvariableZ}{\ALvariableD}{\ALvariableV}{\ALvariableF}}}%
        {24}{\ensuremath{\omegaLbfuenf{\ALvariableE}{\ALvariableZ}{\ALvariableD}{\ALvariableV}{\ALvariableF}}}%
        {25}{\ensuremath{\omegaMafuenf{\ALvariableE}{\ALvariableZ}{\ALvariableD}{\ALvariableV}{\ALvariableF}}}%
        {26}{\ensuremath{\omegaMbfuenf{\ALvariableE}{\ALvariableZ}{\ALvariableD}{\ALvariableV}{\ALvariableF}}}%
        {27}{\ensuremath{\omegaNafuenf{\ALvariableE}{\ALvariableZ}{\ALvariableD}{\ALvariableV}{\ALvariableF}}}%
        {28}{\ensuremath{\omegaNbfuenf{\ALvariableE}{\ALvariableZ}{\ALvariableD}{\ALvariableV}{\ALvariableF}}}%
        {29}{\ensuremath{\omegaOafuenf{\ALvariableE}{\ALvariableZ}{\ALvariableD}{\ALvariableV}{\ALvariableF}}}%
        {30}{\ensuremath{\omegaObfuenf{\ALvariableE}{\ALvariableZ}{\ALvariableD}{\ALvariableV}{\ALvariableF}}}%
        {31}{\ensuremath{\omegaPafuenf{\ALvariableE}{\ALvariableZ}{\ALvariableD}{\ALvariableV}{\ALvariableF}}}%
        {32}{\ensuremath{\omegaPbfuenf{\ALvariableE}{\ALvariableZ}{\ALvariableD}{\ALvariableV}{\ALvariableF}}}%
    }[\PackageError{omegaFuenf}{Undefined option to omegaFuenf: #1}{}]%
}

\newcommand{\druckeWeltFuenf}[5]{\ensuremath{\mathit{#1\setzeAbstandZwischenLiteralen#2\setzeAbstandZwischenLiteralen#3\setzeAbstandZwischenLiteralen#4\setzeAbstandZwischenLiteralen#5}}} %

\newcommand{\omegaAafuenf}[5]{\ensuremath{\druckeWeltFuenf{#1}{#2}{#3}{#4}{#5}}}
\newcommand{\omegaAbfuenf}[5]{\ensuremath{\druckeWeltFuenf{#1}{#2}{#3}{#4}{\ol{#5}}}}
\newcommand{\omegaBafuenf}[5]{\ensuremath{\druckeWeltFuenf{#1}{#2}{#3}{\ol{#4}}{#5}}}
\newcommand{\omegaBbfuenf}[5]{\ensuremath{\druckeWeltFuenf{#1}{#2}{#3}{\ol{#4}}{\ol{#5}}}}
\newcommand{\omegaCafuenf}[5]{\ensuremath{\druckeWeltFuenf{#1}{#2}{\ol{#3}}{#4}{#5}}}
\newcommand{\omegaCbfuenf}[5]{\ensuremath{\druckeWeltFuenf{#1}{#2}{\ol{#3}}{#4}{\ol{#5}}}}
\newcommand{\omegaDafuenf}[5]{\ensuremath{\druckeWeltFuenf{#1}{#2}{\ol{#3}}{\ol{#4}}{#5}}}
\newcommand{\omegaDbfuenf}[5]{\ensuremath{\druckeWeltFuenf{#1}{#2}{\ol{#3}}{\ol{#4}}{\ol{#5}}}}

\newcommand{\omegaEafuenf}[5]{\ensuremath{\druckeWeltFuenf{#1}{\ol{#2}}{#3}{#4}{#5}}}
\newcommand{\omegaEbfuenf}[5]{\ensuremath{\druckeWeltFuenf{#1}{\ol{#2}}{#3}{#4}{\ol{#5}}}}
\newcommand{\omegaFafuenf}[5]{\ensuremath{\druckeWeltFuenf{#1}{\ol{#2}}{#3}{\ol{#4}}{#5}}}
\newcommand{\omegaFbfuenf}[5]{\ensuremath{\druckeWeltFuenf{#1}{\ol{#2}}{#3}{\ol{#4}}{\ol{#5}}}}
\newcommand{\omegaGafuenf}[5]{\ensuremath{\druckeWeltFuenf{#1}{\ol{#2}}{\ol{#3}}{#4}{#5}}}
\newcommand{\omegaGbfuenf}[5]{\ensuremath{\druckeWeltFuenf{#1}{\ol{#2}}{\ol{#3}}{#4}{\ol{#5}}}}
\newcommand{\omegaHafuenf}[5]{\ensuremath{\druckeWeltFuenf{#1}{\ol{#2}}{\ol{#3}}{\ol{#4}}{#5}}}
\newcommand{\omegaHbfuenf}[5]{\ensuremath{\druckeWeltFuenf{#1}{\ol{#2}}{\ol{#3}}{\ol{#4}}{\ol{#5}}}}

\newcommand{\omegaIafuenf}[5]{\ensuremath{\druckeWeltFuenf{\ol{#1}}{#2}{#3}{#4}{#5}}}
\newcommand{\omegaIbfuenf}[5]{\ensuremath{\druckeWeltFuenf{\ol{#1}}{#2}{#3}{#4}{\ol{#5}}}}
\newcommand{\omegaJafuenf}[5]{\ensuremath{\druckeWeltFuenf{\ol{#1}}{#2}{#3}{\ol{#4}}{#5}}}
\newcommand{\omegaJbfuenf}[5]{\ensuremath{\druckeWeltFuenf{\ol{#1}}{#2}{#3}{\ol{#4}}{\ol{#5}}}}
\newcommand{\omegaKafuenf}[5]{\ensuremath{\druckeWeltFuenf{\ol{#1}}{#2}{\ol{#3}}{#4}{#5}}}
\newcommand{\omegaKbfuenf}[5]{\ensuremath{\druckeWeltFuenf{\ol{#1}}{#2}{\ol{#3}}{#4}{\ol{#5}}}}
\newcommand{\omegaLafuenf}[5]{\ensuremath{\druckeWeltFuenf{\ol{#1}}{#2}{\ol{#3}}{\ol{#4}}{#5}}}
\newcommand{\omegaLbfuenf}[5]{\ensuremath{\druckeWeltFuenf{\ol{#1}}{#2}{\ol{#3}}{\ol{#4}}{\ol{#5}}}}

\newcommand{\omegaMafuenf}[5]{\ensuremath{\druckeWeltFuenf{\ol{#1}}{\ol{#2}}{#3}{#4}{#5}}}
\newcommand{\omegaMbfuenf}[5]{\ensuremath{\druckeWeltFuenf{\ol{#1}}{\ol{#2}}{#3}{#4}{\ol{#5}}}}
\newcommand{\omegaNafuenf}[5]{\ensuremath{\druckeWeltFuenf{\ol{#1}}{\ol{#2}}{#3}{\ol{#4}}{#5}}}
\newcommand{\omegaNbfuenf}[5]{\ensuremath{\druckeWeltFuenf{\ol{#1}}{\ol{#2}}{#3}{\ol{#4}}{\ol{#5}}}}
\newcommand{\omegaOafuenf}[5]{\ensuremath{\druckeWeltFuenf{\ol{#1}}{\ol{#2}}{\ol{#3}}{#4}{#5}}}
\newcommand{\omegaObfuenf}[5]{\ensuremath{\druckeWeltFuenf{\ol{#1}}{\ol{#2}}{\ol{#3}}{#4}{\ol{#5}}}}
\newcommand{\omegaPafuenf}[5]{\ensuremath{\druckeWeltFuenf{\ol{#1}}{\ol{#2}}{\ol{#3}}{\ol{#4}}{#5}}}
\newcommand{\omegaPbfuenf}[5]{\ensuremath{\druckeWeltFuenf{\ol{#1}}{\ol{#2}}{\ol{#3}}{\ol{#4}}{\ol{#5}}}}

\newcommand{\iiopnames}{inductive inference operators\xspace}   %

\newcommand{\iiopP}{\ensuremath{\nmableit^{\!p}}}   %

\newcommand{\iiopZ}{\ensuremath{\nmableit^{\!z}}}   %
\newcommand{\iiopSKc}{\ensuremath{\nmableit^{\!c}}}   %

\newcommand{\symbolsystemWableitung}{\ensuremath{\sf{w}}}
\newcommand{\iiopW}{\ensuremath{\nmableit^{\!\symbolsystemWableitung}}}   %

\newcommand{\symbolLexInfableitung}{\ensuremath{\mathit{lx}}}
\newcommand{\iiopL}{\ensuremath{\nmableit^{\!\symbolLexInfableitung}}}   %

\newcommand{\BB}{belief base\xspace}   %
\newcommand{\BBs}{belief bases\xspace}   %
\newcommand{\NFallKonditionale}[1]{\ensuremath{\mathit{NFC}(#1)}}

\newcommand{\closureMatrixName}[1]{\ensuremath{\mathit{CM}}} %

\newcommand{\listL}[1]{\ensuremath{\lbrack\mathcal{L}_{\Sigma}\rbrack}} %
\newcommand{\listNFC}[1]{\ensuremath{\lbrack\NFallKonditionale{\Sigma}\rbrack}} %

\makeatletter
\newcommand{\preceqdotLOKAL}{\mathrel{\mathpalette\LOKALpr@ceqd@t\relax}}
\newcommand{\LOKALpr@ceqd@t}[2]{%
  \begingroup
  \sbox\z@{$#1\prec$}\sbox\tw@{$#1\preccurlyeq$}%
  \dimen@=0.5\dimexpr\ht\tw@-\ht\z@\relax
  {\preccurlyeq}%
  \mkern-4mu
  \raisebox{\dimen@}{$\m@th#1\cdot$}%
  \endgroup
}
\makeatother

\makeatletter
\newcommand{\preceqprecdotLOKAL}{\mathrel{\mathpalette\LOKALprpr@ceqd@t\relax}}
\newcommand{\LOKALprpr@ceqd@t}[2]{%
  \begingroup
  \sbox\z@{$#1\prec$}\sbox\tw@{$#1\preccurlyeq$}%
  \dimen@=0.5\dimexpr\ht\tw@-\ht\z@\relax
  {\preccurlyeq}%
  \mkern-9mu
  {\prec}%
  \mkern-4mu
  \raisebox{\dimen@}{$\m@th#1\cdot$}%
  \endgroup
}
\makeatother

\newcommand{\displayTransformationRule}[3]%
{#1 &  #2 & #3}

\newcommand{\displayTRrule}[3]%
{#1 &  #2 & #3}

\newcommand{\symbolSSVker}{\ensuremath{\textbf{R1}}\xspace}  %
\newcommand{\symbolSSFker}{\ensuremath{\textbf{R2}}\xspace}  %
\newcommand{\symbolELMker}{\ensuremath{\textbf{R3}}\xspace}  %
\newcommand{\symbolTRIVIAL}{\textrm{\textbf{R4}}\xspace}  %
\newcommand{\symbolSUBSETS}{\textrm{\textbf{R5}}\xspace}  %
\newcommand{\symbolCIRCLE}{\textrm{\textbf{R6}}\xspace}  %

\newcommand{\refTRIVIAL}{\ensuremath{\symbolTRIVIAL}\xspace}    %
\newcommand{\refSUBSETS}{\ensuremath{\symbolSUBSETS}\xspace}    %
\newcommand{\refCIRCLE}{\ensuremath{\symbolCIRCLE}\xspace}    %

\newcommand{\refSSVker}{\ensuremath{\symbolSSVker}\xspace}    %
\newcommand{\refSSFker}{\ensuremath{\symbolSSFker}\xspace}    %
\newcommand{\refELMker}{\ensuremath{\symbolELMker}\xspace}    %

\newcommand{\nameSSV}{\ensuremath{\textit{subset-V}}\xspace}  %
\newcommand{\nameSSF}{\ensuremath{\textit{subset-F}}\xspace}  %
\newcommand{\nameELM}{\ensuremath{\textit{element}}\xspace}  %
\newcommand{\nameTRIVIAL}{\ensuremath{\textit{trivial}}\xspace}  %
\newcommand{\nameSUBSETS}{\ensuremath{\textit{subsets}}\xspace}  %
\newcommand{\nameCIRCLE}{\ensuremath{\textit{circle}}\xspace}  %

\newcommand{\nsrPaarKER}[3]{\ensuremath{\langle #1,\; #2 \rangle_{#3}}}

{
     
    \begin{enumerate}
}
{
    \end{enumerate}
}

{

    \begin{enumerate}
}
{
    \end{enumerate}
}

\newcommand\satzCLab[4]%
{\ifthenelse{\equal{#1}{no}}{\ensuremath{\satzCL{\{#3\}}{\{#4\}}}}{\ensuremath{\hbox{\ensuremath{#1_{#2}\!\!:\!\,\satzCL{\{#3\}}{\{#4\}}}}}}}

\newcommand{\nfcab}[2][no]{%
    \IfEqCase{#2}{%
        {01}{\ensuremath{\satzCLab{#1}{#2}{3}{3,2}}}%
        {02}{\ensuremath{\satzCLab{#1}{#2}{2}{3,2}}}%
        {03}{\ensuremath{\satzCLab{#1}{#2}{3}{3,0}}}%
        {04}{\ensuremath{\satzCLab{#1}{#2}{0}{3,0}}}%
        {05}{\ensuremath{\satzCLab{#1}{#2}{2}{2,1}}}%
        {06}{\ensuremath{\satzCLab{#1}{#2}{2}{2,0}}}%
        {07}{\ensuremath{\satzCLab{#1}{#2}{0}{2,0}}}%
        {08}{\ensuremath{\satzCLab{#1}{#2}{3}{3,2,1}}}%
        {09}{\ensuremath{\satzCLab{#1}{#2}{2}{3,2,1}}}%
        {12}{\ensuremath{\cb{\satzCLab{#1}{#2}{3}{3,2,0}}}}%
        {13}{\ensuremath{\cb{\satzCLab{#1}{#2}{2}{3,2,0}}}}%
        {14}{\ensuremath{\cb{\satzCLab{#1}{#2}{0}{3,2,0}}}}%
        {18}{\ensuremath{\cb{\satzCLab{#1}{#2}{2}{2,1,0}}}}%
        {19}{\ensuremath{\cb{\satzCLab{#1}{#2}{0}{2,1,0}}}}%
        {10}{\ensuremath{\cb{\satzCLab{#1}{#2}{3,2}{3,2,1}}}}%
        {11}{\ensuremath{\cb{\satzCLab{#1}{#2}{2,1}{3,2,1}}}}%
        {15}{\ensuremath{\cb{\satzCLab{#1}{#2}{3,2}{3,2,0}}}}%
        {16}{\ensuremath{\cb{\satzCLab{#1}{#2}{3,0}{3,2,0}}}}%
        {17}{\ensuremath{\cb{\satzCLab{#1}{#2}{2,0}{3,2,0}}}}%
        {20}{\ensuremath{\satzCLab{#1}{#2}{2,1}{2,1,0}}}%
        {21}{\ensuremath{\satzCLab{#1}{#2}{2,0}{2,1,0}}}%
        {22}{\ensuremath{\satzCLab{#1}{#2}{3}{3,2,1,0}}}%
        {23}{\ensuremath{\satzCLab{#1}{#2}{2}{3,2,1,0}}}%
        {24}{\ensuremath{\satzCLab{#1}{#2}{0}{3,2,1,0}}}%
        {25}{\ensuremath{\satzCLab{#1}{#2}{3,2}{3,2,1,0}}}%
        {26}{\ensuremath{\satzCLab{#1}{#2}{3,0}{3,2,1,0}}}%
        {27}{\ensuremath{\satzCLab{#1}{#2}{2,1}{3,2,1,0}}}%
        {28}{\ensuremath{\satzCLab{#1}{#2}{2,0}{3,2,1,0}}}%
        {29}{\ensuremath{\satzCLab{#1}{#2}{3,2,1}{3,2,1,0}}}%
        {30}{\ensuremath{\satzCLab{#1}{#2}{3,2,0}{3,2,1,0}}}%
        {31}{\ensuremath{\satzCLab{#1}{#2}{2,1,0}{3,2,1,0}}}%
        {32}{\ensuremath{\satzCLab{#1}{#2}{3}{3,1}}}%
        {33}{\ensuremath{\satzCLab{#1}{#2}{1}{3,1}}}%
        {34}{\ensuremath{\satzCLab{#1}{#2}{1}{2,1}}}%
        {35}{\ensuremath{\satzCLab{#1}{#2}{1}{1,0}}}%
        {36}{\ensuremath{\satzCLab{#1}{#2}{0}{1,0}}}%
        {37}{\ensuremath{\satzCLab{#1}{#2}{1}{3,2,1}}}%
        {39}{\ensuremath{\cb{\satzCLab{#1}{#2}{3}{3,1,0}}}}%
        {40}{\ensuremath{\cb{\satzCLab{#1}{#2}{1}{3,1,0}}}}%
        {41}{\ensuremath{\cb{\satzCLab{#1}{#2}{0}{3,1,0}}}}%
        {45}{\ensuremath{\cb{\satzCLab{#1}{#2}{1}{2,1,0}}}}%
        {38}{\ensuremath{\cb{\satzCLab{#1}{#2}{3,1}{3,2,1}}}}%
        {42}{\ensuremath{\cb{\satzCLab{#1}{#2}{3,1}{3,1,0}}}}%
        {43}{\ensuremath{\cb{\satzCLab{#1}{#2}{3,0}{3,1,0}}}}%
        {44}{\ensuremath{\cb{\satzCLab{#1}{#2}{1,0}{3,1,0}}}}%
        {46}{\ensuremath{\satzCLab{#1}{#2}{1,0}{2,1,0}}}%
        {47}{\ensuremath{\satzCLab{#1}{#2}{1}{3,2,1,0}}}%
        {48}{\ensuremath{\satzCLab{#1}{#2}{3,1}{3,2,1,0}}}%
        {49}{\ensuremath{\satzCLab{#1}{#2}{1,0}{3,2,1,0}}}%
        {50}{\ensuremath{\satzCLab{#1}{#2}{3,1,0}{3,2,1,0}}}%
    }[\PackageError{nfcab}{Undefined option to nfcab: #1}{}]%
}%

\newcommand{\nfcabChar}[2][no]{%
    \IfEqCase{#2}{%
        {01}{\ensuremath{\satzCLab{#1}{#2}{ab}{ab,a\ol{b}}}}%
        {02}{\ensuremath{\satzCLab{#1}{#2}{a\ol{b}}{ab,a\ol{b}}}}%
        {03}{\ensuremath{\satzCLab{#1}{#2}{ab}{ab,\ol{a}\ol{b}}}}%
        {04}{\ensuremath{\satzCLab{#1}{#2}{\ol{a}\ol{b}}{ab,\ol{a}\ol{b}}}}%
        {05}{\ensuremath{\satzCLab{#1}{#2}{a\ol{b}}{a\ol{b},\ol{a}b}}}%
        {06}{\ensuremath{\satzCLab{#1}{#2}{a\ol{b}}{a\ol{b},\ol{a}\ol{b}}}}%
        {07}{\ensuremath{\satzCLab{#1}{#2}{\ol{a}\ol{b}}{a\ol{b},\ol{a}\ol{b}}}}%
        {08}{\ensuremath{\satzCLab{#1}{#2}{ab}{ab,a\ol{b},\ol{a}b}}}%
        {09}{\ensuremath{\satzCLab{#1}{#2}{a\ol{b}}{ab,a\ol{b},\ol{a}b}}}%
        {12}{\ensuremath{\satzCLab{#1}{#2}{ab}{ab,a\ol{b},\ol{a}\ol{b}}}}%
        {13}{\ensuremath{\satzCLab{#1}{#2}{a\ol{b}}{ab,a\ol{b},\ol{a}\ol{b}}}}%
        {14}{\ensuremath{\satzCLab{#1}{#2}{\ol{a}\ol{b}}{ab,a\ol{b},\ol{a}\ol{b}}}}%
        {18}{\ensuremath{\satzCLab{#1}{#2}{a\ol{b}}{a\ol{b},\ol{a}b,\ol{a}\ol{b}}}}%
        {19}{\ensuremath{\satzCLab{#1}{#2}{\ol{a}\ol{b}}{a\ol{b},\ol{a}b,\ol{a}\ol{b}}}}%
        {10}{\ensuremath{\satzCLab{#1}{#2}{ab,a\ol{b}}{ab,a\ol{b},\ol{a}b}}}%
        {11}{\ensuremath{\satzCLab{#1}{#2}{a\ol{b},\ol{a}b}{ab,a\ol{b},\ol{a}b}}}%
        {15}{\ensuremath{\satzCLab{#1}{#2}{ab,a\ol{b}}{ab,a\ol{b},\ol{a}\ol{b}}}}%
        {16}{\ensuremath{\satzCLab{#1}{#2}{ab,\ol{a}\ol{b}}{ab,a\ol{b},\ol{a}\ol{b}}}}%
        {17}{\ensuremath{\satzCLab{#1}{#2}{a\ol{b},\ol{a}\ol{b}}{ab,a\ol{b},\ol{a}\ol{b}}}}%
        {20}{\ensuremath{\satzCLab{#1}{#2}{a\ol{b},\ol{a}b}{a\ol{b},\ol{a}b,\ol{a}\ol{b}}}}%
        {21}{\ensuremath{\satzCLab{#1}{#2}{a\ol{b},\ol{a}\ol{b}}{a\ol{b},\ol{a}b,\ol{a}\ol{b}}}}%
        {22}{\ensuremath{\satzCLab{#1}{#2}{ab}{ab,a\ol{b},\ol{a}b,\ol{a}\ol{b}}}}%
        {23}{\ensuremath{\satzCLab{#1}{#2}{a\ol{b}}{ab,a\ol{b},\ol{a}b,\ol{a}\ol{b}}}}%
        {24}{\ensuremath{\satzCLab{#1}{#2}{\ol{a}\ol{b}}{ab,a\ol{b},\ol{a}b,\ol{a}\ol{b}}}}%
        {25}{\ensuremath{\satzCLab{#1}{#2}{ab,a\ol{b}}{ab,a\ol{b},\ol{a}b,\ol{a}\ol{b}}}}%
        {26}{\ensuremath{\satzCLab{#1}{#2}{ab,\ol{a}\ol{b}}{ab,a\ol{b},\ol{a}b,\ol{a}\ol{b}}}}%
        {27}{\ensuremath{\satzCLab{#1}{#2}{a\ol{b},\ol{a}b}{ab,a\ol{b},\ol{a}b,\ol{a}\ol{b}}}}%
        {28}{\ensuremath{\satzCLab{#1}{#2}{a\ol{b},\ol{a}\ol{b}}{ab,a\ol{b},\ol{a}b,\ol{a}\ol{b}}}}%
        {29}{\ensuremath{\satzCLab{#1}{#2}{ab,a\ol{b},\ol{a}b}{ab,a\ol{b},\ol{a}b,\ol{a}\ol{b}}}}%
        {30}{\ensuremath{\satzCLab{#1}{#2}{ab,a\ol{b},\ol{a}\ol{b}}{ab,a\ol{b},\ol{a}b,\ol{a}\ol{b}}}}%
        {31}{\ensuremath{\satzCLab{#1}{#2}{a\ol{b},\ol{a}b,\ol{a}\ol{b}}{ab,a\ol{b},\ol{a}b,\ol{a}\ol{b}}}}%
        {32}{\ensuremath{\satzCLab{#1}{#2}{ab}{ab,\ol{a}b}}}%
        {33}{\ensuremath{\satzCLab{#1}{#2}{\ol{a}b}{ab,\ol{a}b}}}%
        {34}{\ensuremath{\satzCLab{#1}{#2}{\ol{a}b}{a\ol{b},\ol{a}b}}}%
        {35}{\ensuremath{\satzCLab{#1}{#2}{\ol{a}b}{\ol{a}b,\ol{a}\ol{b}}}}%
        {36}{\ensuremath{\satzCLab{#1}{#2}{\ol{a}\ol{b}}{\ol{a}b,\ol{a}\ol{b}}}}%
        {37}{\ensuremath{\satzCLab{#1}{#2}{\ol{a}b}{ab,a\ol{b},\ol{a}b}}}%
        {39}{\ensuremath{\satzCLab{#1}{#2}{ab}{ab,\ol{a}b,\ol{a}\ol{b}}}}%
        {40}{\ensuremath{\satzCLab{#1}{#2}{\ol{a}b}{ab,\ol{a}b,\ol{a}\ol{b}}}}%
        {41}{\ensuremath{\satzCLab{#1}{#2}{\ol{a}\ol{b}}{ab,\ol{a}b,\ol{a}\ol{b}}}}%
        {45}{\ensuremath{\satzCLab{#1}{#2}{\ol{a}b}{a\ol{b},\ol{a}b,\ol{a}\ol{b}}}}%
        {38}{\ensuremath{\satzCLab{#1}{#2}{ab,\ol{a}b}{ab,a\ol{b},\ol{a}b}}}%
        {42}{\ensuremath{\satzCLab{#1}{#2}{ab,\ol{a}b}{ab,\ol{a}b,\ol{a}\ol{b}}}}%
        {43}{\ensuremath{\satzCLab{#1}{#2}{ab,\ol{a}\ol{b}}{ab,\ol{a}b,\ol{a}\ol{b}}}}%
        {44}{\ensuremath{\satzCLab{#1}{#2}{\ol{a}b,\ol{a}\ol{b}}{ab,\ol{a}b,\ol{a}\ol{b}}}}%
        {46}{\ensuremath{\satzCLab{#1}{#2}{\ol{a}b,\ol{a}\ol{b}}{a\ol{b},\ol{a}b,\ol{a}\ol{b}}}}%
        {47}{\ensuremath{\satzCLab{#1}{#2}{\ol{a}b}{ab,a\ol{b},\ol{a}b,\ol{a}\ol{b}}}}%
        {48}{\ensuremath{\satzCLab{#1}{#2}{ab,\ol{a}b}{ab,a\ol{b},\ol{a}b,\ol{a}\ol{b}}}}%
        {49}{\ensuremath{\satzCLab{#1}{#2}{\ol{a}b,\ol{a}\ol{b}}{ab,a\ol{b},\ol{a}b,\ol{a}\ol{b}}}}%
        {50}{\ensuremath{\satzCLab{#1}{#2}{ab,\ol{a}b,\ol{a}\ol{b}}{ab,a\ol{b},\ol{a}b,\ol{a}\ol{b}}}}%
    }[\PackageError{nfcab}{Undefined option to nfcab: #1}{}]%
}%

\newcommand{\nfcabNeu}[2][no]{%
    \IfEqCase{#2}{%
        {01}{\ensuremath{\satzCLab{#1}{#2}{3}{3,2}}}%
        {02}{\ensuremath{\satzCLab{#1}{#2}{3}{3,1}}}%
        {03}{\ensuremath{\satzCLab{#1}{#2}{2}{3,2}}}%
        {04}{\ensuremath{\satzCLab{#1}{#2}{1}{3,1}}}%
        {05}{\ensuremath{\satzCLab{#1}{#2}{3}{3,0}}}%
        {06}{\ensuremath{\satzCLab{#1}{#2}{0}{3,0}}}%
        {07}{\ensuremath{\satzCLab{#1}{#2}{2}{2,1}}}%
        {08}{\ensuremath{\satzCLab{#1}{#2}{1}{2,1}}}%
        {09}{\ensuremath{\satzCLab{#1}{#2}{2}{2,0}}}%
        {10}{\ensuremath{\satzCLab{#1}{#2}{1}{1,0}}}%
        {11}{\ensuremath{\satzCLab{#1}{#2}{0}{2,0}}}%
        {12}{\ensuremath{\satzCLab{#1}{#2}{0}{1,0}}}%
        {13}{\ensuremath{\satzCLab{#1}{#2}{3}{3,2,1}}}%
        {14}{\ensuremath{\satzCLab{#1}{#2}{2}{3,2,1}}}%
        {15}{\ensuremath{\satzCLab{#1}{#2}{1}{3,2,1}}}%
        {16}{\ensuremath{\satzCLab{#1}{#2}{3,2}{3,2,1}}}%
        {17}{\ensuremath{\satzCLab{#1}{#2}{3,1}{3,2,1}}}%
        {18}{\ensuremath{\satzCLab{#1}{#2}{2,1}{3,2,1}}}%
        {19}{\ensuremath{\satzCLab{#1}{#2}{3}{3,2,0}}}%
        {20}{\ensuremath{\satzCLab{#1}{#2}{3}{3,1,0}}}%
        {21}{\ensuremath{\satzCLab{#1}{#2}{2}{3,2,0}}}%
        {22}{\ensuremath{\satzCLab{#1}{#2}{1}{3,1,0}}}%
        {23}{\ensuremath{\satzCLab{#1}{#2}{0}{3,2,0}}}%
        {24}{\ensuremath{\satzCLab{#1}{#2}{0}{3,1,0}}}%
        {25}{\ensuremath{\satzCLab{#1}{#2}{3,2}{3,2,0}}}%
        {26}{\ensuremath{\satzCLab{#1}{#2}{3,1}{3,1,0}}}%
        {27}{\ensuremath{\satzCLab{#1}{#2}{3,0}{3,2,0}}}%
        {28}{\ensuremath{\satzCLab{#1}{#2}{3,0}{3,1,0}}}%
        {29}{\ensuremath{\satzCLab{#1}{#2}{2,0}{3,2,0}}}%
        {30}{\ensuremath{\satzCLab{#1}{#2}{1,0}{3,1,0}}}%
        {31}{\ensuremath{\satzCLab{#1}{#2}{2}{2,1,0}}}%
        {32}{\ensuremath{\satzCLab{#1}{#2}{1}{2,1,0}}}%
        {33}{\ensuremath{\satzCLab{#1}{#2}{0}{2,1,0}}}%
        {34}{\ensuremath{\satzCLab{#1}{#2}{2,1}{2,1,0}}}%
        {35}{\ensuremath{\satzCLab{#1}{#2}{2,0}{2,1,0}}}%
        {36}{\ensuremath{\satzCLab{#1}{#2}{1,0}{2,1,0}}}%
        {37}{\ensuremath{\satzCLab{#1}{#2}{3}{3,2,1,0}}}%
        {38}{\ensuremath{\satzCLab{#1}{#2}{2}{3,2,1,0}}}%
        {39}{\ensuremath{\satzCLab{#1}{#2}{1}{3,2,1,0}}}%
        {40}{\ensuremath{\satzCLab{#1}{#2}{0}{3,2,1,0}}}%
        {41}{\ensuremath{\satzCLab{#1}{#2}{3,2}{3,2,1,0}}}%
        {42}{\ensuremath{\satzCLab{#1}{#2}{3,1}{3,2,1,0}}}%
        {43}{\ensuremath{\satzCLab{#1}{#2}{3,0}{3,2,1,0}}}%
        {44}{\ensuremath{\satzCLab{#1}{#2}{2,1}{3,2,1,0}}}%
        {45}{\ensuremath{\satzCLab{#1}{#2}{2,0}{3,2,1,0}}}%
        {46}{\ensuremath{\satzCLab{#1}{#2}{1,0}{3,2,1,0}}}%
        {47}{\ensuremath{\satzCLab{#1}{#2}{3,2,1}{3,2,1,0}}}%
        {48}{\ensuremath{\satzCLab{#1}{#2}{3,2,0}{3,2,1,0}}}%
        {49}{\ensuremath{\satzCLab{#1}{#2}{3,1,0}{3,2,1,0}}}%
        {50}{\ensuremath{\satzCLab{#1}{#2}{2,1,0}{3,2,1,0}}}%
    }[\PackageError{nfcab}{Undefined option to nfcab: #1}{}]%
}%

\newcommand{\nfcabCharNeu}[2][no]{%
    \IfEqCase{#2}{%
        {01}{\ensuremath{\satzCLab{#1}{#2}{ab}{ab,a\ol{b}}}}%
        {02}{\ensuremath{\satzCLab{#1}{#2}{ab}{ab,\ol{a}b}}}%
        {03}{\ensuremath{\satzCLab{#1}{#2}{a\ol{b}}{ab,a\ol{b}}}}%
        {04}{\ensuremath{\satzCLab{#1}{#2}{\ol{a}b}{ab,\ol{a}b}}}%
        {05}{\ensuremath{\satzCLab{#1}{#2}{ab}{ab,\ol{a}\ol{b}}}}%
        {06}{\ensuremath{\satzCLab{#1}{#2}{\ol{a}\ol{b}}{ab,\ol{a}\ol{b}}}}%
        {07}{\ensuremath{\satzCLab{#1}{#2}{a\ol{b}}{a\ol{b},\ol{a}b}}}%
        {08}{\ensuremath{\satzCLab{#1}{#2}{\ol{a}b}{a\ol{b},\ol{a}b}}}%
        {09}{\ensuremath{\satzCLab{#1}{#2}{a\ol{b}}{a\ol{b},\ol{a}\ol{b}}}}%
        {10}{\ensuremath{\satzCLab{#1}{#2}{\ol{a}b}{\ol{a}b,\ol{a}\ol{b}}}}%
        {11}{\ensuremath{\satzCLab{#1}{#2}{\ol{a}\ol{b}}{a\ol{b},\ol{a}\ol{b}}}}%
        {12}{\ensuremath{\satzCLab{#1}{#2}{\ol{a}\ol{b}}{\ol{a}b,\ol{a}\ol{b}}}}%
        {13}{\ensuremath{\satzCLab{#1}{#2}{ab}{ab,a\ol{b},\ol{a}b}}}%
        {14}{\ensuremath{\satzCLab{#1}{#2}{a\ol{b}}{ab,a\ol{b},\ol{a}b}}}%
        {15}{\ensuremath{\satzCLab{#1}{#2}{\ol{a}b}{ab,a\ol{b},\ol{a}b}}}%
        {16}{\ensuremath{\satzCLab{#1}{#2}{ab,a\ol{b}}{ab,a\ol{b},\ol{a}b}}}%
        {17}{\ensuremath{\satzCLab{#1}{#2}{ab,\ol{a}b}{ab,a\ol{b},\ol{a}b}}}%
        {18}{\ensuremath{\satzCLab{#1}{#2}{a\ol{b},\ol{a}b}{ab,a\ol{b},\ol{a}b}}}%
        {19}{\ensuremath{\satzCLab{#1}{#2}{ab}{ab,a\ol{b},\ol{a}\ol{b}}}}%
        {20}{\ensuremath{\satzCLab{#1}{#2}{ab}{ab,\ol{a}b,\ol{a}\ol{b}}}}%
        {21}{\ensuremath{\satzCLab{#1}{#2}{a\ol{b}}{ab,a\ol{b},\ol{a}\ol{b}}}}%
        {22}{\ensuremath{\satzCLab{#1}{#2}{\ol{a}b}{ab,\ol{a}b,\ol{a}\ol{b}}}}%
        {23}{\ensuremath{\satzCLab{#1}{#2}{\ol{a}\ol{b}}{ab,a\ol{b},\ol{a}\ol{b}}}}%
        {24}{\ensuremath{\satzCLab{#1}{#2}{\ol{a}\ol{b}}{ab,\ol{a}b,\ol{a}\ol{b}}}}%
        {25}{\ensuremath{\satzCLab{#1}{#2}{ab,a\ol{b}}{ab,a\ol{b},\ol{a}\ol{b}}}}%
        {26}{\ensuremath{\satzCLab{#1}{#2}{ab,\ol{a}b}{ab,\ol{a}b,\ol{a}\ol{b}}}}%
        {27}{\ensuremath{\satzCLab{#1}{#2}{ab,\ol{a}\ol{b}}{ab,a\ol{b},\ol{a}\ol{b}}}}%
        {28}{\ensuremath{\satzCLab{#1}{#2}{ab,\ol{a}\ol{b}}{ab,\ol{a}b,\ol{a}\ol{b}}}}%
        {29}{\ensuremath{\satzCLab{#1}{#2}{a\ol{b},\ol{a}\ol{b}}{ab,a\ol{b},\ol{a}\ol{b}}}}%
        {30}{\ensuremath{\satzCLab{#1}{#2}{\ol{a}b,\ol{a}\ol{b}}{ab,\ol{a}b,\ol{a}\ol{b}}}}%
        {31}{\ensuremath{\satzCLab{#1}{#2}{a\ol{b}}{a\ol{b},\ol{a}b,\ol{a}\ol{b}}}}%
        {32}{\ensuremath{\satzCLab{#1}{#2}{\ol{a}b}{a\ol{b},\ol{a}b,\ol{a}\ol{b}}}}%
        {33}{\ensuremath{\satzCLab{#1}{#2}{\ol{a}\ol{b}}{a\ol{b},\ol{a}b,\ol{a}\ol{b}}}}%
        {34}{\ensuremath{\satzCLab{#1}{#2}{a\ol{b},\ol{a}b}{a\ol{b},\ol{a}b,\ol{a}\ol{b}}}}%
        {35}{\ensuremath{\satzCLab{#1}{#2}{a\ol{b},\ol{a}\ol{b}}{a\ol{b},\ol{a}b,\ol{a}\ol{b}}}}%
        {36}{\ensuremath{\satzCLab{#1}{#2}{\ol{a}b,\ol{a}\ol{b}}{a\ol{b},\ol{a}b,\ol{a}\ol{b}}}}%
        {37}{\ensuremath{\satzCLab{#1}{#2}{ab}{ab,a\ol{b},\ol{a}b,\ol{a}\ol{b}}}}%
        {38}{\ensuremath{\satzCLab{#1}{#2}{a\ol{b}}{ab,a\ol{b},\ol{a}b,\ol{a}\ol{b}}}}%
        {39}{\ensuremath{\satzCLab{#1}{#2}{\ol{a}b}{ab,a\ol{b},\ol{a}b,\ol{a}\ol{b}}}}%
        {40}{\ensuremath{\satzCLab{#1}{#2}{\ol{a}\ol{b}}{ab,a\ol{b},\ol{a}b,\ol{a}\ol{b}}}}%
        {41}{\ensuremath{\satzCLab{#1}{#2}{ab,a\ol{b}}{ab,a\ol{b},\ol{a}b,\ol{a}\ol{b}}}}%
        {42}{\ensuremath{\satzCLab{#1}{#2}{ab,\ol{a}b}{ab,a\ol{b},\ol{a}b,\ol{a}\ol{b}}}}%
        {43}{\ensuremath{\satzCLab{#1}{#2}{ab,\ol{a}\ol{b}}{ab,a\ol{b},\ol{a}b,\ol{a}\ol{b}}}}%
        {44}{\ensuremath{\satzCLab{#1}{#2}{a\ol{b},\ol{a}b}{ab,a\ol{b},\ol{a}b,\ol{a}\ol{b}}}}%
        {45}{\ensuremath{\satzCLab{#1}{#2}{a\ol{b},\ol{a}\ol{b}}{ab,a\ol{b},\ol{a}b,\ol{a}\ol{b}}}}%
        {46}{\ensuremath{\satzCLab{#1}{#2}{\ol{a}b,\ol{a}\ol{b}}{ab,a\ol{b},\ol{a}b,\ol{a}\ol{b}}}}%
        {47}{\ensuremath{\satzCLab{#1}{#2}{ab,a\ol{b},\ol{a}b}{ab,a\ol{b},\ol{a}b,\ol{a}\ol{b}}}}%
        {48}{\ensuremath{\satzCLab{#1}{#2}{ab,a\ol{b},\ol{a}\ol{b}}{ab,a\ol{b},\ol{a}b,\ol{a}\ol{b}}}}%
        {49}{\ensuremath{\satzCLab{#1}{#2}{ab,\ol{a}b,\ol{a}\ol{b}}{ab,a\ol{b},\ol{a}b,\ol{a}\ol{b}}}}%
        {50}{\ensuremath{\satzCLab{#1}{#2}{a\ol{b},\ol{a}b,\ol{a}\ol{b}}{ab,a\ol{b},\ol{a}b,\ol{a}\ol{b}}}}%
    }[\PackageError{nfcab}{Undefined option to nfcab: #1}{}]%
}%

\newcommand{\paarKlNr}[2]{\ensuremath{#1.#2}}

\newcommand{\nfcabNeuKlNr}[2][no]{%
  \IfEqCase{#2}{%
        {01}{\ensuremath{\satzCLab{#1}{\paarKlNr{01}{1}}{3}{3,2}}}%
        {02}{\ensuremath{\satzCLab{#1}{\paarKlNr{01}{2}}{3}{3,1}}}%
        {03}{\ensuremath{\satzCLab{#1}{\paarKlNr{02}{1}}{2}{3,2}}}%
        {04}{\ensuremath{\satzCLab{#1}{\paarKlNr{02}{2}}{1}{3,1}}}%
        {05}{\ensuremath{\satzCLab{#1}{\paarKlNr{03}{1}}{3}{3,0}}}%
        {06}{\ensuremath{\satzCLab{#1}{\paarKlNr{04}{1}}{0}{3,0}}}%
        {07}{\ensuremath{\satzCLab{#1}{\paarKlNr{05}{1}}{2}{2,1}}}%
        {08}{\ensuremath{\satzCLab{#1}{\paarKlNr{05}{2}}{1}{2,1}}}%
        {09}{\ensuremath{\satzCLab{#1}{\paarKlNr{06}{1}}{2}{2,0}}}%
        {10}{\ensuremath{\satzCLab{#1}{\paarKlNr{06}{2}}{1}{1,0}}}%
        {11}{\ensuremath{\satzCLab{#1}{\paarKlNr{07}{1}}{0}{2,0}}}%
        {12}{\ensuremath{\satzCLab{#1}{\paarKlNr{07}{2}}{0}{1,0}}}%
        {13}{\ensuremath{\satzCLab{#1}{\paarKlNr{08}{1}}{3}{3,2,1}}}%
        {14}{\ensuremath{\satzCLab{#1}{\paarKlNr{09}{1}}{2}{3,2,1}}}%
        {15}{\ensuremath{\satzCLab{#1}{\paarKlNr{09}{2}}{1}{3,2,1}}}%
        {16}{\ensuremath{\satzCLab{#1}{\paarKlNr{10}{1}}{3,2}{3,2,1}}}%
        {17}{\ensuremath{\satzCLab{#1}{\paarKlNr{10}{2}}{3,1}{3,2,1}}}%
        {18}{\ensuremath{\satzCLab{#1}{\paarKlNr{11}{1}}{2,1}{3,2,1}}}%
        {19}{\ensuremath{\satzCLab{#1}{\paarKlNr{12}{1}}{3}{3,2,0}}}%
        {20}{\ensuremath{\satzCLab{#1}{\paarKlNr{12}{2}}{3}{3,1,0}}}%
        {21}{\ensuremath{\satzCLab{#1}{\paarKlNr{13}{1}}{2}{3,2,0}}}%
        {22}{\ensuremath{\satzCLab{#1}{\paarKlNr{13}{1}}{1}{3,1,0}}}%
        {23}{\ensuremath{\satzCLab{#1}{\paarKlNr{14}{1}}{0}{3,2,0}}}%
        {24}{\ensuremath{\satzCLab{#1}{\paarKlNr{14}{1}}{0}{3,1,0}}}%
        {25}{\ensuremath{\satzCLab{#1}{\paarKlNr{15}{1}}{3,2}{3,2,0}}}%
        {26}{\ensuremath{\satzCLab{#1}{\paarKlNr{15}{2}}{3,1}{3,1,0}}}%
        {27}{\ensuremath{\satzCLab{#1}{\paarKlNr{16}{1}}{3,0}{3,2,0}}}%
        {28}{\ensuremath{\satzCLab{#1}{\paarKlNr{16}{2}}{3,0}{3,1,0}}}%
        {29}{\ensuremath{\satzCLab{#1}{\paarKlNr{17}{1}}{2,0}{3,2,0}}}%
        {30}{\ensuremath{\satzCLab{#1}{\paarKlNr{17}{2}}{1,0}{3,1,0}}}%
        {31}{\ensuremath{\satzCLab{#1}{\paarKlNr{18}{1}}{2}{2,1,0}}}%
        {32}{\ensuremath{\satzCLab{#1}{\paarKlNr{18}{2}}{1}{2,1,0}}}%
        {33}{\ensuremath{\satzCLab{#1}{\paarKlNr{19}{1}}{0}{2,1,0}}}%
        {34}{\ensuremath{\satzCLab{#1}{\paarKlNr{20}{1}}{2,1}{2,1,0}}}%
        {35}{\ensuremath{\satzCLab{#1}{\paarKlNr{21}{1}}{2,0}{2,1,0}}}%
        {36}{\ensuremath{\satzCLab{#1}{\paarKlNr{21}{2}}{1,0}{2,1,0}}}%
        {37}{\ensuremath{\satzCLab{#1}{\paarKlNr{22}{1}}{3}{3,2,1,0}}}%
        {38}{\ensuremath{\satzCLab{#1}{\paarKlNr{23}{1}}{2}{3,2,1,0}}}%
        {39}{\ensuremath{\satzCLab{#1}{\paarKlNr{23}{2}}{1}{3,2,1,0}}}%
        {40}{\ensuremath{\satzCLab{#1}{\paarKlNr{24}{1}}{0}{3,2,1,0}}}%
        {41}{\ensuremath{\satzCLab{#1}{\paarKlNr{25}{1}}{3,2}{3,2,1,0}}}%
        {42}{\ensuremath{\satzCLab{#1}{\paarKlNr{25}{2}}{3,1}{3,2,1,0}}}%
        {43}{\ensuremath{\satzCLab{#1}{\paarKlNr{26}{1}}{3,0}{3,2,1,0}}}%
        {44}{\ensuremath{\satzCLab{#1}{\paarKlNr{27}{1}}{2,1}{3,2,1,0}}}%
        {45}{\ensuremath{\satzCLab{#1}{\paarKlNr{28}{1}}{2,0}{3,2,1,0}}}%
        {46}{\ensuremath{\satzCLab{#1}{\paarKlNr{28}{2}}{1,0}{3,2,1,0}}}%
        {47}{\ensuremath{\satzCLab{#1}{\paarKlNr{29}{1}}{3,2,1}{3,2,1,0}}}%
        {48}{\ensuremath{\satzCLab{#1}{\paarKlNr{30}{1}}{3,2,0}{3,2,1,0}}}%
        {49}{\ensuremath{\satzCLab{#1}{\paarKlNr{30}{2}}{3,1,0}{3,2,1,0}}}%
        {50}{\ensuremath{\satzCLab{#1}{\paarKlNr{31}{1}}{2,1,0}{3,2,1,0}}}%
    }[\PackageError{nfcab}{Undefined option to nfcab: #1}{}]%
}%

\newcommand{\Vmin}{\ensuremath{V_{min}}}

\newcommand{\Fmin}{\ensuremath{F_{min}}}

\newcommand{\nsrV}{\ensuremath{\mathcal{V}}}
\newcommand{\nsrF}{\ensuremath{\mathcal{F}}}
\newcommand{\leereMenge}{\ensuremath{\varnothing}}

\newcommand{\wableitDelta}{\wableit{\Delta}}

\newcommand{\worder}[1]{\ensuremath{<_{#1}^{\textsf{w}}}}
\newcommand{\worderDelta}{\ensuremath{\worder{\Delta}}}

\newcommand{\lexableit}[1]{\ensuremath{\nmany{#1}{\mathit{lex}}}}
\newcommand{\lexableitDelta}{\lexableit{\Delta}}
\newcommand{\lexorder}[1]{\ensuremath{<_{#1}^{\mathit{lex}}}}
\newcommand{\lexorderDelta}{\lexorder{\Delta}}

\newcommand{\signa}{\Sigma}
\newcommand{\subsigna}{\Theta}
\newcommand{\subsignabar}{\ol{\subsigna}}

\newcommand{\omegasub}{\omega^{\subsigna}}
\newcommand{\omegasubbar}{\omega^{\subsignabar}}

\newcommand{\omegasubj}[1]{\omega^{#1}}

\renewcommand{\ast}{*}

\newcommand{\marg}[2]{{#1|}_{#2}}
\newcommand{\margpw}[2]{{#1}^{#2}}

\makeatletter
\def\moverlay{\mathpalette\mov@rlay}
\def\mov@rlay#1#2{\leavevmode\vtop{%
		\baselineskip\z@skip \lineskiplimit-\maxdimen
		\ialign{\hfil$\m@th#1##$\hfil\cr#2\crcr}}}
\newcommand{\charfusion}[3][\mathord]{%
	#1{\ifx#1\mathop\vphantom{#2}\fi
		\mathpalette\mov@rlay{#2\cr#3}%
	}%
	\ifx#1\mathop\expandafter\displaylimits\fi}
\makeatother

\newcommand{\nmsymbolneu}{\charfusion[\mathrel]{|\mspace{11mu}}{\sim}}

\newcommand{\nmany}[2]{\nmsymbolneu^{#2}_{#1}}

\newcommand{\nmDelta}{\nmany{\Delta}{}}

\newcommand{\nmkappa}{\nmany{\kappa}{}}

\newcommand{\indreasop}{\mathbf{C}}    
\newcommand{\splitcup}{\bigcup\limits_{\signa_1, \signa_2}}

\newcommand{\cindep}{(CInd)}

\newcommand{\fctOPname}{{\ensuremath{\mathit{OP}}}} %
\newcommand{\fctOP}[1]{\ensuremath{\fctOPname(#1)}} %

\newcommand{\fctOPvector}{\ensuremath{(\R^0,\ldots,\R^k)}} %

\newcommand{\falsWname}{\ensuremath{\mathit{\xi}}}
\newcommand{\falsW}{\ensuremath{\falsWname}}

\newcommand{\namePS}{preferred structure\xspace}

\newcommand{\N}{\mathbb{N}}
\newcommand{\dotcup}{\mathbin{\dot\cup}}

\newcommand{\Lprop}{\ensuremath{\mathcal{L}}}

\newcommand{\thistheoremname}{}
\newtheorem{generischespostulat}{\thistheoremname}   %

\newcommand{\selectionStrategy}{\ensuremath{\sigma}}

\newcommand{\indreascrepSelect}{\ensuremath{\mathbf{C}^{\textit{c-rep}}_{\selectionStrategy}}}     %
\newcommand{\indreasskcinf}{\mathbf{C}^{\textit{c-sk}}}     %

\newcommand{\kappaiminussubj}[2]{\ensuremath{\eta_{#1}^{#2}}}

\newcommand{\kappaiminusVektorsubj}[1]{\ensuremath{\vv{\eta}^{#1}}}

\newcommand{\composeVektor}[2]{\ensuremath{(#1,#2)}}

\newcommand{\skCinfwrt}[1]{\nmableit^{\!\!\textit{c-sk}}_{\!\!#1}}

\newcommand{\Rdelta}{\ensuremath{\Delta}}

\newcommand{\cspRof}[1]{\ensuremath{\mathit{CR}(#1)}}

\newcommand{\solutionsRof}[1]{\ensuremath{\solutionsOf{\cspRof{#1}}}}
\newcommand{\CondSynSplitMacro}[6]{\ensuremath{#1 = #2 \bigcup_{#4,#5} #3 \mid #6}}
          
\newcommand{\CondSynSplitGen}{\ensuremath{\CondSynSplitMacro{\Delta}{\Delta^1}{\Delta^2}{\Sigma_1}{\Sigma_2}{\Sigma_3}}}

\newcommand{\safeCondSynSplitMacro}[6]{\ensuremath{#1 = #2 \bigcup^{\sf s}_{#4,#5} #3 \mid #6}}
          
\newcommand{\safeCondSynSplitGen}{\ensuremath{\safeCondSynSplitMacro{\Delta}{\Delta^1}{\Delta^2}{\Sigma_1}{\Sigma_2}{\Sigma_3}}}

\newcommand{\condind}{\perp \!\!\! \perp}

\newcommand{\DeltaIWoJ}[2]{\Delta_{#1\setminus#2}}
\newcommand{\DeltaOneWoThree}{\DeltaIWoJ{1}{3}}
\newcommand{\DeltaTwoWoThree}{\DeltaIWoJ{2}{3}}
\newcommand{\DeltaIWoThree}{\DeltaIWoJ{i}{3}}
\newcommand{\DeltaJWoThree}{\DeltaIWoJ{i'}{3}}
\newcommand{\DeltaIWoJof}[3]{#3_{#1\setminus#2}}

\newcommand{\DeltaVar}[1]{\Delta_{#1}}
\newcommand{\DeltaI}{\DeltaVar{i}}
\newcommand{\DeltaJ}{\DeltaVar{i'}}
\newcommand{\DeltaOne}{\DeltaVar{1}}
\newcommand{\DeltaTwo}{\DeltaVar{2}}
\newcommand{\DeltaThree}{\DeltaVar{3}}

\newcommand{\SigmaVar}[1]{\Sigma_{#1}}

\newcommand{\SigmaThree}{\SigmaVar{3}}
\newcommand{\SigmaZero}{\Sigma_{\nullZero}}

\newcommand{\omegacomp}[2]{\omega^{#1}\omega^{#2}}
\newcommand{\omegacompThree}[3]{\omega^{#1}\omega^{#2}\omega^{#3}}

\newcommand{\omegaIThree}{\omegacomp{i}{3}}

\newcommand{\omegaOneThreeTwo}{\omegacompThree{1}{3}{2}}
\newcommand{\omegaIThreeJ}{\omegacompThree{i}{3}{i'}}
\newcommand{\omegaI}{\omega^i}
\newcommand{\omegaJ}{\omega^{i'}}

\newcommand{\cRel}{(CRel)}
\newcommand{\cInd}{(CInd)}
\newcommand{\cSynSplit}{(CSynSplit)}

\newcommand{\cRelG}{(CRel\textsuperscript{g})}
\newcommand{\cIndG}{(CInd\textsuperscript{g})}
\newcommand{\cSynSplitG}{(CSynSplit\textsuperscript{g})}

\newcommand{\ipCSPG}{(IP-CSP\textsuperscript{g})}

\newcommand{\echtesSplitting}{genuine}

\newcommand{\genSafeCondSynSplitMacro}[6]{\ensuremath{#1 = #2 \bigcup^{\sf gs}_{#4,#5} #3 \mid #6}}
\newcommand{\genSafeCondSynSplitGen}{\genSafeCondSynSplitMacro{\Delta}{\DeltaOne}{\DeltaTwo}{\Sigma_1}{\Sigma_2}{\Sigma_3}}

\newcommand{\nc}{{\,\mid\!\sim\,}}

\newcommand{\genSafe}{generalized safe}
\newcommand{\genSafety}{generalized safety}
\newcommand{\genSafely}{generalized safely}

\newcommand{\nmableitDS}[2]{\nmableit_{\!\!#1}^{\!\!#2}}
\newcommand{\notnmableitDS}[2]{\notnmableit_{\!\!#1}^{\!\!#2}}

\newcommand{\OP}{\mathit{OP}}
\renewcommand{\R}{\Delta}

\newcommand{\ver}{ver}
\newcommand{\fal}{fal}
\newcommand{\VSet}{V}
\newcommand{\FSet}{F}

\newcommand{\wrt}{with respect to}

\newcommand{\Ind}[4]{#2 \mbox{$\, \condind_{#1} \,$} #3 | #4}
\newcommand{\jhii}[1]{#1}

\newcommand{\DeltaKiwi}{\Delta^{k}}

\newcommand{\DeltaSun}{\Delta^{sun}}
\newcommand{\DeltaRain}{\Delta^{rain}}

\newcommand{\CR}{\ensuremath{\mathit{C\!R}}\xspace}

\newcommand{\OmegaOf}[1]{\Omega(#1)}
\newcommand{\VminI}{\Vmin^i}

\newcommand{\FminI}{\Fmin^i}

\newcommand{\nutzeTextstyle}{}

\newcommand{\DeltaBird}{\Delta^{b}}

\newcommand{\worderOf}[1]{<_{#1}^{\sf w}}

\newcommand{\wableitOf}[1]{\nmableitDS{#1}{\sf w}}

\newcommand{\falsWOf}[1]{\falsW_{#1}}
\newcommand{\falsWDelta}{\falsWOf{\Delta}}
\newcommand{\falsWDeltaI}{\falsWOf{\Delta_{i}}}
\newcommand{\falsWDeltaJ}{\falsWOf{\Delta_{\j}}} 
\newcommand{\kappaiminusVektorsubjAlt}[1]{\ensuremath{\vv{\mu}_{#1}}}
\newcommand{\kappaiminusAlt}[1]{\mu_{#1}}

\renewcommand{\SigmaZero}{\Sigma_3}

\renewcommand{\safeCondSynSplitGen}{\ensuremath{\safeCondSynSplitMacro{\Delta}{\DeltaOne}{\DeltaTwo}{\Sigma_1}{\Sigma_2}{\Sigma_3}}}
\renewcommand{\CondSynSplitGen}{\ensuremath{\CondSynSplitMacro{\Delta}{\DeltaOne}{\DeltaTwo}{\Sigma_1}{\Sigma_2}{\Sigma_3}}}

\renewcommand{\j}{i'}
\renewcommand{\cspR}{\ensuremath{\mathit{CR}(\Rdelta)}}
\renewcommand{\R}{\Delta}

\renewcommand{\worderDelta}{<_\Delta^{\sf w}}
\renewcommand{\wableitDelta}{\nmableitDS{\Delta}{\sf w}}

\newtheoremrep{lemma}{Lemma}

\newtheorem{example}[lemma]{Example}
\newtheorem{definition}[lemma]{Definition}

\newtheoremrep{proposition}[lemma]{Proposition}

\begin{document}
	\let\WriteBookmarks\relax
	\def\floatpagepagefraction{1}
	\def\textpagefraction{.001}
	\shorttitle{Broadening the Applicability of Conditional Syntax Splitting}
	\shortauthors{L.Spiegel et~al.}

	\author[FUH]{Lars-Phillip Spiegel}[orcid=0009-0001-1962-752X]
	\cormark[1]
	\ead{lars-phillip.spiegel@fernuni-hagen.de}
	\author[CPT,TUW]{Jonas Haldimann}[orcid=0000-0002-2618-8721]
	\ead{jonas@haldimann.de}
	\author[OUH,CPT]{Jesse Heyninck}[orcid=0000-0002-3825-4052]
	\ead{jesse.heyninck@ou.nl}
	\author[TUD]{Gabriele Kern-Isberner}[orcid=0000-0001-8689-5391]
	\ead{gabriele.kern-isberner@cs.tu-dortmund.de}
	\author[FUH]{Christoph Beierle}[orcid=0000-0002-0736-8516]
	\ead{christoph.beierle@fernuni-hagen.de}
	\affiliation[FUH]{
		organization={FernUniversität in Hagen},
		city={Hagen},
		country={Germany}
	}
	\affiliation[CPT]{
		organization={University of Cape Town and CAIR},
		city={Cape Town},
		country={South Africa}
	}
	\affiliation[TUW]{
		organization={TU Wien},
		city={Vienna},
		country={Austria}
	}
	\affiliation[OUH]{
		organization={Open Universiteit},
		city={Heerlen},
		postcode={6419 AT},
		country={the Netherlands}
	}
	\affiliation[TUD]{
		organization={TU Dortmund University},
		city={Dortmund},
		country={Germany}
	}
	\cortext[cor1]{Corresponding author}
	
	\title[mode = title]{Broadening the Applicability of Conditional Syntax Splitting for Reasoning from Conditional Belief Bases}

\begin{abstract}
In nonmonotonic reasoning from conditional belief bases, an inference operator satisfying syntax splitting postulates allows for taking only the relevant parts of a belief base into account, provided that the belief base splits into subbases based on disjoint signatures.
Because such disjointness is rare in practice, safe conditional syntax splitting has been proposed as a generalization of syntax splitting, allowing the conditionals in the subbases to share some atoms.
Recently this overlap of conditionals has been shown to be limited to trivial, self-fulfilling conditionals.
In this article, we propose a generalization of safe conditional syntax splittings that broadens the 
applicability of splitting postulates.
In contrast to safe conditional syntax splitting, our generalized notion 
supports syntax splittings of a \BB\ $\Delta$ where the subbases of $\Delta$ may share atoms and nontrivial conditionals. We illustrate how this new notion overcomes limitations of previous splitting concepts, and we identify genuine splittings, separating them from simple splittings that do not provide benefits for inductive inference from $\Delta$. We introduce adjusted inference postulates based on our generalization
of conditional syntax splitting, and we evaluate several popular inductive inference operators \wrt\ these postulates. Furthermore, we show
that, while every inductive inference operator satisfying generalized conditional syntax splitting also satisfies conditional syntax splitting, the reverse does not hold.
\end{abstract}

	\begin{highlights}
	\item
	  Generalization of safe conditional syntax splitting for \BBs;

	\item
	Postulates \cRelG, \cIndG, and \cSynSplitG\
	for inductive inference operators;

	\item
	Identification of 
	genuine splittings
	as a necessary condition for splittings beneficial for inductive inference;

	\item
	Evaluation of established inductive inference operators with respect to
	generalized conditional syntax splitting;

	\item Proofs that lexicographic inference, c-inference, inference with a single c-representation 
	determined by an appropriate selection strategy, 
	c-core-closure inference, and System~W satisfy \cSynSplitG;
        
	\item
	Showing that \cSynSplitG\ implies \cSynSplit, but not the other way around.
	\end{highlights}
	
	
	\maketitle

\clearpage
\tableofcontents
\clearpage

\section{Introduction}
\FARBE{Both human \FARBE{and} formal reasoning methods often rely on restricting the amount of information taken into account for a given reasoning task, tuning out unrelated facts and knowledge. The concept of syntax splitting \citep{Parikh99,PeppasWilliamsChopraFoo2015,Kern-IsbernerBrewka17},
and of the related idea of minimum irrelevance \citep{Weydert98KR} have been \FARBE{\FARBE{introduced} as steps} to formalize this goal. This idea has been transferred to inductive inference from conditional \BBs under the motto ``syntax splitting =
relevance + independence'' in the form of postulates (Rel) and (Ind) for inductive inference operators \citep{KernIsbernerBeierleBrewka2020KR}, taking splittings over a \BB\ $\Delta$ into account where the subbases  $\R_1, \R_2$
are given over disjoint subsignatures of $\R$.}

\FARBE{In practice, such
\FARBE{splittings are} rare because the disjointness of subsignatures imposes a harsh restriction.}
The concept
of conditional syntax splitting
\citep{HeyninckKernIsbernerMeyerHaldimannBeierle2023AAAI}
is an approach to overcome this restriction by allowing $\R_1$ and  $\R_2$
to \FARBE{share atoms in their respective subsignatures.}
\FARBE{To ensure
semantic} (conditional) independence given the joint atoms, \FARBE{a safety condition has been formulated,}
enabling local reasoning within the subbases.
\FARBE{The postulate of conditional relevance \cRel\ implements this idea of localized reasoning, formalizing the ability to focus only on syntactically relevant parts of a \BB, while the postulate of conditional independence (CInd) describes the ability to leave aside syntactically irrelevant information \FARBE{\citep{HeyninckKernIsbernerMeyerHaldimannBeierle2023AAAI}}.}
\FARBE{Furthermore, the postulate} of conditional independence
\cInd\ for safe conditional \FARBE{splittings characterizes} avoiding the
drowning effect \citep{Pearl1990systemZTARK,BenferhatDuboisPrade93},
yielding the first formal definition of the notorious drowning problem that
had been described before only by specific examples
\citep{HeyninckKernIsbernerMeyerHaldimannBeierle2023AAAI}.
\FARBE{These splittings are not only interesting from a theoretical point of view by formalizing notions of \FARBE{conditional} relevance and independence, but also have consequences for applications by allowing the breaking down of conditional reasoning to the subbases relevant for a given query, usually reducing the relevant signature significantly. Hence, broadening the \FARBE{applicability} of such splittings provides not only theoretical insights, but also benefits for applications.}

\FARBE{It has been} shown recently that the safety condition in
\citep{HeyninckKernIsbernerMeyerHaldimannBeierle2023AAAI}
has the undesirable consequence that every conditional in the intersection
of $\R_1$ and  $\R_2$ is a trivial self-fulfilling conditional, \FARBE{meaning that it cannot be falsified}
\citep{BeierleSpiegelHaldimannWilhelmHeyninckKernIsberner2024KR},
\FARBE{%
\FARBE{thus}
  imposing a strong restriction on possible splitting benefits for inference.
We develop a generalization of this safety condition,
allowing the intersection of $\Delta_1$ and $\Delta_2$ to contain more meaningful conditionals. 
This greatly broadens the application possibilities of syntax splitting by increasing both the amount of splittings and the amount of \BBs\ where splittings can be exploited for inductive reasoning.}
\FARBE{The main contributions of this \FARBE{article} are:}
\begin{itemize}
\item
  \FARBE{Generalization of} safe conditional syntax splitting \FARBE{for} \BBs; %
\item
  Postulates \cRelG, \cIndG, and \cSynSplitG\
  for
  \genSafe\ conditional syntax splitting;
 \item
  Identification of the subclass of genuine splittings, 
  separating them from the large
  class of simple splittings that have no benefits for inductive inference
  \FARBE{because existing postulates cannot be meaningfully applied to them};
\item
  Evaluation of \FARBE{established} inductive inference operators with respect to
   generalized conditional syntax splitting;
 \item \FARBE{Proofs that lexicographic inference \citep{Lehmann1995}, c-inference \citep{BeierleEichhornKernIsbernerKutsch2018AMAI,BeierleEichhornKernIsbernerKutsch2021AIJ}, inference with a single c-representation determined by an appropriate selection strategy \citep{BeierleKernIsberner2021FLAIRS}, c-core-closure inference \citep{WilhelmKernIsbernerBeierle2024FoIKScb}, and System~W \citep{KomoBeierle2020KI,KomoBeierle2022AMAI} satisfy \cSynSplitG;}
 \item
 Showing that \cSynSplitG\ implies \cSynSplit, but not the other way around.
\end{itemize}

\FARBE{
	This article is a \FARBE{revised} and largely extended version of the paper previously published at IJCAI 2025 \citep{SpiegelHaldimannHeyninckKernIsbernerBeierle2025IJCAI}.
	\FARBE{In particular, we added the evaluation of further inductive inference operators}, showing that lexicographic inference \citep{Lehmann1995} and c-core closure inference \citep{WilhelmKernIsbernerBeierle2024FoIKScb} both satisfy \cSynSplitG. \FARBE{Furthermore}, this articles contains all proofs which were not present in the conference paper, and we added more explanations of the concepts introduced and additional examples illustrating them.}

\FARBE{After recalling the needed background in Sect.~\ref{sec_basics},
we point out the limitations of safe conditional splittings
in Sect.~\ref{sec_limitations}.
Next, we introduce the concepts of generalized safe and genuine conditional syntax splitting and present adapted postulates for inference in Sect.~\ref{sec_gensafe}. \FARBE{We evaluate inductive inference operators with respect to these new postulates in Sect.~\ref{sec_evaluation}\FARBE{ and Sect.~\ref{sec_cinf}. In Sect.~\ref{sec_csynsplitg_stronger}, we 
show that 
\cSynSplitG\ implies \cSynSplit\ but not the other way around}}, before concluding in Sect~\ref{sec_conclusions}.}
\section{Formal Basics}
\label{sec_basics}
Let $\cL$ be a finitely generated propositional language over a \FARBE{signature} $\signa$ with atoms \FARBE{$a,b,c, \ldots$ and}
with formulas \FARBE{$A,B,C, \ldots$} %
\FARBE{We may write} $AB$ instead of $A \wedge B$,
and \FARBE{overline formulas to} indicate negation, i.e., $\notA$ means $\neg A$.
\FARBE{If a statement holds for both $A$ and $\ol{A}$, we will sometimes use $\dot{A}$ to denote both formulas at the same time.}
Let $\Omega$ denote the set of \emph{possible worlds} over $\cL$, %
\FARBE{taken here} simply
as the  set of all propositional interpretations over $\cL$. $\omega \models A$ means that
the propositional formula $A \in \cL$ holds \FARBE{in %
$\omega \in \Omega$; in this case} $\omega$ is called a \emph{model} of $A$, and the set of all models of $A$ is denoted by $\Mod(A)$. For propositions $A,B \in \cL$,  $A \models B$ holds iff $\Mod(A) \subseteq \Mod(B)$, as usual. 
\FARBE{We will} use $\omega$ both for the model and the corresponding \FARBE{complete} conjunction of all positive or negated atoms, \FARBE{allowing} us to use $\omega$ both as an interpretation and a \FARBE{proposition.} %

\FARBE{For $\subsigna \subseteq \signa$,}
let $\cL(\subsigna)$ \FARBE{or short \(\cL_{\subsigna}\)} denote the propositional language defined by $\subsigna$, with associated set of interpretations $\Omega(\subsigna)$ \FARBE{or short \(\Omega_{\subsigna}\)}. Note that while each \FARBE{formula} of $\cL(\subsigna)$ can also be considered as a \FARBE{formula} of $\cL$, the interpretations $\omegasub \in \Omega(\subsigna)$ are not elements of $\Omega(\signa)$ if $\subsigna \neq \signa$. But each interpretation $\omega \in \Omega$ can be written uniquely in the form $\omega = \omegasub \omegasubbar$ with concatenated $\omegasub \in \Omega(\subsigna)$ and $\omegasubbar \in \Omega(\subsignabar)$, where \FARBE{$\subsignabar = \signa \backslash \subsigna$.} %
\FARBE{The world} $\omegasub$ is called the  \emph{reduct} of $\omega$ to $\subsigna$ \citep{Delgrande17local}.
If $\Omega' \subseteq \Omega$ is a subset of models, then $\marg{\Omega'}{\subsigna} = \{ \margpw{\omega}{\subsigna} | \omega \in \Omega'\} \subseteq \Omega(\subsigna)$ restricts $\Omega'$ to a subset of $\Omega(\subsigna)$.
In the following, we will often %
\FARBE{denote subsignatures of \(\Sigma\) by \(\Sigma_1, \Sigma_2, \dots\) and}
write $\omegasubj{i}$ instead of $\omegasubj{\signa_i}$
\FARBE{to ease notation.}

By making use of a conditional operator $|$, we introduce the language %
\FARBE{
\(
 (\cL | \cL) = \{(B|A) \mid A,B \in \cL \}
\)
of \emph{conditionals} over $\cL$.}
Conditionals $(B|A)$  are meant to express  plausible, defeasible rules ``If $A$ then plausibly (usually, possibly, probably, typically etc.) $B$''.
\FARBE{For a world \(\omega\) a conditional \(\satzCL{B}{A}\) is either \emph{verified} by \(\omega\) if \(\omega \models AB\), \emph{falsified} by \(\omega\) if \(\omega \models A\ol{B}\), or \emph{not applicable} to \(\omega\) if \(\omega \models \ol{A}\).}
\FARBE{A conditional $(F|E)$ is called \emph{self-fulfilling}\FARBE{, or \emph{trivial},} if $E\models F$, i.e., there is no world that can falsify it.}
\FARBE{
For a conditional $\delta_i=(B_i|A_i)$ let
\begin{align*}
	\ver(\delta_i)=\{\omega\in\OmegaOf{\Sigma}\mid \omega\models A_iB_i\}\\
	\fal(\delta_i)=\{\omega\in\OmegaOf{\Sigma}\mid \omega\models A_i\ol{B_i}\}
\end{align*}
denote the sets of \FARBE{verifying and falsifying worlds, respectively}.}
\FARBE{A \BB\ $\Delta$ \FARBE{(over $\Sigma$)} is a set of finitely many conditionals from $\condL$.}
A \FARBE{popular semantic framework} %
for interpreting conditionals %
\FARBE{are
\emph{ordinal conditional functions (OCFs)} $\kappa: \Omega \to \naturals \cup \{\infty\}$ with
$\kappa^{-1}(0) \neq \emptyset$. OCFs, also called \emph{ranking
    functions}, %
\FARBE{introduced}, in a more general form, by \citep{Spohn88}.} 
\FARBE{Intuitively, less plausible worlds are assigned higher numbers.}
\FARBE{Formulas are assigned the rank of their most plausible models, i.e.,}
$\kappa(A) := \min\{\kappa(\omega) \mid \omega \models A\}$.
\FARBE{The rank of $(B|A)$ is $\kappa(B|A)=\kappa(AB)-\kappa(A)$.}
A conditional $\condAB$ \FARBE{is \emph{accepted}
by $\kappa$,}
written as $\kappa \models \condAB$, iff
$\kappa(AB) < \kappa(\AnotB)$, i.e., iff $AB$ is more plausible than $\AnotB$.
\FARBE{This is lifted to \BBs\ via $\kappa\models\Delta$ if $\kappa\models (B|A)$ for all $(B|A)\in\Delta$.}
Consistency of a \FARBE{belief base} $\Delta$ can be defined in terms of OCFs \citep{Pearl1990systemZTARK}: $\Delta$ is \FARBE{(strongly)} consistent iff there is an OCF $\kappa$ such that $\kappa \models \Delta$ \FARBE{and \(\kappa(\omega) < \infty\) for all \(\omega \in \Omega\).}
\FARBE{We focus} on (strongly) consistent belief bases in the sense of \citep{Pearl1990systemZTARK,GoldszmidtPearl96} in order to elaborate our approach without having to deal with distracting technical particularities.
The nonmonotonic inference relation $\nmkappa$ induced by an OCF $\kappa$ \FARBE{is given by 
\citep{Spohn88}}
\begin{equation}
	\label{eq_nmocf}
	A \nmkappa B \quad \mbox{iff} \quad A\equiv\bot \ \mbox{or} \ \kappa(AB) < \kappa(A\notB).
\end{equation}

The \emph{marginal of $\kappa$ on} $\subsigna \subseteq \signa$, denoted by $\marg{\kappa}{\subsigna}$, is defined by
$\marg{\kappa}{\subsigna}(\omegasub) = \kappa(\omegasub)$
for any $\omegasub \in \Omega(\subsigna)$. \FARBE{Here $\omegasub$ is treated as a world in $\marg{\kappa}{\subsigna}(\omegasub)$ but as a formula in $\kappa(\omegasub)$.}
\FARBE{Note that this
  marginalization
is a special
case of the general forgetful functor $\mathit{Mod}(\sigma)$ from $\Sigma$-models to
$\subsigna$-models
\citep{BeierleKernIsberner2012AMAI} where
$\sigma$ is
the inclusion from $\subsigna$ to $\Sigma$.}

\FARBE{To formalize inductive inference from \FARBE{belief bases,} the notion of inductive inference operators was introduced \citep{KernIsbernerBeierleBrewka2020KR}.}
An \FARBE{\emph{inductive inference operator}} %
is a mapping $\indreasop$ that assigns to each \FARBE{belief base} $\Delta \subseteq \condL$ an inference relation $\nmDelta$ on $\cL$, i.e.,
\FARBE{\(
\indreasop: \Delta \mapsto \mathord{\nmDelta},
\)}
such that the following two properties hold: %
\begin{description}
	\item[Direct Inference (DI):] If \(\satzCL{B}{A} \in \kb\) then \(A \nmDelta B\), and
	\item[Trivial Vacuity (TV):]  \(A \nmany{\emptyset}{} B\) implies \(A \models B\).
\end{description}
\FARBE{A special subclass of inductive inference operators are OCF-based inductive inference operators $\indreasop^{ocf}: \Delta\mapsto \kappa_\Delta$, assigning to each belief base $\Delta$ an OCF $\kappa_\Delta$ \citep{KernIsbernerBeierleBrewka2020KR}. The inference relation for OCF-based inductive inference operators is then obtained via Equation~\eqref{eq_nmocf}.}
\FARBE{The following are examples for inductive inference operators:}
\FARBE{
\begin{description}
	\item[\textbf{p-Entailment \(\iiopP\)}] \citep{Adams1975,GoldszmidtPearl96}
	considers all models of a \BB\ and 
	is the most cautious preferential inductive inference operator.
	It can be characterized by system P
	because it 
	licenses precisely the inferences that can be obtained by iteratively applying the rules of system P
	\citep{LehmannMagidor92,DuboisPrade1994ConditionalObjects}.

	\item[\textbf{System Z \(\iiopZ\)}] \citep{GoldszmidtPearl96} %
	determines the uniquely defined minimal \FARBE{ranking} model of  $\Delta$,
	and it
	coincides with rational closure \citep{LehmannMagidor92}. \FARBE{System Z is an example of an OCF-based inductive inference operator}.

	\item[\textbf{Lexicographic inference \(\iiopL\)}] \citep{Lehmann1995}
	employs a comparison of the number of conditionals
	falsified by a world, and it 
	extends rational closure.

	\item[\textbf{c-Inference \(\iiopSKc\)}] \citep{BeierleEichhornKernIsbernerKutsch2018AMAI,BeierleEichhornKernIsbernerKutsch2021AIJ}
	considers all c-representations
	which are special ranking functions obtained by summing up natural number impacts assigned to
	falsified conditionals
	\citep{KernIsberner2001,KernIsberner2004AMAI}.

	\item[\textbf{System W \(\iiopW\)}] \citep{KomoBeierle2020KI,KomoBeierle2022AMAI}
	is based on a preferred structure on worlds, and it
	captures both c-inference and system Z
	and thus rational closure.
\end{description}
}
\FARBE{p-Entailment is extended by the four other \iiopnames in the sense that all five \iiopnames are preferential, while}
\FARBE{%
lexicographic inference extends all other four \iiopnames; for an overview of the interrelationships among them see \citep{HaldimannBeierle2024IJAR}. In the rest of this article we will evaluate these and further \iiopnames with respect to their syntax splitting behaviour.}
\section{Safe Conditional Syntax Splitting and its Limitations}
\label{sec_limitations}
Syntax splittings describe that a belief base contains completely independent information about different parts of the signature. \FARBE{According to \citep{KernIsbernerBeierleBrewka2020KR},
a \FARBE{belief base}} $\Delta$ \emph{splits} into subbases $\DeltaOne, \DeltaTwo$ 
\FARBE{if $\{\Sigma_1, \Sigma_2\}$ is a partition of $\Sigma$}
such that $\Delta = \DeltaOne \cup \DeltaTwo$, $\DeltaI \subsetneq (\cL_i | \cL_i), \cL_i = \cL(\signa_i)$ for \FARBE{$i \in \{1, 2\}$}, $\signa_1 \cap \signa_2 = \emptyset$, and $\signa_1 \cup \signa_2 = \signa$\FARBE{, denoted} as 
\begin{equation}
\nutzeTextstyle
\Delta = \DeltaOne \bigcup_{\Sigma_1,\Sigma_2} \DeltaTwo.
\end{equation}
\FARBE{Syntax splittings are \FARBE{very useful 
for}
formalizing the idea that independent information about different topics should not affect each  other in reasoning.}
\FARBE{Syntax splittings}  were generalized in \citep{HeyninckKernIsbernerMeyerHaldimannBeierle2023AAAI} to \emph{conditional} syntax splittings, which allow subbases to share \FARBE{some} atoms in a given subsignature $\Sigma_3$. 

\begin{definition}[conditional syntax splitting \citep{HeyninckKernIsbernerMeyerHaldimannBeierle2023AAAI}]
  \FARBE{A belief base} $\Delta$
  \FARBE{\emph{splits into subbases $\DeltaOne$,$\DeltaTwo$ conditional on $\Sigma_3$},} 
	if there are \(\Sigma_1, \Sigma_2 \subseteq \Sigma\)
	such that $\DeltaI = \Delta \cap ({\cal L}(\Sigma_i\cup\Sigma_3)\mid {\cal L}(\Sigma_i\cup\Sigma_3))$ for $i=1,2$, 
	\FARBE{and $\{\Sigma_1, \Sigma_2,\Sigma_3\}$ is a partition of $\Sigma$.} 
	This is denoted as
	\begin{equation}
		\label{eq:def_cond_syn_split}
		\nutzeTextstyle
		\FARBE{\CondSynSplitGen}.
	\end{equation}
\end{definition}

\FARBE{Unlike} syntax splitting, conditional syntax splitting does not require the subbases $\DeltaOne$ and $\DeltaTwo$ to be disjoint.
For the remainder of this paper, we will use the notation introduced in the following straightforward proposition.
\begin{proposition}
	\label{prop_disjoint}
	Let \CondSynSplitGen\ and let
	\begin{align}
		\FARBE{\DeltaThree}&=\FARBE{\DeltaOne\cap \DeltaTwo}\label{eq_delta3}\\
		\DeltaIWoJ{1}{3}&=\DeltaOne\setminus\DeltaThree\label{eq_delta1wo3} \\
		\DeltaIWoJ{2}{3}&=\DeltaTwo\setminus\DeltaThree\label{eq_delta2wo3}.
	\end{align} 
	Then \(\DeltaOneWoThree, \DeltaTwoWoThree,\DeltaThree\) are pairwise disjoint and
	\begin{equation}\label{eq_disjoint}
		\FARBE{\Delta = \DeltaOneWoThree{\cup}\DeltaTwoWoThree{\cup}\Delta_3.}
	\end{equation}
\end{proposition}
\FARBE{
\begin{proof}
	Pairwise disjointness and \eqref{eq_disjoint} follow immediately from \eqref{eq_delta3}, \eqref{eq_delta1wo3}, and \eqref{eq_delta2wo3}.
\end{proof}
}
	\FARBE{Note that in Proposition~\ref{prop_disjoint}, \FARBE{$\DeltaThree=\Delta\cap({\cal L}(\Sigma_3)|{\cal L}(\Sigma_3))$, implying that $\DeltaThree\subseteq ({\cal L}(\Sigma_3)|{\cal L}(\Sigma_3))$,} and, \FARBE{ for $i\in\{1,2\}$}, $\DeltaJWoThree\subseteq(\cL(\Sigma_{i}\cup\Sigma_{3})|\cL(\Sigma_{i}\cup\Sigma_{3}))$.}

\FARBE{\FARBE{For} $\omega\in\Omega$ and $A\in{\cal L}(\Sigma_i)$  we have}
\begin{equation}
	\label{eq_omega_reduce}
	\omegaOneThreeTwo\models A \quad \text{ iff }\quad \omegaIThree\models A.
\end{equation}

\FARBE{Given a complete conjunction over \(\Sigma_3\),} \FARBE{i.e., a formula uniquely describing a world in $\OmegaOf{\Sigma_3}$}, conditional syntax splittings in general do not ensure 
complete independence of \(\DeltaOne\) and \(\DeltaTwo\) \FARBE{(for details see \citep[Example 6]{HeyninckKernIsbernerMeyerHaldimannBeierle2023AAAI})}. To fix this, \emph{safe} conditional syntax splittings were introduced.
\begin{definition}[safe conditional syntax splitting \citep{HeyninckKernIsbernerMeyerHaldimannBeierle2023AAAI}]
	\label{def:safe:splitting}
	\FARBE{A belief base} $\CondSynSplitGen$ can be \emph{safely split into subbases $\DeltaOne$, $\DeltaTwo$ conditional on a subsignature $\Sigma_3$}, \FARBE{writing}
	\begin{equation}
		\label{eq_safecsynsplit}
		\nutzeTextstyle
		\safeCondSynSplitGen
	\end{equation}
	if the following \emph{safety property} holds \FARBE{for $i,i'\in\{1,2\}$, $i\neq i'$:}
	\begin{align}
		\label{eq_safety}
			\text{for every }\omegaIThree\in\OmegaOf{\Sigma_i\cup\Sigma_3},
			\text{ there \FARBE{is} } \omegaJ\in\OmegaOf{\Sigma_{i'}}		
			\text{ such that }\  \omegaIThreeJ\not\models \bigvee_{(F|E)\in \DeltaJ} E\land \lnot F .
	\end{align}
\end{definition}

The safety condition demands, in essence, that no \FARBE{complete} conjunction over $\SigmaZero$ may force the falsification of a conditional in $\Delta$ when considering $\Sigma$ as a whole.
\FARBE{
\begin{example}[$\DeltaSun$]
	\label{exa_safety}
	Consider the \BB\ $\DeltaSun=\{(\ol{s}|r),\allowbreak(\ol{r}|s),\allowbreak(b|sr),\allowbreak(g|b),\allowbreak(o|s\ol{r}),\allowbreak(\ol{o}|r),\allowbreak(u|or)\}$ 
	\FARBE{describing \FARBE{the following:} %
	If it is (r)ainy, then usually} it is not (s)unny and vice versa. If it is rainy and sunny at the same time, then we can usually observe a rain(b)ow. Maybe superstitiously, we believe that there is usually some (g)old to be found at the end of the rainbow. Unrelated to this, we usually spend some time (o)utside if %
	\FARBE{it is} sunny and not rainy. If it is rainy, then we usually do not spend time outside. If, despite our \FARBE{normal} habits, we do spend time outside and it is rainy, then we \FARBE{usually have} an (u)mbrella. $\DeltaSun$ has a safe conditional syntax splitting%
	\begin{equation}
		\label{eq:exacontrasafety}
		\nutzeTextstyle
		\safeCondSynSplitMacro{\DeltaSun}{\DeltaSun_{1}}{\DeltaSun_{2}}{\{g\}}{\{s,r,o,u\}}{\{b\}}
	\end{equation}
	\FARBE{where $\Sigma_{1}=\{g\}, \Sigma_{2}=\{s,r,o,u\}, \SigmaThree=\{b\}, \DeltaSun_1=\{(g|b)\}$, $\DeltaSun_2=\{(\ol{s}|r),\allowbreak(\ol{r}|s),\allowbreak(b|sr),\allowbreak(o|s\ol{r}),\allowbreak(\ol{o}|r),\allowbreak(u|or)\}$, and $\DeltaSun_3=\emptyset$.}
	This splitting is safe: We can extend any $\omega^1\in\OmegaOf{\Sigma_{1}\cup\Sigma_{3}}$ by any $\omega'\in\OmegaOf{\Sigma_2}$ with $\omega'\models \ol{s}\land\ol{r}\land\ol{o}\land\ol{u}$ without falsifying a conditional in $\DeltaSun_2$. Similarly we can extend any $\omega^2\in\OmegaOf{\Sigma_2\cup\Sigma_{3}}$ by any $\omega''\in\OmegaOf{\Sigma_1}$ with $\omega''\models g$ without falsifying a conditional in $\DeltaSun_1$. 
\end{example}
}

\FARBE{Safe conditional syntax splitting provides similar benefits for inductive inference as syntax splitting. Reasoning in the language of $\Delta_1$ is independent of the conditionals in $\Delta_2$, and vice versa, given we have full knowledge \FARBE{over the atoms in} $\Sigma_3$.}
\FARBE{However, \FARBE{it has been \FARBE{shown} recently \citep{BeierleSpiegelHaldimannWilhelmHeyninckKernIsberner2024KR} that} the safety property \eqref{eq_safety} imposes a \FARBE{strong,} \FARBE{undesired} restriction on $\DeltaThree$. %
	\begin{lemma}[\citep{BeierleSpiegelHaldimannWilhelmHeyninckKernIsberner2024KR}]
		\label{lem_selffulfilling}
		Let \safeCondSynSplitGen, then $\DeltaThree\FARBE{=\Delta_1\cap\Delta_2}$ contains only self-fulfilling conditionals.
	\end{lemma}
	While it is true that $\Delta_3$ can not contain ``meaningful'' information, the elements in $\Sigma_3$ are still relevant and can occur in conditionals of both $\DeltaOne$ and $\DeltaTwo$ \FARBE{(as we see in Example~\ref{exa_safety})}.

\FARBE{A generalization of the safety property to avoid \FARBE{the effect described in Lemma~\ref{lem_selffulfilling}} would be advantageous.}
	Recall that $\SigmaThree$ represents a sort of global knowledge, that should be considered in both subbases. However it is not always possible to find a safe splitting, given some intuitive or in practice desirable allocation of signature elements to $\Sigma_3$. %

	\FARBE{
	\begin{example}[$\DeltaSun$ cont.]
		\label{exa_contrasafety}
\FARBE{Assume we want to reason based on $\DeltaSun$, under the assumption that we have full knowledge about $s$ and $r$.}
\FARBE{A conditional syntax splitting \FARBE{reflecting} our knowledge about the weather is}
		\begin{equation}
			\nutzeTextstyle
			\label{eq:excontrasafety2split}
			\CondSynSplitMacro{\DeltaSun}{\DeltaSun_{1}}{\DeltaSun_{2}}{\{b,g\}}{\{o,u\}}{\{s,r\}}
		\end{equation}
		\FARBE{where $\FARBE{\Sigma_{1}=\{b,g\}},\allowbreak \Sigma_{2}=\{o,u\}, \allowbreak \SigmaThree=\{s,r\},\allowbreak \DeltaSun_1=\{(\ol{s}|r),\allowbreak(\ol{r}|s),\allowbreak(b|sr),\allowbreak\FARBE{(g|b)}\},\allowbreak \DeltaSun_2=\{(\ol{s}|r),\allowbreak(\ol{r}|s),\allowbreak (o|s\ol{r}),\allowbreak(\ol{o}|r),\allowbreak (u|or)\}$, and $\DeltaSun_3=\{(\ol{s}|r),\allowbreak(\ol{r}|s)\}$.}
		\FARBE{Assume that we know that it is sunny and rainy at the same time, and we would like to know if there will usually be a rainbow, i.e., 
		whether
		$sr\nmableitDS{\DeltaSun}{}b$ holds.}
		\FARBE{
		Employing the splitting \eqref{eq:excontrasafety2split},
		it suffices to consider $\DeltaSun_1$ to answer this query because we have full knowledge about $\{s,r\}$.}
		\FARBE{However, because the conditionals in $\DeltaSun_3$ are not self-fulfilling the splitting \eqref{eq:excontrasafety2split} is not safe.}
	\end{example}
	\FARBE{Comparing the splittings \eqref{eq:exacontrasafety} and \eqref{eq:excontrasafety2split}, we can see that there are situations where \eqref{eq:excontrasafety2split} provides benefits not provided by \eqref{eq:exacontrasafety}. For instance,
	as it will be shown formally in the following sections,
	answering the query $sr\nmableitDS{\DeltaSun}{}b$
	can be done using 
	the subbase $\DeltaSun_{1}$ from \eqref{eq:excontrasafety2split}
	while the splitting \eqref{eq:exacontrasafety} does not provide any advantage for answering this query.}
	}

	Another limitation of safe conditional syntax splittings} is that there exist \BBs\
	\FARBE{where all safe conditional syntax splittings 
	involve \FARBE{a subset} relationship between the subbases.}
	\begin{example}[$\DeltaRain$]
		\label{exa_contrasafety2}
		Starting from $\DeltaSun$, we get rid of our superstitious beliefs by removing the signature element $g$ and all associated conditionals,
		yielding $\DeltaRain=\{(\ol{s}|r),(\ol{r}|s),(b|sr),(o|s\ol{r}),(\ol{o}|r), (u|or)\}$. %
		\FARBE{$\DeltaRain$ has a splitting conditional on $\{s, r\}$}
		\begin{align}
			\label{eq:excontrasafety3}
			\begin{split}
				\DeltaRain=\{(\ol{s}|r),(\ol{r}|s),(b|sr)\}
				\bigcup_{\{b\},\FARBE{\{o,u\}}}\{(\ol{s}|r),(\ol{r}|s), (o|s\ol{r}),(\ol{o}|r),(u|or)\}\mid\{s,r\}
			\end{split}
		\end{align}
\FARBE{which, however, is not safe.
In fact, every safe splitting of $\safeCondSynSplitMacro{\DeltaRain}{\DeltaRain_1}{\DeltaRain_2}{\Sigma_1}{\Sigma_2}{\SigmaZero}$ 
satisfies
	$\DeltaRain_1\subseteq\DeltaRain_2$ or $\DeltaRain_2\subseteq\DeltaRain_1$.}
	\end{example}
	
	\FARBE{
	Every conditional syntax splitting \CondSynSplitGen\ with $\Delta_1\subseteq\Delta_2$ or $\Delta_2\subseteq\Delta_1$ is of little use for inductive inference.
	Suppose $\Delta_1\subseteq\Delta_2$. 
	Then answering any query over $\Sigma_2\cup\Sigma_3$ requires considering $\Delta$ as a whole
	because $\Delta_2=\Delta$.
	Furthermore, any query over $\Sigma_1\cup\Sigma_3$ can also not benefit from the splitting. This is because
	atoms of $\Sigma_1$ can not appear in \FARBE{$\Delta_1$}, since all conditionals of $\Delta_1$ are defined over $\Sigma_3$ as $\Delta_1=\Delta_1\cap\Delta_2\FARBE{=\Delta_3}$ and full knowledge of $\Sigma_3$ is required to make use of the splitting.}

	\FARBE{The observations above give rise to two points. First, we will extend the notion of safety to cover conditional splittings like \eqref{eq:excontrasafety2split} and \eqref{eq:excontrasafety3}. 
	Second, we will identify splittings that are useful for inductive inference.} 
\FARBE{
}

\section{Generalized Safe Conditional Syntax Splitting}
\label{sec_gensafe}

\FARBE{
We first introduce \emph{generalized safe} splittings as a generalization of safe splittings to cover cases where the subbases may share non-trivial conditionals,
and we introduce \emph{genuine} splittings, which identify splittings that provide benefits for inductive inference.}
\begin{definition}[generalized safe conditional syntax splitting]
	\label{def_generalsafety}
	A \FARBE{belief base} $\CondSynSplitGen$ can be \emph{ \genSafely\ split into subbases $\DeltaOne$, $\DeltaTwo$ conditional on a subsignature $\Sigma_3$}, \FARBE{writing}
	\begin{equation}
		\label{eq_gsafecsynsplit}
		\nutzeTextstyle
		\genSafeCondSynSplitGen
	\end{equation}
	if the following \emph{\genSafety\ property} holds \FARBE{for $i,i'\in\{1,2\}, i\neq i'$}:
	\begin{align}
		\label{eq_gsafety}
			\text{for every }\omegaIThree\in\OmegaOf{\Sigma_i\cup\Sigma_3},
			\text{ there \FARBE{is} } \omegaJ\in\OmegaOf{\Sigma_{i'}}\ 
			\text{ such that }\  \omegaIThreeJ\not\models \bigvee_{(F|E)\in \DeltaJWoThree} E\land \lnot F.
	\end{align}
\end{definition}
The deciding difference in \eqref{eq_gsafety} compared to \eqref{eq_safety} is that only conditionals in $\DeltaJWoThree$ are considered for the \genSafety\ property as opposed to all conditionals in $\DeltaJ$ for \FARBE{the safety property.}
\FARBE{Generalized safety extends the notion of safety in the sense that every safe conditional syntax splitting is also \genSafe. The notions coincide whenever $\DeltaThree=\emptyset$.}
	\begin{propositionrep}\label{prop_relationship_safeties}
		Let $\Delta$ be a \BB over $\Sigma$ with a conditional syntax splitting $S:\CondSynSplitGen$.
		\begin{enumerate}
			\item If $S$ is safe, then $S$ is \genSafe.
			\item If %
			\FARBE{$\Delta_3=\Delta_1\cap\Delta_2=\emptyset$, then $S$} is safe iff $S$ is \genSafe.
		\end{enumerate}
	\end{propositionrep}

	\begin{proof}
		\FARBE{
		Let $\Delta$ be a \BB\ with $\safeCondSynSplitGen$. 1. For $i,j\in\{1,2\}, i\neq j$, for every $\omega^3\in \OmegaOf{\Sigma_i\cup\Sigma_3}$ there is $\omega^j\in\OmegaOf{\Sigma_j}$ such that $\omega^3\omega^j$ does not falsify any conditional in $\DeltaJ$. Then $\omega^3\omega^j$ does not falsify any conditional in $\DeltaJWoThree$ and thus $\genSafeCondSynSplitGen$.
		2. For $\DeltaThree=\emptyset$ we have $\DeltaI=\DeltaIWoThree$, thus the two properties are equivalent.
	}
\end{proof}

\FARBE{Generalized safety allows for more splittings adhering to \FARBE{a notion of safety. In particular,} \genSafe\ splittings allow for non-trivial conditionals \FARBE{in $\DeltaThree$. %
}
\FARBE{
\begin{example}[$\DeltaSun, \DeltaRain$ cont.]
	\label{exa_gensafety}
While not safe, the conditional syntax splittings in \FARBE{Examples~\ref{exa_contrasafety} and~\ref{exa_contrasafety2}} are \genSafe. For \FARBE{instance,} in both examples, the conditional $(\ol{r}|s)$ can be falsified by $rs\in\OmegaOf{\Sigma_i\cup\Sigma_3}$, thus making the splittings not safe, but since $(\ol{r}|s)\in\DeltaSun_3$ and $(\ol{r}|s)\in\DeltaRain_3$, this fact does not lead to a violation of \genSafety. 
\end{example}
}

\FARBE{
With \Cref{exa_gensafety} and \Cref{prop_relationship_safeties} we can see that there exist more \genSafe\ conditional syntax splittings than safe conditional syntax splittings.
Thus, \genSafety\ properly extends the amount of \BBs\ that can be 
conditionally split while adhering to a notion of safety.}}

\FARBE{\FARBE{We will now introduce postulates to evaluate inductive inference operators \wrt\ generalized safe conditional syntax splitting.}	
	\FARBE{The idea of these postulates is that for a belief base with (\FARBE{generalized safe} conditional) syntax splitting, inference over one subsignature should be independent from the information about the other subsignature (given that \FARBE{the valuation of the shared subsignature} is fixed).}
	\FARBE{These postulates are analogous to the postulates conditional relevance and conditional independence \citep{HeyninckKernIsbernerMeyerHaldimannBeierle2023AAAI} but we adapt them to include also generalized safe splittings.}
	\begin{description}\label{post_crelG}
		\item[\cRelG]
		An inductive inference operator ${\bf C}$ satisfies \emph{\FARBE{generalized conditional relevance}} if for any \FARBE{\genSafeCondSynSplitGen}, \FARBE{for $i\in \{1,2\}$ and any $A,B\in {\cal L}(\Sigma_i)$,} and a complete conjunction $E\in {\cal L}(\Sigma_3)$, 
		\[AE\nc_\Delta B \quad \mbox{ iff }  \quad AE\nc_{\Delta_i} B.
		\]
	\end{description}

	Thus,  \cRelG\ restricts the scope of inference by requiring that inferences in the sub-language $\Sigma_1\cup\Sigma_3$ can be made \FARBE{taking only \(\kb_1\) into account.} %
	
	\begin{description}\label{post_cindG}
		\item[\cIndG]
		An inductive inference operator ${\bf C}$ satisfies \emph{\FARBE{generalized conditional independence}} if for any \FARBE{\genSafeCondSynSplitGen}, \FARBE{for $i,j\in \{1,2\}$, $j\neq i$, and  any $A,B\in {\cal L}(\Sigma_i)$, $D\in {\cal L}(\Sigma_j)$,} and a complete conjunction $E\in {\cal L}(\Sigma_3)$, \FARBE{such that $DE\not\nmany{\Delta}{}\bot$ we have}
		\[AE\nc_\Delta B \quad \mbox{ iff } \quad AED\nc_\Delta B.
		\]
	\end{description}
	
	Thus, an inductive inference operator satisfies \cIndG\ if, for any $\Delta$ that safely splits into $\Delta_1$ and $\Delta_2$ conditional on $\Sigma_3$, whenever we have all the necessary information about $\Sigma_3$, inferences from one sub-language are independent from formulas over the other sub-language. 
	
	\FARBE{Analogously to conditional syntax splitting \citep{HeyninckKernIsbernerMeyerHaldimannBeierle2023AAAI}, generalized conditional} syntax splitting  \cSynSplitG\ \FARBE{is the combination of} the two properties \cIndG\ and  \cRelG.
	\begin{description}\label{post_csynsplitG}
		\item[\cSynSplitG]
		
		An inductive inference operator ${\bf C}$ satisfies \FARBE{\emph{generalized conditional syntax splitting}} if it satisfies  \cRelG\ and \cIndG.
	\end{description}
	\FARBE{The difference between \cSynSplit\ and our new variant \cSynSplitG\ is that \cSynSplit\ is defined regarding safe conditional syntax splittings only, while our adjusted variant \cSynSplitG\ takes into account all \genSafe\ conditional syntax splittings. %
		\FARBE{Thus, an inductive inference operator satisfying \cSynSplitG\ respects an increased %
			number of conditional splittings.}} %
\FARBE{%
	\begin{example}[$\DeltaRain$ cont.]
		Recall %
		\FARBE{the splitting \eqref{eq:excontrasafety3} for}
		\(\DeltaRain\) from Example~\ref{exa_contrasafety2}.
		Let ${\bf C}:\Delta\mapsto\nmableitDS{\Delta}{rain}$ be an inductive inference \FARBE{operator} that satisfies \cSynSplitG. \FARBE{Applying} \cRelG, \FARBE{we obtain} that the inference $sr\nmableitDS{\Delta}{rain} b$ holds iff the inference $sr\nmableitDS{\Delta_1}{rain} b$ holds. Thus, if we want to know \FARBE{whether} the inference holds in $\Delta$, it is sufficient to consider only $\Delta_1$, reducing the number of conditionals we need to take into account from 6 to 3 \FARBE{and the number of signature elements from 5 to 3}. By applying \FARBE{\cIndG}, we additionally know that the inferences $sro\nmableitDS{\Delta}{rain} b$ and $sr\ol{o}\nmableitDS{\Delta}{rain} b$ \FARBE{hold} if the inference $sr\nmableitDS{\Delta}{rain} b$ holds. In this way, we can localize our reasoning tasks for the entire \BB\ to a smaller subbase, given that our reasoning mechanism satisfies \cSynSplitG.	
	\end{example}
	
	\FARBE{Because \cSynSplitG\ takes into account strictly more splittings than \cSynSplit, \cSynSplitG\ poses a harder requirement for an inductive inference operator to satisfy than \cSynSplit.
	This means that}
 	there are inductive inference operators that satisfy \cSynSplit, but do not satisfy \cSynSplitG. We will formally \FARBE{prove this observation later in} Section~\ref{sec_csynsplitg_stronger}.
	}}

		\FARBE{For \FARBE{governing} inductive \FARBE{inference, we} are only interested in \FARBE{generalized safe or safe} conditional syntax splittings. %
		\FARBE{\Cref{exa_contrasafety2} shows} that there exist \BBs\ that have safe conditional syntax splittings, %
		\FARBE{but $\Delta_1$ is a subset of $\Delta_2$ or vice versa} 
		and therefore, even though the splitting \FARBE{is safe,} it does not provide any meaningful information for inductive inference.		
		\FARBE{
			In fact, 
			given some $\SigmaZero\subseteq\Sigma$, every \FARBE{\BB}\ has at least one syntax splitting conditional on $\SigmaZero$.
			\begin{propositionrep}
				\label{prop_css_every_sigma}
				Let $\Delta$ be a \BB\ over a signature $\Sigma$.
				For every $\SigmaZero\subseteq\Sigma$, there exists the conditional syntax splitting
				\begin{equation}
					\label{eq_everyS3split}
					\FARBE{\CondSynSplitMacro{\Delta}{\Delta}{(\Delta \cap ({\cal L}(\Sigma_3)|{\cal L}(\Sigma_3)))}{\Sigma\setminus\Sigma_3}{\emptyset}{\Sigma_3}}.
				\end{equation}
		\end{propositionrep}
		\begin{proof}
		Let $\Delta$ be a \BB. Consider, for any $\SigmaZero$, the splitting $\CondSynSplitMacro{\Delta}{\Delta}{(\Delta \cap ({\cal L}(\Sigma_3)|{\cal L}(\Sigma_3)))}{\Sigma\setminus\Sigma_3}{\emptyset}{\Sigma_3}$. This splitting is a conditional syntax splitting of $\Delta$, \FARBE{because $\Delta\cap(\cL((\Sigma\setminus\SigmaZero)\cup\SigmaZero)|\cL((\Sigma\setminus\SigmaZero)\cup\SigmaZero))=\Delta$ and $\Delta\cap(\cL(\SigmaZero)|\cL(\SigmaZero))=\Delta_3$ with $\Sigma=(\Sigma\setminus\Sigma_3)\cup\SigmaZero$ and additionally $\Sigma\setminus\Sigma_3$ and $\SigmaZero$ being disjoint}.
		\end{proof}
		Thus, for every $\Sigma_3\subseteq\Sigma$, there  exists a syntax splitting conditional on $\Sigma_3$. However, as we have elaborated at the end of Sect.~\ref{sec_limitations}, the conditional syntax splitting postulates can not be meaningfully applied to \FARBE{splittings of form \eqref{eq_everyS3split} even if they are generalized safe}.} %
		We now identify splittings that are meaningful \wrt\ inductive inference as so-called
		\emph{\echtesSplitting\ splittings}.
		\begin{definition}[genuine splitting]
			Let $\Delta$ be a \BB\ over a signature $\Sigma$. A conditional syntax splitting \CondSynSplitGen\ of $\Delta$ is called \emph{\echtesSplitting}, if $\Delta_1\not\subseteq\Delta_2$ and $\Delta_2 \not\subseteq\Delta_1$.
		\end{definition}
		\FARBE{Note that \echtesSplitting\ splittings can be equivalently characterized by $\DeltaOneWoThree\neq\emptyset$ and $\DeltaTwoWoThree\neq\emptyset$.}
		\FARBE{Intuitively, we call a splitting \echtesSplitting\ \FARBE{if} each subbase contains \FARBE{information %
			that can} not be found in the other subbase.}
		\FARBE{Thus, genuine splittings help us determine those splittings where conditional syntax splitting postulates actually yield advantages for inductive inference.}
		\FARBE{We call a splitting that is not genuine a \emph{simple} splitting.}
\FARBE{Especially, the following types of splittings are 
\FARBE{simple}.}
		\begin{description}
			\item[\textlangle trivial\textrangle] \FARBE{(i)}~$\safeCondSynSplitMacro{\Delta}{\Delta}{\emptyset}{\Sigma}{\emptyset}{\emptyset}$ \  or \  \FARBE{(ii)}~$\genSafeCondSynSplitMacro{\Delta}{\Delta}{\Delta}{\emptyset}{\emptyset}{\Sigma}$.
			\item[\textlangle set-empty\textrangle] $\Delta_1=\emptyset$ \ or \ $\Delta_2=\emptyset$.
			\item[\textlangle sig-empty\textrangle] $\Sigma_1 = \emptyset$ \ or \ $\Sigma_2 = \emptyset$.
		\end{description}}
			While the \FARBE{\textlangle trivial\textrangle\ splitting (i)} is safe, the \FARBE{splitting (ii)} is not safe, unless $\Delta$ contains self-fulfilling conditionals only. \FARBE{However, splitting (ii) is} \genSafe.
			Furthermore, in the case that $\Sigma$ contains no elements that do not appear in $\Delta$, the 
			\FARBE{simple} 
			splittings are exactly the \textlangle sig-empty\textrangle\ splittings. If \FARBE{$\DeltaThree=\emptyset$, %
				then} the 
				\FARBE{simple}
				splittings are exactly the \textlangle set-empty\textrangle\ splittings.
			\FARBE{Genuine splittings usually make out only a small subset of all conditional syntax splittings (see Proposition~\ref{prop_css_every_sigma}) and thus greatly reduce the number of splittings that need to be considered for generalized safety.}
			\FARBE{We give an} example to illustrate the importance of \FARBE{identifying} \echtesSplitting\ splittings.
			\begin{example}[$\DeltaRain$ cont.]
				\label{exa_simplesplits}
				We continue \Cref{exa_contrasafety2}. The \BB\ $\DeltaRain$ has a total of \FARBE{37 conditional syntax splittings}, out of which
				32 are \genSafe\ splittings, but only 16 are safe splittings.
				Only 5 of the 37 splittings are \echtesSplitting.
				For this \BB\ all \echtesSplitting\ splittings are \genSafe, while no safe splitting is \echtesSplitting.
				\FARBE{The five \echtesSplitting\ \genSafe\ splittings of $\DeltaRain$ are listed here. For a full list of all conditional syntax splittings we refer to the appendix.}
				\begin{align*}
					&\genSafeCondSynSplitMacro{\DeltaRain}{\{(\ol{s}|r),(\ol{r}|s),(b|sr)\}}{\{(\ol{s}|r),(\ol{r}|s),(o|s\ol{r}),(\ol{o}|r),(u|ro)\}}{\{b\}}{\{o,u\}}{\{s,r\}}\\
					&\genSafeCondSynSplitMacro{\DeltaRain}{\{(\ol{s}|r),(\ol{r}|s),(b|sr),(o|s\ol{r}),(\ol{o}|r)\}}{\{(\ol{o}|r),(u|ro)\}}{\{b,s\}}{\{u\}}{\{r,o\}}\\
					&\genSafeCondSynSplitMacro{\DeltaRain}{\{(\ol{s}|r),(\ol{r}|s),(b|sr),(o|s\ol{r}),(\ol{o}|r)\}}{\{(\ol{o}|r),(u|ro)\}}{\{s\}}{\{u\}}{\{b,r,o\}}\\
					&\genSafeCondSynSplitMacro{\DeltaRain}{\{(\ol{s}|r),(\ol{r}|s),(b|sr),(o|s\ol{r}),(\ol{o}|r)\}}{\{(\ol{s}|r),(\ol{r}|s),(o|s\ol{r}),(\ol{o}|r),(u|ro)\}}{\{b\}}{\{u\}}{\{s,r,o\}}\\
					&\genSafeCondSynSplitMacro{\DeltaRain}{\{(\ol{s}|r),(\ol{r}|s),(b|sr)\}}{\{(\ol{s}|r),(\ol{r}|s),(o|s\ol{r}),(\ol{o}|r),(u|ro)\}}{\{b\}}{\{o\}}{\{s,r,u\}}
				\end{align*}
			\end{example}

			\FARBE{While in Example~\ref{exa_simplesplits} the set of genuine and generalized safe splittings coincide, this does not hold in general.}
			\FARBE{ 
			\Cref{exa_simplesplits} \FARBE{shows} that  there exist \BBs\ for which no \echtesSplitting, safe conditional syntax splitting exists, but a \echtesSplitting, \genSafe\ splitting exists. 
			Indeed, conditional syntax splittings that are both \echtesSplitting\ and \genSafe\ are those splittings, where \FARBE{the properties} of conditional relevance and conditional independence for inductive inference may be meaningfully applied.}
\section{Evaluating Inductive Inference Operators with respect to (CSynSplit\textsuperscript{g})}
\label{sec_evaluation}

\FARBE{In this section we \FARBE{consider several} inductive inference \FARBE{operators and evaluate} whether \FARBE{they satisfy \cSynSplitG.} %
}

\subsection{System~Z}
	System Z is an OCF-based inductive inference operator based on the ranking function $\kappa^z$
	\citep{Pearl1990systemZTARK}.
	The definition of $\kappa^z$ crucially relies on the notion of \emph{tolerance}.
	A conditional $\satzCL{B}{A}$ is \emph{tolerated} by a set of conditionals \FARBE{$\Delta=\{(B_1|A_1),\dots,(B_n|A_n)\}$} if there is a world $\omega\in\Omega$ such that $\omega\models AB$ and
	$\omega\models\bigwedge_{i=1}^n(\ol{A_i} \vee B_i)$,
	i.e., iff \(\omega\) verifies $\satzCL{B}{A}$ and does not falsify any conditional in \(\R\).
	For every consistent knowledge base, the notion of tolerance yields
	a unique \emph{inclusion-maximal ordered partition},
	in the following denoted by
	$\fctOP{\R} = \fctOPvector$,
	of $\R$ where each $\R_i$ is the (with respect to set inclusion) maximal subset of $\bigcup_{j=i}^k\R_j$ that is tolerated by $\bigcup_{j=i}^k\R_j$. 
		\FARBE{Intuitively, general conditionals of \(\kb\) are placed in the first sets of \(\fctOP{\kb}\) while more specific conditionals are placed in later parts of the partition.}
		\FARBE{The system~Z ranking function $\kappa^z(\omega)$ is defined as follows.
			If \(\omega\) does not falsify  any conditional \FARBE{in \kb}, then let \(\kappa^z(\omega) = 0\). Otherwise,
			let \(\kb_j\) be the latest part in the tolerance partition containing a conditional falsified by \(\omega\), and let \(\kappa^z(\omega) = j+1\) \citep{GoldszmidtPearl96}.
			System Z yields the inference relation induced by \(\kappa^z\), i.e., %
			${\bf C^z}:\Delta\mapsto \mathord{\nmany{\kappa^z}{}}$.
		}

		Just like safe splittings, \genSafe\ splittings \FARBE{are respected by} the notion of tolerance. 
		\begin{proposition}\label{prop_toleration_gcss}
			Let \genSafeCondSynSplitGen. Then,
			for any \FARBE{$i\in\{1,2\}$, $\Delta_i$} tolerates $(B|A)\in \Delta_i$ iff $\Delta$ tolerates $(B|A)$.
		\end{proposition}
		\begin{proof}
			\FARBE{First,} assume $(B|A)$ is tolerated by $\DeltaI$. Then there must be some $\omegaIThree$ such that $\omegaIThree\models AB$ and there is no $(D|C)\in\DeltaI$ such that $\omegaIThree\models C\ol{D}$. In particular, there \FARBE{is} no such conditional in $\DeltaThree$. Due to \genSafety\ there is an extension $\omegaJ$ such that there is no conditional $(F|E)\in\DeltaJWoThree$ with $\omegaIThreeJ\models E\ol{F}$. Since $\Delta=\DeltaI\cup\DeltaJWoThree$, $(B|A)$ is tolerated by $\Delta$. The other direction is immediate.
		\end{proof}
		
		\FARBE{
			We also state the following \FARBE{proposition} which extends Proposition~\ref{prop_toleration_gcss} to the conditionals in $\Delta_3$ specifically. 
			\begin{proposition}
				\label{lem_d3tol}
				Let \genSafeCondSynSplitGen. Then $\Delta_3$ tolerates $(B|A)\in\Delta_3$ iff $\Delta_1$ and $\Delta_{2}$ tolerate $(B|A)$.
			\end{proposition}
			\begin{proof}
				Assume $(B|A)$ is tolerated by $\Delta_3$. Then there must be some $\omega^3$ such that $\omega^3\models AB$ and there is no $(D|C)\in\Delta_3$ such that $\omega^3\models C\ol{D}$. Due to the generalized safety property there are then extensions $\omega^1$, $\omega^{2}$ of $\omega^3$ such that $\omega^3\omega^1\omega^{2}$ does not falsify any conditional in $\Delta\setminus\Delta_3$. Thus, $(B|A)$ is also tolerated in both $\Delta_1$ and $\Delta_{2}$ and also in $\Delta$. Because $\Delta_3=\Delta_1\cap\Delta_{2}$ the other direction is immediate.
			\end{proof}
		}
		
		\FARBE{From Proposition~\ref{prop_toleration_gcss} and Proposition~\ref{lem_d3tol} we can conclude the following lemma regarding the tolerance partition.
			\begin{proposition}
				\label{corr_op}
				Let \genSafeCondSynSplitGen\ and $\OP(\Delta)=(\Delta^0,\dots,\Delta^k)$. 
				Let $\OP(\Delta_1)=(\Delta_1^0,\dots,\Delta_1^n)$,  $\OP(\Delta_2)=(\Delta_2^0,\dots,\Delta_2^m)$, and $\OP(\Delta_3)=(\Delta_3^0,\dots,\Delta_3^p)$. \FARBE{Let $q=\min\{n,m\}$; furthermore, if $q=n$ let $i=1$ and otherwise let $i=2$}. Then, for $l\in\{0,\dots,q\}$, it holds that
				\begin{enumerate}
					\item $(B_j|A_j)\in\Delta^l$ implies $(B_j|A_j)\in\Delta_1^l$ or $(B_j|A_j)\in\Delta_2^l$ or both, 
					\item $(B_j|A_j)\in\Delta_i^l$ implies $(B_j|A_j)\in\Delta^l$, and 
					\item $(B_j|A_j)\in\Delta_3^l$ iff $(B_j|A_j)\in\Delta_1^l$ and $(B_j|A_j)\in\Delta_2^l$.
				\end{enumerate}
				\FARBE{For $l\in\{q+1,\dots,k\}$ it holds that $(B_j|A_j)\in\Delta^l$ iff $(B_j|A_j)\in\Delta_{i}^l$.}
			\end{proposition}
			\begin{proof}
				For $l=0$ the lemma holds due to Proposition~\ref{prop_toleration_gcss} and Proposition \ref{lem_d3tol}. For $l=1$, observe that $\genSafeCondSynSplitMacro{\Delta\setminus\Delta^0}{\Delta_1\setminus\Delta_1^0}{\Delta_2\setminus\Delta_2^0}{\Sigma_1}{\Sigma_2}{\Sigma_3}$ and the lemma follows from Proposition~\ref{prop_toleration_gcss} and Proposition \ref{lem_d3tol} again.
				\FARBE{Note that Proposition~\ref{prop_toleration_gcss} also implies $\max\{n,m\}=k$.}
				For $0<l\leq k$ we can \FARBE{derive} that $\genSafeCondSynSplitMacro{\Delta\setminus(\Delta^0\cup\dots\cup\Delta^{l-1})}{\Delta_1\setminus(\Delta_1^0\cup\dots\cup\Delta_1^{l-1})}{\Delta_2\setminus(\Delta_2^0\cup\dots\cup\Delta_2^{l-1})}{\Sigma_1}{\Sigma_2}{\Sigma_3}$ and again the lemma follows from Proposition~\ref{prop_toleration_gcss} and Proposition \ref{lem_d3tol}. 
			\end{proof}
			}
		For System~Z we can then show the following result.
		\begin{propositionrep}
			System Z satisfies \cRelG, but does not satisfy \FARBE{\cIndG\ and thus does not satisfy} \cSynSplitG.
		\end{propositionrep}
		\begin{proof}
			Let $\Delta$ be a \BB\ with $\OP(\Delta)=(\Delta^0,\Delta^1,\dots,\Delta^n)$. Let \genSafeCondSynSplitGen\ be some \genSafe\ conditional syntax splitting for $\Delta$ and $\OP(\Delta_i)=(\Delta_i^0,\Delta_i^1,\dots,\Delta_i^n)$ for $i\in\{1,2\}$ the tolerance partition of subbase $\Delta_i$. With \Cref{corr_op} we have that $\Delta_i^m\subseteq\Delta^m$ for $m\in\{0,1,\dots, n\}$, because removing conditionals from $\Delta$ can not undo the \genSafety\ property. Due to the \genSafety\ property, every world in $\Omega(\Sigma_i\cup\SigmaThree)$ has an extension in $\OmegaOf{\Sigma_j}$ such that no conditional in $\DeltaJWoThree$ is falsified. Given that \cRelG\ only considers formulas over $\Sigma_i\cup\SigmaThree$, we have, for any minimal world $\omega$ satisfying any of these formulas, that $\omega$ falsifies a conditional in $\Delta^k$, iff $\omega$ falsifies a conditional in $\Delta_i^k$. Thus System Z satisfies \cRelG.
			\FARBE{Because} System~Z suffers from the drowning problem it does not satisfy \cInd\ \citep{HeyninckKernIsbernerMeyerHaldimannBeierle2023AAAI} and \FARBE{thus does not satisfy \cIndG\ and therefore not \cSynSplitG}.
		\end{proof}
		
		Thus System~Z does not comply with \cSynSplitG, but does satisfy \cRelG. In the following, we will look at two inference operators extending System~Z.

	\subsection{Lexicographic Inference}
	\FARBE{
		Lexicographic inference \citep{Lehmann1995} is another inductive inference operator based on the tolerance partition $\OP(\Delta)=(\Delta^0,\dots,\Delta^k)$. It extends System~Z by taking into account also the \FARBE{number} of falsified conditionals per partition.
		For the definition of lexicographic inference, we use the following functions $\falsW^l$ and $\falsW$
		which map worlds to the set of falsified conditionals
		from the set $\Delta^l$ in the tolerance partition and from \(\Delta\), respectively, given by
		\begin{align}
			\falsWDelta^l (\omega) &:= \{ (B_j|A_j) \in \Delta^l \mid \omega \models A_j \ol{B_j}  \},\label{eq_mapping_f_j} \\
			\label{eq_mapping_f}
			\falsWDelta(\omega) &:=  \{ (B_j|A_j) \in \Delta \mid \omega \models A_j \ol{B_j }  \}.
		\end{align}
		\FARBE{Additionally we will use the \emph{lexicographic ordering} on two vectors in \(\N_0^n\) defined by \((v_1, \dots, v_n) \lexorder{} (w_1, \dots, w_n)\) iff there is a \(k \in \{1, \dots, n\}\) such that \(v_k < w_k\) and \(v_j = w_j\) for \(j = k+1, \dots, n\).}
		
		Utilizing these notions, we obtain the following definition of lexicographic inference.
		
		\begin{definition}[\(\lexorderDelta\), lexicographic inference \citep{Lehmann1995}]
			The binary relation $\FARBE{\lexorderDelta} \subseteq \Omega \times \Omega$  on worlds induced by a belief base \(\Delta\) with \(\OP(\Delta) = (\Delta^0,\dots,\Delta^k)\) is defined by, for any $\omega, \omega' \in \Omega$,
				\[
				\omega \lexorderDelta \omega'  \qquad \text{if} \qquad %
				(|\xi^0_\Delta(\omega)|, \dots, |\xi^k_\Delta(\omega)|) \, \lexorder{} \, (|\xi^0_\Delta(\omega')|, \dots, |\xi^k_\Delta(\omega')|).
				\]
			The order \(\lexorderDelta\) is lifted to 
			consistent formulas by letting, for \(F, G \in \Lprop\),
				\[
				F \lexorderDelta G \qquad \text{if} \qquad 
				\min{}_{\lexorderDelta}\, \Mod(F) \,\, \lexorderDelta \,\, \min{}_{\lexorderDelta}\, \Mod(G)
				\]
			Then, for formulas \(A, B\), \emph{\(A\) lexicographically entails \(B\)} given \(\kb\), denoted as 
			\begin{align*}
				A \lexableitDelta B \qquad \text{if} & \qquad A\ol{B} \equiv \bot\quad \text{ or }\quad AB, A\ol{B} \not\equiv \bot \text{ and } AB \lexorderDelta A\overline{B}. %
			\end{align*}
		\end{definition}
		
		We illustrate lexicographic inference with an example.
		
		\begin{example}[$\DeltaBird$]
			\label{ex:lexinf}
			Let $\Sigma = \left\{b,p,f,w\right\}$
			represent birds, penguins, flying entities and winged entities, and let
			$\DeltaBird=
			\{
			\satzCL{f}{b},
			\satzCL{\ol{f}}{p},
			\satzCL{b}{p},
			\satzCL{w}{b}
			\}$.
			Then $\OP(\DeltaBird) = (\Delta^0, \Delta^1)$ with
			\(\Delta^0 = \{\satzCL{f}{b}, \satzCL{w}{b}\}\) and
			\(\Delta^1 = \{\satzCL{\ol{f}}{p}, \satzCL{b}{p}\}\). From the ordering $\lexorder{\DeltaBird}$ shown in  Figure~\ref{fig:ex_lexorder}, we can see that $\min{}_{\lexorderDelta}\, \Mod(pbw)=bp\ol{f}w\lexorder{\DeltaBird} bp\ol{f}\ol{w}=\min{}_{\lexorderDelta}\, \Mod(pb\ol{w})$ and thus $pb\lexableit{\DeltaBird} w$.
			
			\begin{figure}[tb]
				\centering
				\setlength{\tabcolsep}{5pt}
				\renewcommand{\arraystretch}{1.2}
				\begin{tabular}{c|cc@{\hskip3mm}c}
					\(|\falsWOf{\DeltaBird}^0(\omega)|\) & \(|\falsWOf{\DeltaBird}^1(\omega)|\) & \\
					\cline{1-2}
					0&2&
					\multirow{6}{*}{\qquad\begin{tikzpicture}
							\draw[-{Stealth[length=3mm]}] (0,0.2) -- (0,3) node[midway,above,sloped] {\(\lexorder{\DeltaBird}\)};
					\end{tikzpicture}}
					&\(\ol{b}pf\ol{w}\),\thickspace \(\ol{b}pfw\) \\
					1&1&&\(bpf\ol{w}\)  \\
					0&1&&\(bpfw\),\thickspace \(\ol{b}p\ol{f}\ol{w}\),\thickspace \(\ol{b}p\ol{f}w\)  \\
					2&0&&\(bp\ol{f}\ol{w}\),\thickspace \(b\ol{p}\ol{f}\ol{w}\),\thickspace \\
					1&0&&\(b\ol{p}\ol{f}w\),\thickspace \(b\ol{p}f\ol{w}\),\thickspace \(bp\ol{f}w\) \\
					0&0&&\(\ol{b}\ol{p}\ol{f}\ol{w}\),\thickspace \(\ol{b}\ol{p}fw\),\thickspace \(\ol{b}\ol{p}\ol{f}w\),\thickspace \(\ol{b}\ol{p}f\ol{w}\),\thickspace \(b\ol{p}fw\) \\
				\end{tabular}
				
				\caption{The order \(\lexorder{\DeltaBird}\) over worlds from Example~\ref{ex:lexinf}. The corresponding values of \(|\falsWOf{\DeltaBird}^0(\omega)|, |\falsWOf{\DeltaBird}^1(\omega)|\) are indicated on the left.}
				\label{fig:ex_lexorder}
			\end{figure}
		\end{example}
		
		Before showing that lexicographic inference fully complies with \cSynSplitG\ we state some useful lemmata relating generalized safe conditional syntax splitting to the lexicographic ordering. The first lemma states that the set of falsified conditionals in some partition $\Delta^l$ is equal to the unification of the falsified conditionals in the partitions $\Delta_1^l, \Delta_2^l$ of the subbases.
		
		\begin{lemma}
			\label{lem_falsunion}
			Let \genSafeCondSynSplitGen\ with $\OP(\Delta)=(\Delta^0,\dots,\Delta^k)$, $\OP(\Delta_1)=(\Delta_1^0,\dots,\Delta_1^n)$ and  $\OP(\Delta_2)=(\Delta_2^0,\dots,\Delta_2^m)$. Then, for $l\in\{0,\dots,k\}$ and all $\omega\in\Omega$, we have
			\[
			\falsWDelta^l(\omega)=\falsWOf{\Delta_1}^l(\omega)\cup\falsWOf{\Delta_2}^l(\omega).
			\]
		\end{lemma}
		\begin{proof}
			We show both set inclusions separately.
			
			``$\subseteq$'': Let $(B_j|A_j)\in\falsWDelta^l(\omega)$. Then $(B_j|A_j)\in\Delta^l$ and $\omega\models A\ol{B}$.
			Due to Proposition~\ref{corr_op} we then have $(B_j|A_j)\in\Delta_1^l$ or $(B_j|A_j)\in\Delta_2^l$ or both, \FARBE{and thus} $(B_j|A_j)\in\Delta_1^l\cup\Delta_2^l$. Due to $\omega\models A\ol{B}$ this also means $(B_j|A_j)\in\falsWOf{\Delta_1}^l(\omega)\cup\falsWOf{\Delta_2}^l(\omega)$ and we are done.
			
			``$\supseteq$'': Let $(B_j|A_j)\in\falsWOf{\Delta_1}^l(\omega)\cup\falsWOf{\Delta_2}^l(\omega)$. Then $(B_j|A_j)\in\Delta_1^l$ or $(B_j|A_j)\in\Delta_2^l$ and additionally $\omega\models A\ol{B}$. Due to Proposition~\ref{corr_op} we then have $(B_j|A_j)\in\Delta^l$. Due to $\omega\models A\ol{B}$ this also means $(B_j|A_j)\in\falsWDelta^l(\omega)$ and we are done.		
		\end{proof}
		
		The next lemma states that the number of falsified conditionals in a partition $\Delta^l$ can be calculated by summing up the number of falsified conditionals in the partitions of the subbases $\Delta_1^l$ and  $\Delta_2^l$, taking double counting \FARBE{for $\Delta_3$} into account.
		\begin{lemma}
			\label{lem_falsadd}
			Let \genSafeCondSynSplitGen\ with $\OP(\Delta)=(\Delta^0,\dots,\Delta^k)$. Then, for $l\in\{0,\dots,k\}$ and all $\omega\in\Omega$,
			\begin{align}
				\label{eq_lem_lex1}
				|\falsW^l_\Delta(\omega)|&=|\falsW^l_{\Delta_1}(\omega)|+|\falsW^l_{\Delta_2}(\omega)|-|\falsW^l_{\Delta_3}(\omega)|
				\\
				&=|\falsW^l_{\Delta_1}(\omega^1\omega^3)|+|\falsW^l_{\Delta_2}(\omega^2\omega^3)|-|\falsW^l_{\Delta_3}(\omega^3)|
				\label{eq_lemlex_2}
			\end{align}
		\end{lemma}
		\begin{proof}
			With Lemma~\ref{lem_falsunion} we already know that $|\falsW^l_\Delta(\omega)|=|\falsW^l_{\Delta_1}(\omega)\cup\falsW^l_{\Delta_2}(\omega)|$. Thus $|\falsW^l_\Delta(\omega)|=|\falsW^l_{\Delta_1}(\omega)|+|\falsW^l_{\Delta_2}(\omega)|-|\falsW^l_{\Delta_1}(\omega)\cap\falsW^l_{\Delta_2}(\omega)|$.
			We now show that $|\falsW^l_{\Delta_1}(\omega)\cap\falsW^l_{\Delta_2}(\omega)|=|\falsW^l_{\Delta_3}(\omega)|$. 
			First, $(B_j|A_j)\in\falsW^l_{\Delta_1}(\omega)\cap\falsW^l_{\Delta_2}(\omega)$ implies $(B_j|A_j)\in \Delta_1^l\cap\Delta_2^l$ and $\omega\models A\ol{B}$. With Proposition~\ref{corr_op} we have $(B_j|A_j)\in \Delta_3^l$ and thus with $\omega\models A\ol{B}$ we have $(B_j|A_j)\in\falsWOf{\Delta_3}^l(\omega)$. The other direction is analogous.
			
			Thus we have shown \FARBE{the equalities $|\falsW^l_\Delta(\omega)|=|\falsW^l_{\Delta_1}(\omega)|+|\falsW^l_{\Delta_2}(\omega)|-|\falsW^l_{\Delta_3}(\omega)|$ in \eqref{eq_lem_lex1}. The second step in the equality \eqref{eq_lemlex_2} follows} from the fact that $\Delta_i$ is defined over $\cL(\Sigma_i)\cup\cL(\Sigma_3)$ only and due to \eqref{eq_omega_reduce}.
		\end{proof}
		
		Finally, the next lemma states that the lexicographic ordering of worlds that \FARBE{coincide on} the signature of one subbase is preserved in the lexicographic ordering of the entire \BB\ and vice versa. 
		\begin{lemma}
			\label{lem_lexorderMarg}
			Let \genSafeCondSynSplitGen. Let $\omega_1,\omega_2\in\Omega$ and let $i,\j\in\{1,2\}$, $i\neq \j$. If $\omega_1^{\j}\omega_1^3=\omega_2^{\j}\omega_2^3$ then $\omega_1\lexorderDelta \omega_2$ iff $\omega_1^i\omega_1^3\lexorder{\Delta_i} \omega_2^i\omega_2^3$.
		\end{lemma}
		\begin{proof}
			Let \genSafeCondSynSplitGen.
			Let $\OP(\Delta)=(\Delta^0,\dots,\Delta^k),\OP(\Delta_1)=(\Delta_1^0,\dots,\Delta_1^n), \OP(\Delta_2)=(\Delta_2^0,\dots,\Delta_2^m)$ and $\OP(\Delta_3)=(\Delta_3^0,\dots,\Delta_3^p)$.
			Let $\omega_1,\omega_2\in\Omega$ such that $\omega_1^{\j}\omega_1^3=\omega_2^{\j}\omega_2^3$.
			With Lemma~\ref{lem_falsadd} we have for any $l\in\{0,\dots,k\}$:
			\begin{align*}
				&|\falsWDelta^l(\omega_1)|<|\falsWDelta^l(\omega_2)|
				\\
				\text{ iff }
				&|\falsW^l_{\Delta_i}(\omega_1^i\omega_1^3)|+|\falsW^l_{\Delta_{\j}}(\omega_1^{\j}\omega_1^3)|-|\falsW^l_{\Delta_3}(\omega_1^3)|
				<
				|\falsW^l_{\Delta_i}(\omega_2^i\omega_2^3)|+|\falsW^l_{\Delta_{\j}}(\omega_2^{\j}\omega_2^3)|-|\falsW^l_{\Delta_3}(\omega_2^3)|
				\\
				\text{ iff }
				&|\falsW^l_{\Delta_i}(\omega_1^i\omega_1^3)|+|\falsW^l_{\Delta_{\j}}(\omega_1^{\j}\omega_1^3)|-|\falsW^l_{\Delta_3}(\omega_1^3)|
				<
				|\falsW^l_{\Delta_i}(\omega_2^i\omega_2^3)|+|\falsW^l_{\Delta_{\j}}(\omega_1^{\j}\omega_1^3)|-|\falsW^l_{\Delta_3}(\omega_1^3)|
				\\
				\text{ iff }
				&|\falsW^l_{\Delta_i}(\omega_1^i\omega_1^3)|<|\falsW^l_{\Delta_i}(\omega_2^i\omega_2^3)|
			\end{align*}
			The same arguments can be used to show that $|\falsWDelta^l(\omega_1)|=|\falsWDelta^l(\omega_2)|$ iff $|\falsW^l_{\Delta_i}(\omega_1^i\omega_1^3)|=|\falsW^l_{\Delta_i}(\omega_2^i\omega_2^3)|$.
			Thus, by applying the definition of lexicographic inference, we obtain $\omega_1\lexorderDelta \omega_2$ iff $\omega_1^i\omega_1^3\lexorder{\Delta_i} \omega_2^i\omega_2^3$.
		\end{proof}
		
		Now we are ready to prove that lexicographic inference satisfies \FARBE{generalized conditional syntax splitting}.
		\begin{proposition}
			Lexicographic inference satisfies \cRelG\ and \cIndG\ and thus \cSynSplitG.
		\end{proposition}
		\begin{proof}
			Let \genSafeCondSynSplitGen.
			Let $i\in\{1,2\}$ and let $A,B\in\cL(\Sigma_i)$, $D\in\cL(\Sigma_{\j})$, and let $E$ be a complete conjunction over $\Sigma_3$. \FARBE{We show \cRelG\ \textbf{(I)} and \cIndG\ \textbf{(II)} separately.}
			 
			\textbf{(I)} We show \cRelG\ first. We have to show $AE\lexableitDelta B$ iff $AE\lexableit{\Delta_i} B$, i.e. $AEB\lexorderDelta AE\ol{B}$ iff $AEB\lexorder{\Delta_i} AE\ol{B}$. We show both directions of the iff separately.
			
			``$\Rightarrow$'': Let $AEB\lexorder{\Delta} AE\ol{B}$. Let $\omega_1$ be a minimal model of $AEB$, i.e. $\omega_1\models AEB$ and there is no $\omega_1'$ with $\omega\models AEB$ and $\omega_1'\lexorderDelta \omega_1$. Similarly, let $\omega_2$ be a minimal model with $\omega_2\models AE\ol{B}$. Then we have $\omega_1\lexorderDelta \omega_2$.
			Now choose $\omega_3=\omega_1^i\omega_1^3\omega_2^{\j}$. Because $E$ is a full conjunction, notice that $\omega_1^3=\omega_2^3$.
			With Lemma~\ref{lem_lexorderMarg} we then have $\omega_1^i\omega_1^3\lexorder{\Delta_i}\omega_2^i\omega_2^3$. Due to \eqref{eq_omega_reduce} we have $\omega_1^i\omega_1^3\models AEB$ and $\omega_2^i\omega_2^3\models AE\ol{B}$ and both worlds are minimal in $\lexorder{\Delta_i}$ with this property because $\omega_1$ and $\omega_2$ are minimal in $\lexorderDelta$ with this property. Thus $AEB\lexorder{\Delta_i} AE\ol{B}$.
			
			``$\Leftarrow$'':
			Let $AEB\lexorder{\Delta_i} AE\ol{B}$.
			Let $\omega_1^i\omega_1^3,\omega_2^i\omega_2^3\in\Omega(\Sigma_i\cup\Sigma_3)$ with $\omega_1^i\omega_1^3\models AEB$ and $\omega_2^i\omega_2^3\models AE\ol{B}$ be minimal in $\lexorder{\Delta_i}$ with this property. 
			Now choose some $\omega_\ast$ such that $\omega_\ast^3\omega_\ast^{\j}$ falsifies no conditional outside $\Delta_i$ and such that $\omega_\ast^3=\omega_1^3=\omega_2^3$. Such an $\omega_\ast$ must exist due to the generalized safety property. Now set $\omega_{1}=\omega_1^i\omega_\ast^{\j}\omega_\ast^3$ and $\omega_{2}=\omega_2^i\omega_\ast^{\j}\omega_\ast^3$. Note that the falsification of conditionals outside $\Delta_i$ is only dependent on the $\omega^{\j}\omega^3$ part of any world $\omega$ and thus neither $\omega_{1}$ nor $\omega_{2}$ falsify any conditionals outside $\Delta_i$. With Lemma~\ref{lem_lexorderMarg} we have $\omega_{1}\lexorderDelta\omega_{2}$ and because neither world falsifies a conditional outside of $\Delta_i$ their minimality in $\lexorderDelta$ follows from the minimality  of $\omega_1^i\omega_1^3$ and $\omega_2^i\omega_2^3$ in $\lexorder{\Delta_i}$.

			\textbf{(II)} We show \cIndG\ next. We have to show $AE\lexableitDelta B$ iff $AED\lexableitDelta B$, i.e., $AEB\lexorderDelta AE\ol{B}$ iff $AEDB\lexorderDelta AED\ol{B}$.
			
			``$\Rightarrow$'': Assume $AEB\lexorderDelta AE\ol{B}$. Note that $AE\ol{B}\lexorderDelta AED\ol{B}$. Now choose some $\omega_1,\omega_2$ with $\omega_1\models AEB$ and $\omega_2\models AED\ol{B}$ and such that $\omega_1$ and $\omega_2$ are minimal in $\lexorderDelta$ with this property.
			
			Note that $\omega_1^3=\omega_2^3$ because $E$ is a complete conjunction. Now choose $\omega_{\ast}=\omega_2^i\omega_1^{\j}\omega_1^3$. Then $\omega_\ast\models AE\ol{B}$ and $\omega_1\lexorderDelta\omega_\ast$ per assumption. With Lemma~\ref{lem_lexorderMarg} we have $\omega_1^i\omega_1^3\lexorder{\Delta_i}\omega_2^i\omega_1^3$. Again with Lemma~\ref{lem_lexorderMarg} we have $\omega_1^i\omega_2^{\j}\omega_1^3\lexorderDelta \omega_2^i\omega_2^{\j}\omega_1^3$. Observe that $\omega_1^i\omega_2^{i'}\omega_1^3\models AEDB$. Because $\omega_1=\omega_2^i\omega_2^{\j}\omega_1^3$ is a minimal model of $AED\ol{B}$ we have shown $AEDB\lexorderDelta AED\ol{B}$.
			
			``$\Leftarrow$'': Assume $AEDB\lexorderDelta AED\ol{B}$. Choose $\omega_1,\omega_2$ such that $\omega_1\models AEDB$ and $\omega_2\models AE\ol{B}$ and that both are minimal in $\lexorderDelta$ with this property. Observe again $\omega_1^3=\omega_2^3$. Now choose $\omega_{\ast}=\omega_2^i\omega_1^{\j}\omega_1^3$. Then $\omega_\ast\models AED\ol{B}$ and $\omega_1\lexorderDelta\omega_\ast$ per assumption. With Lemma~\ref{lem_lexorderMarg} we have $\omega_1^i\omega_1^3\lexorder{\Delta_i}\omega_2^i\omega_1^3$. \FARBE{Again} with Lemma~\ref{lem_lexorderMarg} we have $\omega_1^i\omega_2^{\j}\omega_1^3\lexorderDelta \omega_2^i\omega_2^{\j}\omega_1^3$. Observe that $\omega_1^i\omega_2^{\j}\omega_1^3\models AEB$ and that $\omega_2=\omega_2^i\omega_2^{\j}\omega_1^3$ is a minimal model of $AE\ol{B}$ per assumption. Thus we have shown $AEB\lexorderDelta AE\ol{B}$.
		\end{proof}
		
		\FARBE{Hence, lexicographic inference fully complies with \cSynSplitG. In the next section, we evaluate another inference operator employing the tolerance partition with respect to generalized conditional syntax splitting}.
	}
	
	\subsection{System W}
	\FARBE{\FARBE{System~W \citep{KomoBeierle2020KI,KomoBeierle2022AMAI}} is \FARBE{an inference operator also} using the tolerance partition \FARBE{$\OP(\Delta)$}.}
	\FARBE{\FARBE{While System~Z} considers only which parts of \FARBE{$\OP(\Delta)$ contain  falsified conditionals}, %
		\FARBE{and lexicographic inference only considers the \FARBE{number} of conditionals falsified,} 
		System~W  also takes \FARBE{into account the structural information about which conditionals are falsified.}
		
		\FARBE{The results and proofs we give in this subsection are based on results and proofs published in the PhD Dissertation \citep{Haldimann24Diss} for safe conditional syntax splittings. Here, we adapt these results and proofs to generalized safe conditional syntax splittings.}

		\FARBE{System~W utilizes} 
		\FARBE{the \emph{preferred structure on worlds} $\worderDelta$ \FARBE{which} compares worlds according to the set of conditionals in \kb\ they falsify, giving preference to the more specific conditionals according to $\OP(\Delta)$.
	}}
	\FARBE{
\begin{definition}[preferred structure $\worderDelta$ on worlds \citep{KomoBeierle2022AMAI}]
	\label{def_w_structure_worlds}
	Let \FARBE{\(\Delta\)} be a \BB\
	with
	\(\fctOP{\Delta} = (\Delta^0, \dots, \Delta^k)\).
	The \emph{\namePS on worlds} is given by the binary relation
	${\worderDelta} \subseteq \Omega \times \Omega$ 
	defined by, for any $\omega, \omega' \in \Omega$,
	\begin{align}
		\nonumber
		\omega \worderDelta \omega'  \qquad \text{iff} \qquad &\textrm{there exists $l \in \{0\, , \ldots \, , k \}$ such that }\\
		\nonumber
		&\falsWDelta^m(\omega)  = \falsWDelta^m(\omega') \quad \forall m  \in  \{  l + 1 \, , \ldots \, , k  \}, \,\,\textrm{and}\\
		\label{eq_sz_structure}
		& \falsWDelta^l(\omega) \subsetneq  \falsWDelta^l(\omega') \, . 
	\end{align}
\end{definition}

Thus, $\omega \worderDelta \omega'$ if and only if $\omega $ falsifies \FARBE{a strict subset of the conditionals that} $\omega'$ \FARBE{falsifies} in the partition with the \FARBE{largest} index $l$  where the conditionals falsified by $\omega $ and $\omega'$ differ. 

\begin{definition}[System W, $\wableitDelta$ \citep{KomoBeierle2022AMAI}]
	\label{def_sz_inference}
	Let \(\Delta\) be a belief base and \(A, B\) be formulas.
	Then $B$ is a  \emph{System~W inference from $A$ (in the context of \(\Delta\))}, denoted \(A \wableitDelta B \), if we have:
	\begin{align}\label{eq_sz_inference}
		A \wableitDelta B \textrm{ \qquad iff \qquad } &
		\begin{array}[t]{@{}l}
			\text{ for every } \omega' \in \Omega%
			\text{ with } \omega'\models A\ol{B}
			\text{ there is an } \omega \in \Omega%
			\text{ with } \omega\models AB
			\text{ s.t. } \omega \worderDelta \omega ' \,
		\end{array}
	\end{align}
\end{definition}
We illustrate the above definitions with an example.
}

\FARBE{
\begin{example}[$\DeltaBird$ cont.]\label{exa:sysw}
	Consider again
	\(\OP(\DeltaBird) = (\Delta^0, \Delta^1)\) with
	\(\Delta^0 = \{\satzCL{f}{b}, \satzCL{w}{b}\}\) and
	\(\Delta^1 = \{\satzCL{\ol{f}}{p}, \satzCL{b}{p}\}\). 
	Utilizing
	$\worder{\DeltaBird}$,
	shown in Figure~\ref{fig_birds_preferred_structure}, we can see that
	$pb \wableitOf{\DeltaBird} w$ holds, because for every $\omega'\in\Omega$ with $\omega\models pb\ol{w}$ there is the world $\omega=bp\ol{f}w$ such that $\omega\worderOf{\DeltaBird} \omega'$. 

  \begin{figure}%
  \centering
\normalsize
\begin{tikzpicture}
	\tikzstyle{world}=[inner sep=.5mm,]
	\tikzstyle{conn} = [->]
	\tikzstyle{conn1} = []

	\def\xdiff{1.7}
	\def\txdiff{1.7}
	\def\ydiff{-1}
	
	\node (11) at (0,0) [world] {\omegaJbpfw};
	\node (12) at (\xdiff,0) [world] {\omegaIbpfw};
	\node (21) at (-\xdiff,\ydiff) [world] {\omegaBbpfw};
	\node (31) at (-\xdiff,2*\ydiff) [world] {\omegaAbpfw};
	\node (32) at (0,2*\ydiff) [world] {\omegaLbpfw};
	\node (33) at (\xdiff,2*\ydiff) [world] {\omegaKbpfw};
	\node (41) at (-.5*\xdiff,3*\ydiff) [world] {\omegaDbpfw};
	\node (42) at (.5*\xdiff,3*\ydiff) [world] {\omegaHbpfw};
	\node (51) at (-\xdiff,4*\ydiff) [world] {\omegaGbpfw};
	\node (52) at (0,4*\ydiff) [world] {\omegaFbpfw};
	\node (53) at (\xdiff,4*\ydiff) [world] {\omegaCbpfw};
	\node (61) at (-2*\txdiff,5.5*\ydiff) [world] {\omegaMbpfw};
	\node (62) at (-\txdiff,5.5*\ydiff) [world] {\omegaObpfw};
	\node (63) at (0,5.5*\ydiff) [world] {\omegaNbpfw};
	\node (64) at (\txdiff,5.5*\ydiff) [world] {\omegaPbpfw};
	\node (65) at (2*\txdiff,5.5*\ydiff) [world] {\omegaEbpfw};
	
	\draw[conn] (21) -- (11);
	\draw[conn] (21) -- (12);
	\draw[conn] (31) -- (21);
	\draw[conn] (32) -- (11);
	\draw[conn] (32) -- (12);
	\draw[conn] (33) -- (11);
	\draw[conn] (33) -- (12);
	\draw[conn] (41) -- (31);
	\draw[conn] (41) -- (32);
	\draw[conn] (41) -- (33);
	\draw[conn] (42) -- (31);
	\draw[conn] (42) -- (32);
	\draw[conn] (42) -- (33);
	\draw[conn] (51) -- (41);
	\draw[conn] (51) -- (42);
	\draw[conn] (52) -- (41);
	\draw[conn] (52) -- (42);
	\draw[conn] (53) -- (41);
	\draw[conn] (53) -- (42);
	\draw[conn] (61) -- (51);
	\draw[conn] (61) -- (52);
	\draw[conn] (61) -- (53);
	\draw[conn] (62) -- (51);
	\draw[conn] (62) -- (52);
	\draw[conn] (62) -- (53);
	\draw[conn] (63) -- (51);
	\draw[conn] (63) -- (52);
	\draw[conn] (63) -- (53);
	\draw[conn] (64) -- (51);
	\draw[conn] (64) -- (52);
	\draw[conn] (64) -- (53);
	\draw[conn] (65) -- (51);
	\draw[conn] (65) -- (52);
	\draw[conn] (65) -- (53);
\end{tikzpicture}
  \caption{%
  	The preferred structure on worlds \(\worder{\DeltaBird}\) in Example~\ref{exa:sysw}.
    An edge \(\omega \rightarrow \omega'\) indicates that \(\omega \worder{\DeltaBird} \omega'\); edges that can be obtained from transitivity are omitted.}
\label{fig_birds_preferred_structure}
\end{figure}
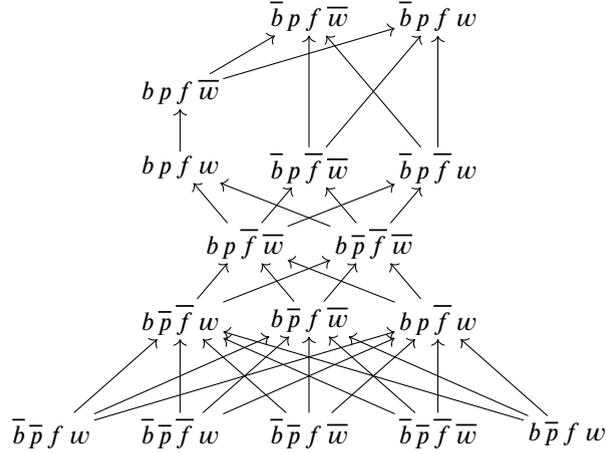

 \end{example}
}
\FARBE{
Before showing that System~W fully complies with \cSynSplitG, we show some lemmata regarding generalized safe conditional syntax splitting and the preferred structure on \FARBE{worlds first.
The first} lemma states that the order between two worlds must be \FARBE{reflected by} the order of one of the subbases.
\begin{lemma}%
	\label{lem_syswp_either}
	Let \genSafeCondSynSplitGen\ and let $\omega,\omega'\in\Omega$. If $\omega\worderDelta \omega'$ then $\omega\worderOf{\Delta_1} \omega'$ or $\omega\worderOf{\Delta_2} \omega'$. 
\end{lemma}
\begin{proof}
	Let $OP(\Delta)=(\Delta^0,\dots,\Delta^k)$. Let $\omega,\omega'\in\Omega$ with $\omega\worderDelta\omega'$. By definition there is $n\in\{0,\dots,k\}$ such that $\falsWDelta^j(\omega)=\falsWDelta^j(\omega')$ for all $j\in\{n+1,\dots,k\}$ and $\falsWDelta^n(\omega)\subsetneq\falsWDelta^n(\omega')$. \FARBE{Then there is $\delta\in\falsWDelta^n(\omega')$ with $\delta\notin\falsWDelta^n(\omega)$. We have $\delta\in \Delta$, implying $\delta\in\Delta_1$ or $\delta\in\Delta_2$.} Let \FARBE{$i\in\{1,2\}$ such that $\delta\in\Delta_i$} and let $OP(\Delta_i)=(\Delta_i^1,\dots,\Delta_i^m)$. With Proposition~\ref{corr_op} we have that $\Delta_i^j\subseteq\Delta^i$. Thus due to $\falsWDelta^j(\omega)=\falsWDelta^j(\omega')$ we also have $\falsWDeltaI^j(\omega)=\falsWDeltaI^j(\omega')$ and because $\falsWDelta^n(\omega)\subsetneq\falsWDelta^n(\omega')$ and $\delta\in\Delta_i$ we have $\falsWDeltaI^n(\omega)\subsetneq\falsWDeltaI^n(\omega')$. Thus $\omega\worderOf{\Delta_i}\omega'$.
\end{proof}

\FARBE{The next lemma 
states that the order between two worlds in the preferred structure of $\Delta_i$ is preserved in the preferred structure of $\Delta$ if the two worlds \FARBE{coincide} on the signature of $\Delta_{\j}$.}
\begin{lemma}%
	\label{lem_syswp_match}
	Let \genSafeCondSynSplitGen, let $\omega,\omega'\in\Omega$ and let $i,i'\in\{1,2\}$ with $i\neq i'$. If $\omega\worderOf{\Delta_i} \omega'$ and $\omega_{\mid \Sigma_{i'}\cup\Sigma_3}=\omega'_{\mid \Sigma_{i'}\cup\Sigma_3}$ then $\omega\worderOf{\Delta} \omega'$.
\end{lemma}
\begin{proof}
	Let $i,i'\in\{1,2\}, i\neq i'$, and let
	$OP(\Delta)=(\Delta^0,\dots,\Delta^k)$. Let $\omega,\omega'\in\Omega$ with $\omega\worderOf{\Delta_i}\omega'$ and $\omega_{\mid\Sigma_{i'}\cup\Sigma_3}=\omega'_{\mid \Sigma_{i'}\cup\Sigma_3}$. Let $OP(\Delta_i)=(\Delta^0_i,\dots,\Delta^m_i)$. Per definition there is $n\in\{0,\dots,m\}$ such that $\falsWDeltaI^j(\omega)=\falsWDeltaI^j(\omega')$ for $j\in\{n+1,\dots,m\}$ and $\falsWDeltaI^n(\omega)\subsetneq\falsWDeltaI^n(\omega').$ 
	With Lemma~\ref{lem_falsunion} we have $\falsWDelta^l(\omega^*)=\falsWDeltaI^l(\omega^*)\cup\falsWDeltaJ^k(\omega^*)$ for all worlds $\omega^*$ and $l\in\{0,\dots,m\}$. 
	Because $\omega_{\mid\Sigma_{i'}\cup\Sigma_3}=\omega'_{\mid \Sigma_{i'}\cup\Sigma_3}$ we have $\falsWDeltaJ(\omega)=\falsWDeltaJ(\omega')$. Then $\falsWDeltaI^j(\omega)=\falsWDeltaI^j(\omega')$ implies $\falsWDelta^j(\omega)=\falsWDelta^j(\omega')$ and $\falsWDeltaI^j(\omega)\subsetneq\falsWDeltaI^j(\omega')$ implies $\falsWDelta^j(\omega)\subsetneq\falsWDelta^j(\omega')$, and thus $\omega\worderOf{\Delta_i}\omega'$ implies $\omega\worderDelta\omega'$.
\end{proof}

Finally, the next lemma shows that the order of worlds \FARBE{with respect to} a subbase is only dependent on the signature relevant to that subbase.
\begin{lemma}%
	\label{lem_syswp_match2}
	Let \genSafeCondSynSplitGen, let $\omega,\omega',\omega^*\in\Omega$ and let $i\in\{1,2\}$ with $\omega_{\mid\Sigma_i\cup\Sigma_3}=\omega'_{\mid\Sigma_i\cup\Sigma_3}$. Then we have $\omega\worderOf{\Delta_i}\omega^*$ iff $\omega'\worderOf{\Delta_i}\omega^*$ and $\omega^*\worderOf{\Delta_i}\omega$ iff $\omega^*\worderOf{\Delta_i}\omega'$.
\end{lemma}
\begin{proof}
	Let $\omega,\omega',\omega^*\in\Omega$ as above. Then $\omega$ and $\omega'$ falsify the same conditionals in $\Delta_i$. Because $\worderOf{\Delta_i}$ is defined solely based on the falsification of conditionals in $\Delta_i$ the lemma follows.
\end{proof}
}

\FARBE{Now we are ready to show that system~W fully complies with 
\FARBE{generalized conditional syntax splitting.}}
\begin{proposition}
\label{prop:Wind:satisfies:c:syn:split}
\FARBE{System W satisfies 
	\cRelG\ and  \cIndG\ and thus
	\cSynSplitG.}
	\end{proposition}
	\begin{proof}
\FARBE{
	\FARBE{Let $\FARBE{i,i'\in\{1,2\}, i\neq i'}$, and let $A,B\in\cL(\Sigma_i)$, $D\in\cL(\Sigma_{\j})$, and let $E$ be a complete conjunction over $\Sigma_3$. We show \cRelG\ \textbf{(I)} and \cIndG\ \textbf{(II)} separately.}
	
	\textbf{(I)} We show \cRelG\ first. We need to show 
	\[
	AE\wableitDelta B\quad \text{iff}\quad AE\wableitOf{\Delta_i} B.
	\]
	Thus, we need to show that for every $\omega'\in\Mod_\Sigma(AE\ol{B})$ there is an $\omega\in\Mod_\Sigma(AEB)$ with $\omega\worderOf{\Delta}\omega'$ iff for every $\omega'\in\Mod_\Sigma(AE\ol{B})$ there is an $\omega\in\Mod_\Sigma(AEB)$ with $\omega\worderOf{\Delta_i}\omega'$.
	
	"$\Rightarrow$":
	Assume that $AE\wableitDelta B$, i.e., $AEB\worderDelta AE\ol{B}$. We need to show $AEB\worderOf{\Delta_i} AE\ol{B}$. Let $\omega'$ be any world with $\omega'\models AE\ol{B}$. Choose $\omega'_{min}$ such that
	\begin{align}
		&\omega'_{min}\leq_\Delta^{\sf w} \omega',\label{eq_wrel_1}
		\\ &\omega'\marg_{\Sigma_i\cup\Sigma_3}=\omega'_{min}\marg_{\Sigma_i\cup\Sigma_3},\text{ and}\label{eq_wrel_2}
		\\
		&\text{there is no }\omega'_{min2} \text{ with } \omega'_{min2}\worderDelta\omega'_{min} \text{ fulfilling \eqref{eq_wrel_1} and \eqref{eq_wrel_2}.} \label{eq_wrel_3}
	\end{align}
	Such an $\omega'_{min}$ exists because $\omega'$ exists. With \eqref{eq_wrel_2} and $\omega'\models AE\ol{B}$  we have $\omega'_{min}\models AE\ol{B}$. Then there must be $\omega$ with $\omega\models AEB$ and $\omega\worderDelta\omega_{min}'$. With Lemma~\ref{lem_syswp_either} either $\omega\worderOf{\Delta_i}\omega_{min}'$ or $\omega\worderOf{\Delta_{i'}}\omega_{min}'$. The second case is not possible, if it were the case that $\omega\worderOf{\Delta_{i'}}\omega_{min}'$ then Lemma~\ref{lem_syswp_match2} implies $\omega'_{min2}=(\omega'_{min}\marg_{\Sigma_{i}}\omega\marg_{\Sigma_{i'}\cup\Sigma_3})\worderOf{\Delta_{i'}}\omega'_{min}$. But with Lemma~\ref{lem_syswp_match} and the fact that $\omega_{min2}'\marg_{\Sigma_i\cup\Sigma_3}=\omega_{min}'\marg_{\Sigma_i\cup\Sigma_3}$ we would have $\omega_{min2}'\worderDelta\omega_{min}'$ which contradicts \eqref{eq_wrel_3}.
	Therefore $\omega\worderOf{\Delta_i}\omega'_{min}$. Then with \eqref{eq_wrel_2} and Lemma~\ref{lem_syswp_match2} we get $\omega\worderOf{\Delta_i}\omega'$ and thus we have $AEB\worderOf{\Delta_i} AE\ol{B}$.
	
	"$\Leftarrow$":
	Assume that $AE\wableitOf{\Delta_i}B$, i.e., $AEB\worderOf{\Delta_i} AE\ol{B}$. We need to show $AEB\worderOf{\Delta} AE\ol{B}$. Let $\omega'$ be any world with $\omega'\models AE\ol{B}$.  $\omega^*$ with $\omega^*\models AEB$ and $\omega^*\worderOf{\Delta_i}\omega'$. Let $\omega=(\omega^*\marg_{\Sigma_i\cup\Sigma_3}\omega'\marg_{\Sigma_{i'}})$. Then $\omega\models AEB$. With Lemma~\ref{lem_syswp_match2} and the fact that $\omega\marg_{\Sigma_{i}\cup\Sigma_3}=\omega^*\marg_{\Sigma_{i}\cup\Sigma_3}$ we have $\omega\worderOf{\Delta_i}\omega'$. With $\omega\marg_{\Sigma_{i'}\cup\Sigma_3}=\omega'\marg_{\Sigma_{i'}\cup\Sigma_3}$ and with Lemma~\ref{lem_syswp_match} we get $\omega\worderDelta\omega'$ and therefore $AEB\worderDelta AE\ol{B}$.
	
	\textbf{(II)} Next we show \cIndG. We need to show 
	\[
	AE\wableitDelta B\quad \text{iff}\quad AED\wableitDelta B.
	\]
	Thus, we need to show that for every $\omega'\in\Mod_\Sigma(AE\ol{B})$ there is an $\omega\in\Mod_\Sigma(AEB)$ with $\omega\worderDelta\omega'$ iff for every $\omega'\in\Mod_\Sigma(AED\ol{B})$ there is an $\omega\in\Mod_\Sigma(AEDB)$ with $\omega\worderDelta\omega'$.
	
	"$\Rightarrow$:
	Assume that $AE\wableitDelta B$, i.e., $AEB\worderDelta AE\ol{B}$. We need to show $AEDB\worderOf{\Delta} AED\ol{B}$. Let $\omega'$ be any world with $\omega'\models AED\ol{B}$.
	Define $\omega'_{min}$ again, such that equations \eqref{eq_wrel_1}, \eqref{eq_wrel_2} and \eqref{eq_wrel_3} are satisfied. Because $\omega'\models AED\ol{B}$ and \eqref{eq_wrel_2} we have that $\omega'_{min}\models AE\ol{B}$. Due to $AEB\worderDelta AE\ol{B}$ there is $\omega^*$ with $\omega^*\models AEB$ and $\omega^*\worderDelta\omega'_{min}$. 
	Then, with Lemma~\ref{lem_syswp_either}, either $\omega^*\worderOf{\Delta_{i}}\omega'_{min}$ or $\omega^*\worderOf{\Delta_{i'}}\omega'_{min}$.
	\FARBE{Utilizing the same arguments as the $\Rightarrow$ direction of the proof for \cRelG,
	$\omega^*\worderOf{\Delta_{i'}}\omega'_{min}$ is not possible due to \eqref{eq_wrel_3}.}
	Now let $\omega=(\omega^*\marg_{\Sigma_{i}\cup\Sigma_3}\omega'\marg_{\Sigma_{i'}})$. Then also $\omega\marg_{\Sigma_{i'}\cup\Sigma_3}=\omega'\marg_{\Sigma_{i'}\cup\Sigma_3}$ because $\omega\models E$ and $\omega'\models E$. With \eqref{eq_wrel_2} and Lemmata~\ref{lem_syswp_match} and \ref{lem_syswp_match2} we get $\omega\worderDelta\omega'$. Because $\omega\models AEDB$ and $\omega'\models AED\ol{B}$ we are done.
	
	"$\Leftarrow$":
	Assume that $AED\worderDelta B$, i.e., $AEDB\worderDelta AED\ol{B}$. We need to show $AEB\worderOf{\Delta} AE\ol{B}$. Let $\omega'$ be any world with $\omega'\models AE\ol{B}$. Choose $\omega'_{min}$ such that
	\begin{align}
		&\omega'_{min}\models D\label{eq_wind_1},
		\\ &\omega'\marg_{\Sigma_i\cup\Sigma_3}=\omega'_{min}\marg_{\Sigma_i\cup\Sigma_3}\label{eq_wind_2},\text{ and}
		\\
		&\text{there is no }\omega'_{min2} \text{ with } \omega'_{min2}\worderDelta\omega'_{min} \text{ fulfilling \eqref{eq_wind_1} and \eqref{eq_wind_2}.}\label{eq_wind_3}
	\end{align}
	Again such a $\omega'_{min}$ exists because $C\not\equiv\bot$, $\worderDelta$ is irreflexive and transitive and $\Sigma$ is finite. With \eqref{eq_wind_1}, \eqref{eq_wind_2} and $\omega'\models AE\ol{B}$ we have $\omega_{min}'\models AED\ol{B}$. Due to $AEDB\worderDelta AED\ol{B}$ there is then $\omega^*$ with $\omega^*\models AEDB$ and $\omega^*\worderDelta\omega_{min}'$. With Lemma~\ref{lem_syswp_either} again either $\omega^*\worderOf{\Delta_{i}}\omega'_{min}$ or $\omega^*\worderOf{\Delta_{i'}}\omega'_{min}$. 
	\FARBE{Utilizing the same arguments as the $\Rightarrow$ direction of the proof for \cRelG,
	$\omega^*\worderOf{\Delta_{i'}}\omega'_{min}$ is not possible due to \eqref{eq_wind_3}.}
	Now let $\omega=(\omega^*\marg_{\Sigma_{i}\cup\Sigma_3}\omega'\marg_{\Sigma_{i'}})$. Then also $\omega\marg_{\Sigma_{i'}\cup\Sigma_3}=\omega'\marg_{\Sigma_{i'}\cup\Sigma_3}$ because $\omega\models E$ and $\omega'\models E$. With \eqref{eq_wind_3} and Lemmata~\ref{lem_syswp_match} and \ref{lem_syswp_match2} we get $\omega\worderDelta\omega'$. Because $\omega\models AEB$ and $\omega'\models AE\ol{B}$ we are done.}
\end{proof}

\FARBE{Thus, while System~Z satisfies \cRelG\ but not \cIndG, both lexicographic inference and also System~W fully comply with \cSynSplitG. In the next section we will evaluate \FARBE{several inference operators based on a special subclass of ranking functions}.}

\section{Evaluating Inductive Inference Operators based on c-Representations}
\label{sec_cinf}
\FARBE{This section addresses several \iiopnames employing c-representations \citep{KernIsberner2001,KernIsberner2004AMAI}. In Sect.~\ref{ssec_crepprelim}, we present propositions and lemmata regarding c-representations and \FARBE{present the notion of conditional} $\kappa$-independence that will be helpful for the following sections. In Sect.~\ref{ssec_singlecrep}, we show that inference with respect to single c-representations selected by an appropriate strategy \citep{BeierleKernIsberner2021FLAIRS} satisfies generalized conditional syntax splitting. In Sect.~\ref{ssec_ccore}, we show that
\FARBE{c-core closure inference \citep{WilhelmKernIsbernerBeierle2024FoIKScb} fully complies with generalized conditional syntax splitting.}		
In Sect. 
\ref{ssec_cinf}, we show that also inference with respect to all c-representations \citep{BeierleEichhornKernIsbernerKutsch2018AMAI} satisfies generalized conditional syntax splitting.}
\subsection{c-Representations and $\kappa$-Independence}\label{ssec_crepprelim}
\FARBE{
Among the OCF models of \Rdelta, c-representations are special
ranking models obtained by assigning individual integer impacts to
the conditionals in \Rdelta\
and generating 
the world ranks as the sum of impacts of falsified conditionals \citep{KernIsberner2001,KernIsberner2004AMAI}.}

\begin{definition}[c-representation \citep{KernIsberner2001,KernIsberner2004AMAI}]\label{def:c-representation}
	\FARBE{A \emph{c-representation} of}
	$\Rdelta = \{\satzCL{B_1}{A_1},\ldots,\satzCL{B_n}{A_n}\}$
	is an OCF
	$\kappa$ constructed from \FARBE{non-negative impacts} $\kappaiminus{j}\in\mathbb{N}_0$ assigned to each
	$\satzCL{B_j}{A_j}$ such that $\kappa$ accepts \Rdelta\ and is given by: 
	\begin{align}\label{form:def:c-representation}
		\kappa(\omega) 
		=\sum\limits_{\substack{1 \leq j \leq n\\\omega\models A_j\ol{B}_j}}\kappaiminus{j}
	\end{align}
\end{definition}

c-Representations can conveniently be specified using a constraint satisfaction problem
\FARBE{(for detailed explanations, see \citep{KernIsberner2001,KernIsberner2004AMAI}):}
\begin{definition}[\cspR, \citep{KernIsberner2001,BeierleEichhornKernIsbernerKutsch2018AMAI}]
	\label{def_csp_fuer_r}
	\FARBE{The  \emph{constraint satisfaction problem \cspR\ for c-representations of}
	\(
	\Rdelta = \{\satzCL{B_1}{A_1},\ldots,\satzCL{B_n}{A_n}\}
	\)}
	is given by the conjunction of the constraints,
	for all \(j \in \{1,\ldots,n\}\):
	\begin{align}
		\label{eq_kappaiminus_positive}
		&
		\kappaiminus{j} \geq 0
		\\
		\label{eq_kappa_accepts_r_with_kappaiminus}
		&
		\kappaiminus{j}  >  
		\min_{\omega \models A_j B_j}
		\sum_{\substack{k \neq j \\ \omega \models A_k \ol{B_k}}} \kappaiminus{k} 
		- 
		\min_{\omega \models A_j \notB_j}
		\sum_{\substack{k \neq j \\ \omega \models A_k \ol{B_k}}} \kappaiminus{k} 
	\end{align}
\end{definition}

\FARBE{Note that}
(\ref{eq_kappaiminus_positive}) expresses that falsification of conditionals should make worlds \FARBE{not more plausible}, and
(\ref{eq_kappa_accepts_r_with_kappaiminus}) ensures that $\kappa$ as specified by
(\ref{form:def:c-representation}) accepts $\Delta$.
\FARBE{A solution} of \cspR\ is a %
vector
\(
\kappaiminusVektor = (\kappaiminus{1}, \ldots,  \kappaiminus{n})
\)
of
natural numbers.
\solutionsR\ denotes
the set of all solutions of \cspR.
For 
\(
\kappaiminusVektor 
\in  \solutionsR
\) 
and \(\kappa\) as in Equation~\eqref{form:def:c-representation}, \(\kappa\) is the \emph{OCF induced by \kappaiminusVektor} and is denoted by 
\(\induzierteOCF{\kappaiminusVektor}\).
	\cspR\ is sound and complete
	\citep{KernIsberner2001,BeierleEichhornKernIsbernerKutsch2018AMAI}:
	For every \(\kappaiminusVektor \in \solutionsR\), \(\induzierteOCF{\kappaiminusVektor}\) is a c-representation with \(\induzierteOCF{\kappaiminusVektor} \models \Rdelta\), and for every c-representation \(\kappa\) with \(\kappa \models \Rdelta\), there is  \(\kappaiminusVektor \in \solutionsR\) such that \(\kappa = \induzierteOCF{\kappaiminusVektor}\).
	For 
	\FARBE{an impact vector} 
	\kappaiminusVektor, we will simply
	write $\kappaiminusVektorsubj{1}$ and $\kappaiminusVektorsubj{2}$ for the corresponding projections
	$\marg{\kappaiminusVektor}{\Rdelta_1}$ and $\marg{\kappaiminusVektor}{\Rdelta_2}$,
	and $\composeVektor{\kappaiminusVektorsubj{1}}{\kappaiminusVektorsubj{2}}$ for their composition. 
	\FARBE{We illustrate these notions with an example.}
	\begin{example}[$\DeltaBird$ cont.]
		\label{exa_birds}
		For the \BB\ $\DeltaBird$ \FARBE{from Example~\ref{ex:lexinf}},		
		\FARBE{$\cspOf{\DeltaBird}$ contains $ \kappaiminus{i} \geq 0$ for $i \in \{1,2,3,4\}$ \FARBE{as well as} the following \FARBE{constraints}:
			\newlength{\zeilenabstand}
			\setlength{\zeilenabstand}{.0cm}
			\newlength{\abstandklein}
			\setlength{\abstandklein}{.5cm}
			\newlength{\abstandbig}
			\setlength{\abstandbig}{.7cm}
			\begin{align*}
				\kappaiminus{1}  >& \hspace{\abstandbig} 
				\min_{\substack{\omega \in \Omega_\Sigma \\ \omega \models bf}} \sum_{\substack{j \neq 1 \\ \omega \models A_j \ol{B_j}}} \kappaiminus{j} 
				\hspace{\abstandklein} - \hspace{\abstandklein} \min_{\substack{\omega \in \Omega_\Sigma \\ \omega \models b\ol{f}}} \sum_{\substack{j \neq 1 \\ \omega \models A_j \ol{B_j}}} \kappaiminus{j}
				\\[\zeilenabstand]
				\kappaiminus{2}  >&  \hspace{\abstandbig} 
				\min_{\substack{\omega \in \Omega_\Sigma \\ \omega \models p\ol{f}}} \sum_{\substack{j \neq 2 \\ \omega \models A_j \ol{B_j}}} \kappaiminus{j} 
				\hspace{\abstandklein} - \hspace{\abstandklein} \min_{\substack{\omega \in \Omega_\Sigma \\ \omega \models pf}} \sum_{\substack{j \neq 2 \\ \omega \models A_j \ol{B_j}}} \kappaiminus{j}
				\\[\zeilenabstand]
				\kappaiminus{3}  >&  \hspace{\abstandbig} 
				\min_{\substack{\omega \in \Omega_\Sigma \\ \omega \models pb}} \sum_{\substack{j \neq 3 \\ \omega \models A_j \ol{B_j}}} \kappaiminus{j} 
				\hspace{\abstandklein} - \hspace{\abstandklein} \min_{\substack{\omega \in \Omega_\Sigma \\ \omega \models p\ol{b}}} \sum_{\substack{j \neq 3 \\ \omega \models A_j \ol{B_j}}} \kappaiminus{j}
				\\[\zeilenabstand]
				\kappaiminus{4}  >&  \hspace{\abstandbig} 
				\min_{\substack{\omega \in \Omega_\Sigma \\ \omega \models bw}} \sum_{\substack{j \neq 4 \\ \omega \models A_j \ol{B_j}}} \kappaiminus{j} 
				\hspace{\abstandklein} - \hspace{\abstandklein} \min_{\substack{\omega \in \Omega_\Sigma \\ \omega \models b\ol{w}}} \sum_{\substack{j \neq 4 \\ \omega \models A_j \ol{B_j}}} \kappaiminus{j}
		\end{align*}}
		Table \ref{tab_birds} shows some solutions for $\DeltaBird$ as well as their corresponding \FARBE{induced} c-representations. For example $\kappaiminusVektor_1=(1,2,2,1)\in\solutionsRof{\DeltaBird}$,   $\kappaiminusVektor_1^1=(1,2,2)\in\solutionsRof{\DeltaIWoJof{1}{3}{\DeltaBird}}$ and $\kappaiminusVektor_1^2=(1)\in\solutionsRof{\DeltaIWoJof{2}{3}{\DeltaBird}}$.
	\end{example}

\begin{table}[t!]
\centering
\addtolength{\spalteAbst}{0mm}
\addtolength{\spalteAbstGr}{0mm}
\addtolength{\spalteAbstGGr}{0mm}
\renewcommand{\verify}{\ensuremath{\textsf{v}}}
\renewcommand{\falsif}{\ensuremath{\textsf{f}}}
\(
\begin{array}{@{}c@{\hspace*{\spalteAbstGr}}c@{\hspace*{\spalteAbst}}c@{\hspace*{\spalteAbst}}c@{\hspace*{\spalteAbst}}c@{\hspace*{\spalteAbstGGr}}c@{\hspace*{\spalteAbstGr}}c@{\hspace*{\spalteAbstGr}}c@{\hspace*{\spalteAbstGr}}c@{}}
\toprule
\omega
&  
\begin{array}{c}
\condLabel{\delta_1}\\
\satzCL{f}{b}
\end{array}
&  
\begin{array}{c}
\condLabel{\delta_2}\\
\satzCL{\ol{f}}{p}
\end{array}
&  
\begin{array}{c}
\condLabel{\delta_3}\\
\satzCL{b}{p}
\end{array}
&  
\begin{array}{c}
\condLabel{\delta_4}\\
\satzCL{w}{b}
\end{array}
&
\begin{array}{c}
\textrm{impact}\\ \textrm{ on }\omega
\end{array}
&
\begin{array}{c}
 \induzierteOCF{\kappaiminusVektor_1}\\(\omega)
\end{array}
&
\begin{array}{c}
 \induzierteOCF{\kappaiminusVektor_2}\\(\omega)
\end{array}
&
\begin{array}{c}
 \induzierteOCF{\kappaiminusVektor_3}\\(\omega)
\end{array}
\\
\midrule
\omegaAbpfw  & \verify & \falsif & \verify & \verify  & \kappaiminus{2}                    & 2    & 4   & 5   
\\
\omegaBbpfw  & \verify & \falsif & \verify & \falsif  & \kappaiminus{2} + \kappaiminus{4}  & 3    & 7   & 12
\\
\omegaCbpfw  & \falsif & \verify & \verify & \verify  & \kappaiminus{1}                    & 1    & 3   & 4
\\
\omegaDbpfw  & \falsif & \verify & \verify & \falsif  & \kappaiminus{1} + \kappaiminus{4}  & 2    & 6   & 11
\\%
\omegaEbpfw  & \verify & \neutra & \neutra & \verify  & 0                                  & 0    & 0   & 0
\\
\omegaFbpfw  & \verify & \neutra & \neutra & \falsif  & \kappaiminus{4}                    & 1    & 3   & 7
\\
\omegaGbpfw  & \falsif & \neutra & \neutra & \verify  & \kappaiminus{1}                    & 1    & 3   & 4
\\
\omegaHbpfw  & \falsif & \neutra & \neutra & \falsif  & \kappaiminus{1} + \kappaiminus{4}  & 2    & 6   & 11
\\%
\omegaIbpfw  & \neutra & \falsif & \falsif & \neutra  & \kappaiminus{2} + \kappaiminus{3}  & 4    & 8   & 11
\\
\omegaJbpfw  & \neutra & \falsif & \falsif & \neutra  & \kappaiminus{2} + \kappaiminus{3}  & 4    & 8   & 11
\\
\omegaKbpfw  & \neutra & \verify & \falsif & \neutra  & \kappaiminus{3}                    & 2    & 4   & 6
\\
\omegaLbpfw  & \neutra & \verify & \falsif & \neutra  & \kappaiminus{3}                    & 2    & 4   & 6
\\%
\omegaMbpfw  & \neutra & \neutra & \neutra & \neutra  & 0                                  & 0    & 0   & 0
\\
\omegaNbpfw  & \neutra & \neutra & \neutra & \neutra  & 0                                  & 0    & 0   & 0
\\
\omegaObpfw  & \neutra & \neutra & \neutra & \neutra  & 0                                  & 0    & 0   & 0
\\
\omegaPbpfw  & \neutra & \neutra & \neutra & \neutra  & 0                                  & 0    & 0   & 0
\\
\midrule
\textrm{impacts:}  & \kappaiminus{1}   &  \kappaiminus{2} & \kappaiminus{3}   &  \kappaiminus{4}      
\\                         %
\kappaiminusVektor_1  &        1          &        2         &         2         &        1        
\\ 
\kappaiminusVektor_2  &        3          &        4         &         4         &        3         
\\
\kappaiminusVektor_3  &        4          &        5         &         6         &        7         
 \\
\bottomrule
\end{array}
\)
\caption{Verification and falsification with induced impacts for $\DeltaBird$ 
in Example~\protect\ref{exa_birds}. \FARBE{The impact vectors $\protect\kappaiminusVektor_1, \protect\kappaiminusVektor_2$, and $\protect\kappaiminusVektor_3$ are solutions of $\protect\CR(\DeltaBird)$ and $\protect\induzierteOCF{\protect\kappaiminusVektor_1},\protect\induzierteOCF{\protect\kappaiminusVektor_2},\protect\induzierteOCF{\protect\kappaiminusVektor_3}$ are their induced ranking functions according to Definition~\protect\ref{def:c-representation}.}}
\label{tab_birds}
\end{table}
 	
	A fundamental property of c-representations \FARBE{is that for any syntax splitting} \FARBE{\(\kb = \kb_1 \splitcup \kb_2\)}
	the composition of \emph{any}
	impact vectors for \FARBE{the} subbases %
	yields an impact vector for \Rdelta,
	and vice versa \citep{KernIsbernerBeierleBrewka2020KR}.
	\FARBE{This property was also shown to extend to \FARBE{safe} conditional syntax splittings \citep{BeierleSpiegelHaldimannWilhelmHeyninckKernIsberner2024KR}. However, a key part in showing this was \Cref{lem_selffulfilling} which no longer holds for generalized safe conditional syntax splittings. Indeed the \FARBE{composition} property no longer holds, as the impacts assigned to the conditionals in $\Delta_3$ can be vastly different between the two subbases. Thus, we show a slightly weaker property here. While it still states that any impact vector for $\Delta$ can be split into impact vectors for the subbases, impact vectors for the subbases may only yield an impact vector for $\Delta$ if  they match on the \FARBE{impacts} assigned to the conditionals \FARBE{in $\Delta_3$}.}

	\FARBE{
	}
		\begin{propositionrep}
			\label{prop_solprop}
			Let \genSafeCondSynSplitGen. The following two properties hold \FARBE{for $i, i'\in\{1,2\}, i\neq i'$:}		
			\begin{itemize}
				\item For every $\kappaiminusVektor\in\solutionsRof{\Delta}$ there are $\kappaiminusVektorsubjAlt{i}\in\solutionsRof{\DeltaI}$, $\kappaiminusVektorsubjAlt{i'}\in\solutionsRof{\DeltaJ}$ and $\kappaiminusVektorsubjAlt{3}\in\solutionsRof{\DeltaThree}$ with $\kappaiminusVektorsubjAlt{i}|_{\DeltaThree}=\kappaiminusVektorsubjAlt{i'}|_{\DeltaThree} = \kappaiminusVektorsubjAlt{3}$, such that $\kappaiminusVektor=(\kappaiminusVektorsubjAlt{i}|_{\DeltaIWoThree},\kappaiminusVektorsubjAlt{i'}|_{\DeltaJWoThree},\kappaiminusVektorsubjAlt{3})$\FARBE{, i.e., $\kappaiminusVektor|_{\Delta_{i}}=\kappaiminusVektorsubjAlt{i}$ and $\kappaiminusVektor|_{\Delta_{i'}}=\kappaiminusVektorsubjAlt{i}$.}
				
				\item For every $\kappaiminusVektorsubjAlt{i}\in\solutionsRof{\DeltaI},  \kappaiminusVektorsubjAlt{i'}\in\solutionsRof{\DeltaJ}$ and $\kappaiminusVektorsubjAlt{3}\in\solutionsRof{\DeltaThree}$ with $\kappaiminusVektorsubjAlt{i}|_{\DeltaThree}=\kappaiminusVektorsubjAlt{i'}|_{\DeltaThree} = \kappaiminusVektorsubjAlt{3}$ there is $\kappaiminusVektor\in\solutionsRof{\Delta}$ such that  $\kappaiminusVektor=(\kappaiminusVektorsubjAlt{i}|_{\DeltaIWoThree},\kappaiminusVektorsubjAlt{i'}|_{\DeltaJWoThree},\kappaiminusVektorsubjAlt{3})$\FARBE{, i.e., $\kappaiminusVektor|_{\Delta_{i}}=\kappaiminusVektorsubjAlt{i}$ and $\kappaiminusVektor|_{\Delta_{i'}}=\kappaiminusVektorsubjAlt{i}$.}
				
			\end{itemize}
		\end{propositionrep}
		\begin{proof}
			For both properties we will construct solution vectors $\kappaiminusVektor=(\kappaiminusVektorsubjAlt{i}|_{\DeltaIWoThree},\kappaiminusVektorsubjAlt{i'}|_{\DeltaJWoThree},\kappaiminusVektorsubjAlt{3})$ where $\kappaiminusVektorsubjAlt{i}|_{\DeltaThree}=\kappaiminusVektorsubjAlt{i'}|_{\DeltaThree}=\kappaiminusVektorsubjAlt{3}$, such that $\kappaiminusVektor|_{\DeltaI}=\kappaiminusVektorsubjAlt{i}, \kappaiminusVektor|_{\DeltaJ}=\kappaiminusVektorsubjAlt{\j}$ and $\kappaiminusVektor|_{\DeltaThree}=\kappaiminusVektorsubjAlt{3}$.
			For the first property, we define $\kappaiminusVektorsubjAlt{i}$ as $\kappaiminusVektor|_{\DeltaI}$, $\kappaiminusVektorsubjAlt{\j}$ as $\kappaiminusVektor|_{\DeltaJ}$ and $\kappaiminusVektorsubjAlt{3}$ as $\kappaiminusVektor|_{\DeltaThree}$. Then clearly $\kappaiminusVektorsubjAlt{i}|_{\DeltaThree}=\kappaiminusVektorsubjAlt{i'}|_{\DeltaThree}=\kappaiminusVektorsubjAlt{3}$.
			For the second property we define $\kappaiminusVektor=(\kappaiminusVektorsubjAlt{i}|_{\DeltaIWoThree},\kappaiminusVektorsubjAlt{i'}|_{\DeltaJWoThree},\kappaiminusVektorsubjAlt{3})$.
			Then it is sufficient to show that $\kappaiminusVektor\in\solutionsRof{\Delta}$ iff $\kappaiminusVektorsubjAlt{l}\in\solutionsRof{\Delta_l}$, for $l\in\{1,2,3\}$ \FARBE{for both points of Proposition~\ref{prop_solprop}.}
			
			We show the hard direction $\Leftarrow$ first:
			Assume $\kappaiminusVektorsubjAlt{l}\in\solutionsRof{\Delta_l}$, for $l\in\{1,2,3\}$, towards a contradiction, now assume $\kappaiminusVektor\notin\solutionsRof{\Delta}$.
			Then there is some conditional $(B_j|A_j)\in\Delta$ whose constraints in $\cspRof{\Delta}$ are not satisfied by $\kappaiminusVektor$: 
			\begin{align}
				\label{eq_kappaiminus_positive2}
				&
				\kappaiminus{j} \geq 0
				\\
				\label{eq_kappa_accepts_r_with_kappaiminus2}
				&
				\kappaiminus{j}  >  
				\underbrace{
					\min_{\omega \models A_j B_j}
					\sum_{\substack{k \neq j \\ \omega \models A_k \ol{B_k} \\ (B_k|A_k)\in\Delta}} \kappaiminus{k}
				}_{\Vmin} 
				- 
				\underbrace{
					\min_{\omega \models A_j \notB_j}
					\sum_{\substack{k \neq j \\ \omega \models A_k \ol{B_k} \\ (B_k|A_k)\in\Delta}} \kappaiminus{k} 
				}_{\Fmin}
			\end{align}
			The unsatisfied constraint must either be of form \eqref{eq_kappaiminus_positive2} or \eqref{eq_kappa_accepts_r_with_kappaiminus2}. If it is of form \eqref{eq_kappaiminus_positive2}, then the solution vector for the subbase containing $(B_j|A_j)$ also falsifies this constraint, leading to a contradiction.		
			So the violated constraint is of form \eqref{eq_kappa_accepts_r_with_kappaiminus2}. Assume that $(B_j|A_j)\in\DeltaI$, the case where $(B_j|A_j)\notin\DeltaI$ implies $(B_j|A_j)\in\DeltaJ$ and is then analogous. There must be a corresponding constraint
			\begin{align}
				\label{eq_proof_solprop_12}
				\kappaiminusAlt{j}^{i}  >  
				\underbrace{
					\min_{\omega \models A_j B_j}
					\sum_{\substack{k \neq j \\ \omega \models A_k \ol{B_k} \\ (B_k|A_k)\in\DeltaI}} \kappaiminusAlt{k}^{i} 
				}_{\VminI}
				- 
				\underbrace{
					\min_{\omega \models A_j \notB_j}
					\sum_{\substack{k \neq j \\ \omega \models A_k \ol{B_k} \\ (B_k|A_k)\in\DeltaI}} \kappaiminusAlt{k}^{i} 
				}_{\FminI}
			\end{align} 
			in \cspRof{\Delta_{i}}.
			\FARBE{Since the impact vectors are constructed such that $\kappaiminusVektorsubjAlt{i}|_{\DeltaThree}=\kappaiminusVektorsubjAlt{i'}|_{\DeltaThree}=\kappaiminusVektorsubjAlt{3}$, we have that $\kappaiminusVektor|_{\Delta_{i}}=\kappaiminusVektorsubjAlt{i}=(\kappaiminusVektorsubjAlt{i}|_{\DeltaIWoThree},\kappaiminusVektorsubjAlt{3})$ and $\kappaiminusVektor|_{\Delta_{i'}}=\kappaiminusVektorsubjAlt{\j}=(\kappaiminusVektorsubjAlt{\j}|_{\DeltaJWoThree},\kappaiminusVektorsubjAlt{3})$. This way, by construction of $\kappaiminusVektor$, we have that $\kappaiminus{j}=\kappaiminus{j}^{i}=\mu^i_j$ for all $(B_j|A_j)\in\DeltaI$.
				
			Now we show that, if  \eqref{eq_kappa_accepts_r_with_kappaiminus2} is not satisfied, then 
			\eqref{eq_proof_solprop_12} is not satisfied either.
			Because $(B_j|A_j)\in\DeltaI$, we have that $(B_j|A_j)\in (\cL(\Sigma_i\cup\Sigma_3)|\cL(\Sigma_i\cup\Sigma_3))$. Thus there is some $\omega^i\omega^3$ with $\omega\models A_jB_j$ iff $\omega^i\omega^3\models A_jB_j$.
			Due to the \genSafety\ property we have that any $\omegaI\omega^3$ can be extended by an $\omegaJ$ such that no conditional in $\DeltaJWoThree$ is falsified. 
			This means that the worlds minimizing $\Vmin$ and $\Fmin$ do not falsify any conditional in $\DeltaJWoThree$ and \FARBE{that they falsify as few conditionals in $\DeltaI$ as possible}. 
			Thus these worlds also minimize $\VminI$ and $\FminI$. Therefore, the restrictions imposed on $\kappaiminus{j}$ and $\mu^i_j$ are the same. %
			Since $\kappaiminus{j}=\mu^i_j$ for all $(B_j|A_j)\in\DeltaI$ this means that if \eqref{eq_kappa_accepts_r_with_kappaiminus2} is not satisfied, then \eqref{eq_proof_solprop_12} is also not satisfied. Thus, $\kappaiminusVektorsubjAlt{i}\notin\solutionsRof{\Delta_i}$, contrary to our assumption.}

			The other direction of the proof uses similar arguments.
		\end{proof}

		\FARBE{We give an example illustrating Proposition~\ref{prop_solprop}.}
\FARBE{\begin{example}[$\DeltaBird$ cont.]\label{ex:tweety:cinf}
	\FARBE{Consider $\kappaiminusVektor_1\in\solutionsRof{\DeltaBird}$ with $\kappaiminusVektor_1=(1,2,2,1)$ from Example~\ref{exa_birds}.
	Because $\genSafeCondSynSplitMacro{\Delta}{\{(f|b),(\ol{f}|p),(b|p)\}}{\{(w|b)\}}{\{p,f\}}{\{w\}}{\{b\}}$ and $\Delta_3=\emptyset$ we can obtain the solution $\kappaiminusVektor_1$ for $\DeltaBird$ by combining the solutions $\kappaiminusVektor_1^1=(1,2,2)$ and $\kappaiminusVektor_1^2=(1)$ for $\DeltaBird_1$ and $\DeltaBird_2$ utilizing Proposition~\ref{prop_solprop}. Vice versa, we can also utilize Proposition~\ref{prop_solprop} to split $\kappaiminusVektor_3=(4,5,6,7)$ into $\kappaiminusVektor^1_3=(4,5,6)$ and $\kappaiminusVektor^2_3=(7)$, obtaining solutions for $\DeltaBird_1$ and $\DeltaBird_2$ from a solution for $\DeltaBird$.}
\end{example}}

To show that nonmonotonic reasoning with c-representations satisfies \cSynSplitG\, we %
	employ
	the concept of conditional $\kappa$-independence.
\begin{definition}[conditional $\kappa$-independence \citep{HeyninckKernIsbernerMeyerHaldimannBeierle2023AAAI},\citep{spohn12}]
	\label{def_ocf_condind}
	Let \FARBE{$\Sigma_1,\Sigma_2,\Sigma_3\subseteq \Sigma$ where $\Sigma_1,\Sigma_2$ and $\Sigma_3$ are pairwise disjoint} and
	let $\kappa$ be an OCF. \jhii{$\Sigma_1, \Sigma_2$ are} \FARBE{\emph{conditionally $\kappa$-independent \jhii{given $\Sigma_3$}}}, in symbols $\Ind{\kappa}{\Sigma_1}{\Sigma_2}{\Sigma_3}$,  if for all $\omega^1 \in \Omega(\Sigma_1), \omega^2 \in \Omega(\Sigma_2)$, and $\omega^3 \in \Omega(\Sigma_3)$,  
	\FARBE{it holds that} $\kappa(\omega^1 | \omega^{2} \omega^3) = \kappa(\omega^1 |\omega^3)$.
\end{definition}

\FARBE{Given a c-representation and a safe conditional syntax splitting, the subsignatures defined by this splitting \FARBE{are} $\kappa$-independent \citep{BeierleSpiegelHaldimannWilhelmHeyninckKernIsberner2024KR}. \FARBE{This result is extended to the case of generalized safe conditional syntax splitting in the following proposition; its proof largely follows the proof of \citep[Proposition 26]{BeierleSpiegelHaldimannWilhelmHeyninckKernIsberner2024KR}, but has been adapted in the last few steps to hold also for generalized safe splittings.}}
\begin{propositionrep}%
	\label{prop_crep_condkind}
	Let \FARBE{\genSafeCondSynSplitGen}, and $\kappa$ a c-representation with $\kappa\models\Delta$. Then $\Ind{\kappa}{\Sigma_1}{\Sigma_2}{\Sigma_3}$.
\end{propositionrep}
\begin{proof}
	\FARBE{Let \genSafeCondSynSplitGen.}
	Let $\omega=\omega^1\omega^2\omega^3$ and let $\kappaiminusVektor\in \solutionsOf{\Delta}$ such that $\kappa=\induzierteOCF{\kappaiminusVektor}$. 
	Recall the definition of c-representations \eqref{form:def:c-representation}.
	We can rewrite \eqref{form:def:c-representation} to
	\begin{align}
		\label{eq_semsplit_crep_proof_2}
		\induzierteOCF{\kappaiminusVektor}(\omega) 
		=
		\sum_{\substack{\omega \models A_j \notB_j \\ \satzCL{B_j}{A_j} \in \DeltaOneWoThree}} \!\! \kappaiminussubj{j}{}
		+
		\sum_{\substack{\omega \models A_j \notB_j \\ \satzCL{B_j}{A_j} \in \DeltaTwoWoThree}} \!\! \kappaiminussubj{j}{}
		+
		\sum_{\substack{\omega \models A_j \notB_j \\ \satzCL{B_j}{A_j} \in \Rdelta_3}} \!\! \kappaiminussubj{j}{}
	\end{align}
	By simply adding and subtracting the last sum of \eqref{eq_semsplit_crep_proof_2} we obtain the following equation.
	\begin{align}
		\label{eq_semsplit_crep_proof_3}
		\induzierteOCF{\kappaiminusVektor}(\omega) 
		=
		\sum_{\substack{\omega \models A_j \notB_j \\ \satzCL{B_j}{A_j} \in \DeltaOneWoThree}} \!\! \kappaiminussubj{j}{}
		+
		\sum_{\substack{\omega \models A_j \notB_j \\ \satzCL{B_j}{A_j} \in \DeltaTwoWoThree}} \!\! \kappaiminussubj{j}{}
		+
		\sum_{\substack{\omega \models A_j \notB_j \\ \satzCL{B_j}{A_j} \in \Rdelta_3}} \!\! \kappaiminussubj{j}{}
		+
		\sum_{\substack{\omega \models A_j \notB_j \\ \satzCL{B_j}{A_j} \in \Rdelta_3}} \!\! \kappaiminussubj{j}{}
		-
		\sum_{\substack{\omega \models A_j \notB_j \\ \satzCL{B_j}{A_j} \in \Rdelta_3}} \!\! \kappaiminussubj{j}{}
	\end{align}
	Then we can combine the sums for $\DeltaOneWoThree$ and $\DeltaTwoWoThree$ with the sum for $\Delta_3$ to obtain sums for $\Delta_1$ and $\Delta_2$ respectively.
	\begin{align}
		\label{eq_semsplit_crep_proof_4}
		\induzierteOCF{\kappaiminusVektor}(\omega) 
		=
		\sum_{\substack{\omega \models A_j \notB_j \\ \satzCL{B_j}{A_j} \in \Rdelta_1}} \!\! \kappaiminussubj{j}{}
		+
		\sum_{\substack{\omega \models A_j \notB_j \\ \satzCL{B_j}{A_j} \in \Rdelta_2}} \!\! \kappaiminussubj{j}{}
		-
		\sum_{\substack{\omega \models A_j \notB_j \\ \satzCL{B_j}{A_j} \in \Rdelta_3}} \!\! \kappaiminussubj{j}{}
	\end{align}
	
	Since $\Delta_1$ is in ${\cal L}(\Sigma_1\cup \Sigma_{3})$, $\Delta_2$ is in ${\cal L}(\Sigma_2\cup\Sigma_3)$ and $\Delta_3$ is in ${\cal L}(\Sigma_3)$ we can use \eqref{eq_omega_reduce} to simplify \eqref{eq_semsplit_crep_proof_4}.
	\begin{align}
		\label{eq_cind_crep_proof_5}
		\induzierteOCF{\kappaiminusVektor}(\omega) 
		=
		\sum_{\substack{\omega^1\omega^3\models A_j \notB_j \\ \satzCL{B_j}{A_j} \in \Rdelta_1}} \!\! \kappaiminussubj{j}{}
		+
		\sum_{\substack{\omega^2\omega^3 \models A_j \notB_j \\ \satzCL{B_j}{A_j} \in \Rdelta_2}} \!\! \kappaiminussubj{j}{}
		-
		\sum_{\substack{\omega^3 \models A_j \notB_j \\ \satzCL{B_j}{A_j} \in \Rdelta_3}} \!\! \kappaiminussubj{j}{}
	\end{align}
	\FARBE{
		Due to the \genSafety\ property $\omegaIThree$ can not falsify any conditional in $\Delta\setminus\DeltaI$ for $i,\in{1,2}$. Thus \eqref{eq_cind_crep_proof_5} can be rewritten to 
		\begin{align}
			\label{eq_cind_crep_proof_6}
			\induzierteOCF{\kappaiminusVektor}(\omega) 
			=
			\sum_{\substack{\omega^1\omega^3\models A_j \notB_j \\ \satzCL{B_j}{A_j} \in \Delta}} \!\! \kappaiminussubj{j}{}
			+
			\sum_{\substack{\omega^2\omega^3 \models A_j \notB_j \\ \satzCL{B_j}{A_j} \in \Delta}} \!\! \kappaiminussubj{j}{}
			-
			\sum_{\substack{\omega^3 \models A_j \notB_j \\ \satzCL{B_j}{A_j} \in \Delta}} \!\! \kappaiminussubj{j}{}
		\end{align}
		By applying the definition of c-representations again, we obtain
	}
	\begin{equation}
		\induzierteOCF{\kappaiminusVektor}(\omega^1\omega^2\omega^3)=\induzierteOCF{\kappaiminusVektor}(\omega^1\omega^3)
		+
		\induzierteOCF{\kappaiminusVektor}(\omega^2\omega^3)
		-
		\induzierteOCF{\kappaiminusVektor}(\omega^3)
	\end{equation}
	which is equivalent to $\induzierteOCF{\kappaiminusVektor}\satzCL{\omega^1}{\omega^2\omega^3}=\induzierteOCF{\kappaiminusVektor}\satzCL{\omega^1}{\omega^3}$, completing the proof.
\end{proof}

\begin{toappendix}
\FARBE{The following useful lemma will aid us in proofs.}
\begin{lemma}[\citep{BeierleSpiegelHaldimannWilhelmHeyninckKernIsberner2024KR}]
	\label{lemma_ocf_condint_formula}
	Let %
	$\Sigma_1,\Sigma_2,\Sigma_3\subseteq \Sigma$ where $\Sigma_1,\Sigma_2$ and $\Sigma_3$ are pairwise disjoint 
	and
	let $\kappa$ be an OCF. $\Sigma_1, \Sigma_2$ are conditionally $\kappa$-independent given $\Sigma_3$ iff for all $A \in \cL(\Sigma_1), B \in \cL(\Sigma_2)$ and complete conjunctions $\FARBE{E}\in \cL(\Sigma_3)$ it holds that
	\[\kappa(ABE) = \kappa(AE) + \kappa(BE) -\kappa(E).\]
\end{lemma}
\end{toappendix}
\FARBE{Lemma~\ref{lemma_ocf_condint_formula} allows us to use the arithmetics provided by OCFs to calculate the ranks of formulas over disjoint and $\kappa$-independent subsignatures \FARBE{which we will exploit in the following subsections}.}

\subsection{Inference with Single c-Representations}\label{ssec_singlecrep}
\FARBE{In this section we look at inference with respect to a single c-representation, obtained by assigning one c-representation to each belief base, yielding an OCF-based inductive inference operator.}
\FARBE{For this, it will be useful to introduce \FARBE{an} alternative characterization of \cIndG\ and \cRelG\ for OCF-based inductive inference operators.} 
\FARBE{Corresponding propositions for safe conditional syntax splittings have been \FARBE{introduced by Heyninck et al.} \citep{HeyninckKernIsbernerMeyerHaldimannBeierle2023AAAI}; here we extend \FARBE{them} to generalized safe conditional syntax splittings.}
\begin{proposition}\label{prop:Cind:for:ocfs}
	An inductive inference operator for OCFs ${\bf C}^{ocf}: \Delta\mapsto \kappa_\Delta$ satisfies \cIndG\ \FARBE{if} for any $\Delta=\Delta_1\bigcup^{\sf \FARBE{gs}}_{\Sigma_1,\Sigma_2}\Delta_2\mid \Sigma_3$ we have $\Ind{\kappa_\Delta}{\Sigma_1}{\Sigma_2}{\Sigma_3}$.
\end{proposition}
\FARBE{
\begin{proof}
	Let $\Delta$ be a \BB, $C^{ocf}:\Delta\mapsto \kappa_\Delta$ be an inductive inference operator for OCFs, and let $\genSafeCondSynSplitGen$.
	Assume $\Ind{\kappa_\Delta}{\Sigma_1}{\Sigma_2}{\Sigma_3}$. W.l.o.g. we assume $i=1, i'=2$, the other case is analogous. We need to show that $C^{ocf}$ satisfies \cIndG, i.e., for all $A,B\in\cL_1, D\in\cL_2$ and \FARBE{every complete conjunction} $E\in\cL_3$ we have $\kappa_\Delta(ABE)<\kappa_\Delta(A\ol{B}E)$ iff $\kappa_\Delta(ABDE)<\kappa_\Delta(A\ol{B}DE)$. 
	With Lemma~\ref{lemma_ocf_condint_formula} we have $\kappa_\Delta(ABDE)=\kappa_\Delta(ABE)+\kappa_\Delta(DE)-\kappa_\Delta(E)$. Then clearly $\kappa_\Delta(ABE)<\kappa_\Delta(A\ol{B}E)$ implies $\kappa_\Delta(ABDE)<\kappa_\Delta(A\ol{B}DE)$. On the other hand we can rearrange \FARBE{$\kappa_\Delta(ABDE)=\kappa_\Delta(ABE)+\kappa_\Delta(DE)-\kappa_\Delta(E)$} to $\kappa_\Delta(ABE)=\kappa_\Delta(ABDE)-\kappa_\Delta(DE)+\kappa_\Delta(E)$. Thus $\kappa_\Delta(ABDE)<\kappa_\Delta(A\ol{B}DE)$ implies $\kappa_\Delta(ABE)<\kappa_\Delta(A\ol{B}E)$ and we are done.
\end{proof}

Thus, an inductive inference operator for OCFs satisfies generalized conditional independence if the subsignatures of any generalized safe conditional syntax splitting are $\kappa$-independent \wrt\ the conditional pivot.
}
\begin{proposition}\label{prop:Crel:for:ocfs}
	An inductive inference operator for OCFs ${\bf C}^{ocf}: \Delta\mapsto \kappa_\Delta$ satisfies \cRelG\ if for any $\Delta=\Delta_1\bigcup^{\sf gs}_{\Sigma_1,\Sigma_2}\Delta_2\mid \Sigma_3$ and $i \in \{1,2\}$
	we have $\kappa_{\Delta_i}=\kappa_{\Delta}\marg_{\Sigma_i\cup\Sigma_3}$.
\end{proposition}
\FARBE{
	\begin{proof}
		Let $\Delta$ be a \BB, $C^{ocf}:\Delta\mapsto \kappa_\Delta$ be an inductive inference operator for OCFs, and let $\genSafeCondSynSplitGen$.
		Assume $\kappa_{\Delta_i}=\kappa_{\Delta}\marg_{\Sigma_i\cup\Sigma_3}$. We need to show that $C^{ocf}$ satisfies $\cRelG$, i.e., for $i\in\{1,2\}$, for all $A,B\in\cL_i$ and \FARBE{every} complete conjunction $E\in\cL_3$ we have $\kappa_\Delta(ABE)<\kappa_\Delta(A\ol{B}E)$ iff $\kappa_{\Delta_i}(ABE)<\kappa_{\Delta_i}(A\ol{B}E)$.
		Since $ABE\in\mathcal{L}_{\Sigma_i\cup\Sigma_3}$ we have that $\kappa_{\Delta}(ABE)=\kappa_{\Delta}\marg_{\Sigma_i\cup\Sigma_3}(ABE)$ which means $\kappa_{\Delta}(ABE)=\kappa_{\Delta_i}(ABE)$ \FARBE{because $\kappa_{\Delta_i}=\kappa_{\Delta}\marg_{\Sigma_i\cup\Sigma_3}$ per our assumption}.
	\end{proof}

Thus, an inductive inference operator for OCFs satisfies generalized conditional relevance if, for every generalized safe conditional syntax splitting, the marginalization of the operator's image \FARBE{of a conditional \BB} to the language of one subbase \FARBE{coincides with applying the operator to that subbase directly}.
}

We will now
define model-based inductive inference operators assigning a c-representation
$\kappa$ to each $\Delta$,
\FARBE{by employing the concept of selection strategies.}
\begin{definition}[selection strategy \selectionStrategy, \citep{BeierleKernIsberner2021FLAIRS}]
	\label{def_selection_strategy}
	A \emph{selection strategy (for c-representations)} is a function
	\FARBE{\(
	\selectionStrategy:  \Delta \mapsto \kappaiminusVektor
	\)}
	assigning to each conditional belief base \Rdelta\ an impact vector
	$\kappaiminusVektor \in \solutionsR$.
\end{definition}

\FARBE{Each selection strategy \FARBE{yields} an inductive inference operator} \FARBE{%
	\(
	\indreascrepSelect:  \Delta \mapsto %
	\induzierteOCF{\selectionStrategy(\Delta)}
	\)
	where %
$\nmableit_{\!\!\induzierteOCF{\selectionStrategy(\Delta)}}$ 
is obtained via Equation (\ref{eq_nmocf}) from $\induzierteOCF{\selectionStrategy(\Delta)}$. }
\FARBE{Note that 
\FARBE{\indreascrepSelect\ is %
an inductive inference operator because} 
each $\nmableit_{\!\!\induzierteOCF{\selectionStrategy(\Delta)}}$} satisfies both (Direct Inference) and (Trivial Vacuity).
A recent example for a specific selection strategy are \emph{minimal core c-representations} \citep{WilhelmKernIsbernerBeierle2024FoIKScb} \FARBE{which we will investigate in \FARBE{Section~\ref{ssec_ccore}}}.

In principle, for every $\Delta$, a selection strategy may choose some
impact vector
independently from the choices for all other belief bases.
\FARBE{The following property generalizes a corresponding postulate \FARBE{(IP-CSP)} for safe conditional splittings~\citep{BeierleSpiegelHaldimannWilhelmHeyninckKernIsberner2024KR} and characterizes selection strategies that preserve the impacts
chosen for subbases of a generalized safe conditional syntax splitting}.
\begin{description}
	\item[ \ipCSPG ] %
	A selection strategy  \selectionStrategy\
	is \emph{impact preserving \wrt\ generalized safe conditional syntax splitting} if, for every %
	\(\genSafeCondSynSplitGen\),
	for $i \in \{1,2\}$, we have
	$\selectionStrategy(\Rdelta_i) =  \marg{\selectionStrategy(\Rdelta)}{\Rdelta_i}$
.
\end{description}

\FARBE{It has been shown that any inductive inference operator based on a selection strategy that is impact reserving according to (IP-CSP) satisfies \cSynSplit\  \citep{BeierleSpiegelHaldimannWilhelmHeyninckKernIsberner2024KR}; we extend this result to \ipCSPG\ and \cSynSplitG.}
\begin{propositionrep}
	\label{prop_selstrat_csynsplit}
	Let $\selectionStrategy$ be a selection strategy \FARBE{satisfying} \FARBE{\ipCSPG}. Then \indreascrepSelect\ satisfies 
	\FARBE{\cRelG\ and  \cIndG\ \FARBE{and thus
		\cSynSplitG.}}
\end{propositionrep}
\begin{proof}
		\FARBE{The proof is obtained by adapting the proof of the proposition for (IP-CSP) and (CSynSplit) \citep[Proposition 27]{BeierleSpiegelHaldimannWilhelmHeyninckKernIsberner2024KR} by observing that the prerequisites of the steps in the proof are also satisfied by generalized safe conditional syntax splittings.}
\end{proof}
\FARBE{Thus, inference based on a single c-representation satisfies \cSynSplitG\ if the underlying selection strategy satisfies \ipCSPG.} \FARBE{In the next section we give an example of an inference operator based on a specific selection strategy.}

\FARBE{
\subsection{c-Core closure Inference}\label{ssec_ccore}
A special subclass of c-representations are \emph{core c-representations} \citep{WilhelmKernIsbernerBeierle2024FoIKScb}. Core c-representations \FARBE{stand out} from the class of all c-representations in the fact that each strongly consistent \BB\ always has a uniquely determined minimal core c-representation. \FARBE{Choosing this minimal core c-representation yields} an OCF-based inductive inference operator via Equation~\eqref{eq_nmocf}.
The definition of core c-representations makes use of a constraint reduction system in the form of transformation rules, simplifying the set of constraints $\CR(\Delta)$ without altering it's solutions. To express these rules \FARBE{compactly, an alternative notation of the constraint system $\CR(\Delta)$ is used by employing}, for each conditional $\satzCL{B_i}{A_i} \in \Delta$, the following sets of sets of verified and falsified conditionals \citep{BeierleKutschSauerwald2019AMAIcompilation}:  
\begin{align}
	\label{constraint-inducing_set_V}
	\VSet_i=\{\{(B_j|A_j)\in\Delta\setminus\{\delta_i\}\mid \omega\models A_j\ol{B_j}\}\mid \omega\in\ver(B_i|A_i)\}\\
	\label{constraint-inducing_set_F}
	\FSet_i=\{\{(B_j|A_j)\in\Delta\setminus\{\delta_i\}\mid \omega\models A_j\ol{B_j}\}\mid \omega\in\fal(B_i|A_i)\}
\end{align}
\FARBE{Employing the constraint-inducing sets \eqref{constraint-inducing_set_V} and \eqref{constraint-inducing_set_F}, the}
constraint satisfaction problem $\CR(\Delta) = \{C_1, \ldots, C_n\}$ can then be specified as follows for \FARBE{all
	$(B_i|A_i)\in\Delta$:} 
\begin{equation}
	\label{eq_crep_vf}
	C_i: \eta_i>\min\{\sum_{\delta_j\in S} \eta_j\mid S\in \VSet_i\} - \min\{\sum_{\delta_j\in S} \eta_j\mid S\in \FSet_i\}
\end{equation}
The positive and negative parts of \eqref{eq_crep_vf} correspond to the positive and negative parts of \eqref{form:def:c-representation}.

Figure~\ref{fig:opt_transformationrules} shows the set $\{R1,\dots, R6\}$ of transformation rules for simplifying $\CR(\Delta)$ employed for the definition of core c-representations \citep{WilhelmKernIsbernerBeierle2024FoIKScb}. These rules were first given in \citep{BeierleKutschSauerwald2019AMAIcompilation} for speeding up the computation of c-representations \FARBE{and of c-inference \citep{BeierleEichhornKernIsbernerKutsch2018AMAI}} and then extended in \citep{BeierleHaldimannKernIsberner2021Boeger75Festschrift,WilhelmSezginKernIsbernerHaldimannBeierleHeyninck2023JELIA}. Since they only make use of general arithmetic properties of the minimum, they do not change the set of solutions $\solutionsR$. Furthermore, $\{R1,\dots, R6\}$ is terminating and confluent; thus, exhaustive application of $\{R1,\dots, R6\}$ yields a uniquely determined constraint system. \FARBE{For a constraint system $\CR$, a constraint $C$, and constraint inducing sets $\VSet$ and $\FSet$, we denote with $\hat{\CR}, \hat{C},  \hat{\VSet}$, and $\hat{\FSet}$ the constraint system, the constraint, and the constraint inducing sets, respectively, after applying $\{R1,\dots, R6\}$ exhaustively.}

\begin{figure}[tb]
\newlength{\abstandInTabelle}
\newcommand{\calD}{\ensuremath{\mathcal{D}}}
\setlength{\abstandInTabelle}{6mm}
\renewcommand{\nsrV}{\VSet}
\renewcommand{\nsrF}{\FSet}
\[
\begin{array}{l@{\,\,}c@{\quad}l}
\displayTransformationRule%
{\refSSVker\ \nameSSV:}%
{\displaystyle\frac{\nsrPaarKER{\nsrV \cup\{S, S'\}}{\nsrF}{i}}%
{\nsrPaarKER{\nsrV \cup\{S\}}{\nsrF}{i}}}%
{S \subsetneq S'}
\\[\abstandInTabelle]

\displayTransformationRule%
{\refSSFker\ \nameSSF:}%
{\displaystyle\frac{\nsrPaarKER{\nsrV}{\nsrF \cup\{S, S'\}}{i}}%
{\nsrPaarKER{\nsrV}{\nsrF \cup\{S\}}{i}}}%
{S \subsetneq S'}
\\[\abstandInTabelle]

\displayTransformationRule%
{\refELMker\ \nameELM:}%
{\displaystyle\frac{\nsrPaarKER{\left\{V_1 \cup \{\delta\},\ldots, V_p \cup \{\delta\}\right\}}{\left\{F_1 \cup \{\delta\},\ldots, F_q \cup \{\delta\}\right\}}{i}}%
		{\nsrPaarKER{\left\{V_1,\ldots,V_p\right\}}{\left\{F_1,\ldots,F_q\right\}}{i}}}%
{}
\\[\abstandInTabelle]

\displayTransformationRule%
{\refTRIVIAL\ \nameTRIVIAL:}%
{\displaystyle\frac{\nsrPaarKER{\nsrV}{\nsrF}{i}}%
{\nsrPaarKER{\{\leereMenge\}}{\{\leereMenge\}}{i}}}%
{\nsrV = \nsrF}
\\[\abstandInTabelle]

\displayTransformationRule%
{\refSUBSETS\ \nameSUBSETS:}%
{\displaystyle\frac{\nsrPaarKER{\left\{S_1 \dotcup T,\ldots, S_p \dotcup T\right\}}{\left\{S_1 \dotcup T',\ldots, S_p \dotcup T'\right\}}{i}}%
		{\nsrPaarKER{\left\{T\right\}}{\left\{T'\right\}}{i}}}%
{}
\\[\abstandInTabelle]

\displayTransformationRule%
{\refCIRCLE\ \nameCIRCLE:}%
{\displaystyle\frac{\nsrPaarKER{\calD \dotcup\{\{\delta_j\}\}}{\{\leereMenge\}}{i} \quad {\nsrPaarKER{\calD \dotcup\{\{\delta_i\}\}}{\{\leereMenge\}}{j}}}%
{\nsrPaarKER{\calD}{\{\leereMenge\}}{i} \quad \nsrPaarKER{\calD}{\{\leereMenge\}}{j}}}
{i \neq j}
\\[\abstandInTabelle]

\end{array}
\]
\caption{Transformation rules \FARBE{$\{R1,\dots, R6\}$} for simplifying \FARBE{(the constraint-inducing sets of)}
$\CR(\Delta)$. A pair~\(\nsrPaarKER{\nsrV}{\nsrF}{i}\) represents the sets of constraint variables in the minimum expressions associated to the verification and the falsification, respectively, of the $i$-th conditional
$\delta_i \in \Delta$
in the constraint $C_i \in \CR(\Delta)$ modeling the acceptance condition of $\delta_i$.} 
\label{fig:opt_transformationrules}
\end{figure}
 
\begin{definition}[Core c-Representation \citep{WilhelmKernIsbernerBeierle2024FoIKScb}]
	\label{def:core_c-representation}
	Let $\Delta$ be a belief base, let $i\in\{1,\dots, n\}$ and let $\CR^+(\Delta) = \{\hat{C}_1^+, \ldots, \hat{C}_n^+ \}$ where
	\begin{equation}
		\label{eq:core_c-representation}
		\hat{C}^+_i \colon \quad
		\eta_i 
		> \min \{ \sum_{\delta_j \in S} \eta_j 
		\mid S \in \hat{V}_i \}
	\end{equation}
	If $\kappaiminusVektor \in \mathbb{N}_0^n$ is a solution of $\CR^+(\Delta)$, \FARBE{then the c-representation $\induzierteOCF{\kappaiminusVektor}$ determined from $\kappaiminusVektor$ via equation \eqref{form:def:c-representation} is called a \emph{core c-representation} of $\Delta$.}
\end{definition}
Thus core c-representations are defined by first applying the transformation rules $\{R_1,\dots,R_6\}$ \FARBE{exhaustively} and then focusing only on the positive part of the reduced constraint system. 
\FARBE{While in general $\CR(\Delta)$ can have different pareto-minimal solutions, $\CR^+(\Delta)$ always has a unique pareto-minimal solution $\kappaiminusVektor^{mc}_\Delta$ \citep{WilhelmKernIsbernerBeierle2024FoIKScb}. The c-representation $\kappa^{mc}_\Delta=\induzierteOCF{\kappaiminusVektor^{mc}_\Delta}$ induced by $\kappaiminusVektor^{mc}_\Delta$ is called the \emph{minimal core c-representation} of $\Delta$ \citep{WilhelmKernIsbernerBeierle2024FoIKScb}.}
\FARBE{A method for constructing the minimal core c-representation involving multiple other concepts, such as generalized tolerance partitions and the notion of base functions is given in \citep{WilhelmKernIsbernerBeierle2024FoIKScb}. According to \citep{WilhelmKernIsbernerBeierle2025KERlocal}, $\kappa^{mc}_\Delta$ can be characterized as \FARBE{in the following proposition}.}
\begin{proposition}[\citep{WilhelmKernIsbernerBeierle2025KERlocal}, adapted]
	\label{def:minimal_core_c-representation}
	\FARBE{Let $\Delta$ be a belief base 
	and let
	$\kappaiminusVektor^{mc}_\Delta = (\eta^{mc}_1,\ldots, \eta^{mc}_n)$ be the impact vector inducing the minimal core c-representation of $\Delta$. %
	Then}
	\begin{align}
		\label{eq_min_ccore}
		\eta^{mc}_i
		= \min \{ \sum_{\delta_j \in S} \eta^{mc}_j \mid S \in \hat{V}_i \} + 1.
	\end{align}
\end{proposition}
\begin{proof}
	\FARBE{The proof is obtained %
	by applying a result from Wilhelm et. al.  \citep[Proposition~9]{WilhelmKernIsbernerBeierle2025KERlocal} to the original definition of core c-representations \citep[Definition~11]{WilhelmKernIsbernerBeierle2025KERlocal}}.
\end{proof}
\FARBE{Crucial to Proposition~\ref{def:minimal_core_c-representation} is the fact that $\kappaiminusVektor^{mc}_\Delta$ can be computed in a stratified manner such that the right-hand-side of \eqref{eq_min_ccore} only mentions those impacts that have already been determined in a previous stratum; for details we refer to \citep{WilhelmKernIsbernerBeierle2025KERlocal}.}
	
Because the minimal core c-representation is
\FARBE{uniquely determined %
\citep{WilhelmKernIsbernerBeierle2024FoIKScb}, it}
yields the OCF-based inductive inference operator \emph{c-core closure}.

\begin{definition}[c-Core closure \citep{WilhelmKernIsbernerBeierle2024FoIKScb}]
	\label{def:c-core_closure}
	Let $\Delta$ be a belief base, and let $A, B \in \mathcal{L}(\Sigma)$. The \emph{c-core closure} inference operator $C^{mc} \colon \Delta \mapsto \nmableit^{mc}_{\!\Delta\,}$ is defined by $A \nmableit^{mc}_{\!\Delta\,} B$ iff $A
	\nmableit_{\!\kappa^{mc}_\Delta\,} B$ where $\kappa^{mc}_\Delta$ is the minimal core c-representation of $\Delta$.
\end{definition}

We illustrate these notions with an example.

\begin{example}[$\DeltaBird$ cont.]
	\label{exa_corec}
	Table~\ref{tab_corec} shows the sets $\VSet_i$ and $\FSet_i$ and their reductions $\hat{\VSet_i}$ and $\hat{\FSet_i}$ of $\DeltaBird$. 
	\FARBE{Thus $\CR^+(\DeltaBird)$ consists of the following constraints:
	\begin{align*}
		&\hat{C}_1^+: \eta_1>0& &\hat{C}_2^+: \eta_2>\eta_1\\
		&\hat{C}_3^+: \eta_3>\eta_1& &\hat{C}_4^+: \eta_4>0
	\end{align*}
	To compute the minimal core c-representation $\kappa^{mc}_{\DeltaBird}$ of $\DeltaBird$, we need to find the pareto-minimal solution $\kappaiminusVektor^{mc}_{\!\DeltaBird}=(\eta^{mc}_1,\eta^{mc}_2,\eta^{mc}_3,\eta^{mc}_4)$ of $\CR^+(\DeltaBird)$. By utilizing Proposition~\ref{def:minimal_core_c-representation}, we obtain
	\begin{align*}
		&\eta_1^{mc}=1& &
		\eta_2^{mc}=\eta_1^{mc}+1\\
		&\eta_3^{mc}=\eta_1^{mc}+1& &
		\eta_4^{mc}=1.
	\end{align*}
	Now $\kappaiminusVektor^{mc}_{\!\DeltaBird}$ can be computed in a stratified manner: First, we obtain $\eta^{mc}_1=1$ and $\eta^{mc}_4=1$ immediately. Then we can determine $\eta^{mc}_2=2$ and $\eta^{mc}_2=2$, yielding $\kappaiminusVektor^{mc}_{\!\DeltaBird}=(1,2,2,1)$. Note that $\eta^{mc}_2$ and $\eta^{mc}_3$ can be determined by taking only the previously computed impacts $\eta^{mc}_1$ and $\eta^{mc}_4$ into account.
	The minimal core c-representation 
	$\induzierteOCF{\DeltaBird}^{mc}$ of $\DeltaBird$ is then given by $\induzierteOCF{\DeltaBird}^{mc}=\induzierteOCF{\kappaiminusVektor^{mc}_{\!\DeltaBird}}$ and can be seen in Table~\ref{tab_birds}, where $\kappaiminusVektor^{mc}_{\!\DeltaBird}=\kappaiminusVektor_1$.} We have $\induzierteOCF{\DeltaBird}^{mc}(pbw)=1$ and $\induzierteOCF{\DeltaBird}^{mc}(pb\ol{w})=2$ and thus $pb\nmableit^{mc}_{\!\DeltaBird\,} w$.
\end{example}

\begin{table}[tb]
	\centering
	\(
	\begin{array}{c|c|c|c|c}
		&
		\VSet_i
		&
		\hat{\VSet_i}
		&
		\FSet_i
		&
		\hat{\FSet_i}
		\\
		\hline
		\begin{array}{c}
			\condLabel{\delta_1}
			\satzCL{f}{b}
		\end{array}
		&
		\{\emptyset,\{\delta_2\},\{\delta_4\},\{\delta_2,\delta_4\}\}
		&
		\{\emptyset\}
		&
		\{\emptyset,\{\delta_4\}\}
		&
		\{\emptyset\}
		\\
		\begin{array}{c}
			\condLabel{\delta_2}
			\satzCL{\ol{f}}{p}
		\end{array}
		&
		\{\{\delta_1\},\{\delta_3\},\{\delta_1,\delta_4\}\}
		&
		\{\{\delta_1\}\}
		&
		\{\emptyset,\{\delta_3\},\{\delta_4\}\}
		&
		\{\emptyset\}
		\\
		\begin{array}{c}
			\condLabel{\delta_3}
			\satzCL{b}{p}
		\end{array}
		&
		\{\{\delta_1\},\{\delta_2\},\{\delta_1,\delta_4\},\{\delta_2,\delta_4\}\}
		&
		\{\{\delta_1\}\}
		&
		\{\emptyset,\{\delta_2\}\}
		&
		\{\emptyset\}
		\\
		\begin{array}{c}
			\condLabel{\delta_4}
			\satzCL{w}{b}
		\end{array}
		&
		\{\emptyset,\{\delta_1\},\{\delta_2\}\}
		&
		\{\emptyset\}
		&
		\{\emptyset,\{\delta_1\},\{\delta_2\}\}
		&
		\{\emptyset\}
	\end{array}
	\)
	\caption{Sets $\VSet_i$ and $\FSet_i$ and their reductions $\hat{\VSet_i}$ and $\hat{\FSet_i}$ for $\DeltaBird$ 
		in Example~\protect\ref{exa_corec}.}
	\label{tab_corec}
\end{table}

\FARBE{In order to show that c-core closure fully complies with generalized conditional syntax splitting, we first first show that the selection strategy assigning to each belief base $\Delta$ the impact vector $\kappaiminusVektor^{mc}_{\Delta}$ satisfies \ipCSPG.}
\FARBE{
\begin{proposition}
	\label{prop_coreselstrat_ipcspg}
	The selection strategy $\sigma^{mc}:\Delta\rightarrow \kappaiminusVektor^{mc}_{\Delta}$ satisfies \ipCSPG.
\end{proposition}
\begin{proof}
	Let $\sigma^{mc}$ be the selection strategy assigning to each \BB\ $\Delta$ \FARBE{the impact vector $\kappaiminusVektor^{mc}_{\Delta}$, yielding the minimal core c-representation $\kappa^{mc}_\Delta=\induzierteOCF{\kappaiminusVektor^{mc}_{\Delta}}$ of $\Delta$}.	
	Let $\Delta=\{(B_1|A_1),\dots,(B_n|A_n)\}$ with \genSafeCondSynSplitGen\ and let $i,i'\in\{1,2\}, i\neq i'$.
	Let $\kappa^{mc}_\Delta$ be the minimal core c-representation of $\Delta$ based on the impact vector $\kappaiminusVektor^{mc}_{\Delta}$, i.e., $\sigma^{mc}(\Delta)=\kappaiminusVektor^{mc}_{\Delta}$ and $\induzierteOCF{\kappaiminusVektor^{mc}_{\Delta}} = \kappa^{mc}_\Delta$.
	We need to show $\sigma^{mc}(\Delta)\marg_{\Delta_{i}}=\sigma^{mc}(\Delta_i)$, i.e., $\kappaiminusVektor^{mc}_{\Delta}\marg_{\Delta_{i}} = \kappaiminusVektor^{mc}_{{\Delta_i}}$. 
	
	First we show $\hat{\CR}(\Delta)\marg_{\Delta_{i}}=\hat{\CR}(\Delta_i)$.
	Consider the constraint $C_j$  for the conditional $(B_j|A_j)\in\Delta$:
	\begin{equation}
		\label{eq_ccore_proof1}
		C_j:
		\kappaiminus{j}  >  
		\min_{\omega \models A_j B_j}
		\sum_{\substack{k \neq j \\ \omega \models A_k \ol{B_k} \\ (B_k|A_k)\in\Delta}} \kappaiminus{k} 
		- 
		\min_{\omega \models A_j \notB_j}
		\sum_{\substack{k \neq j \\ \omega \models A_k \ol{B_k}\\ (B_k|A_k)\in\Delta}} \kappaiminus{k} 
	\end{equation}	
	Assume now $(B_j|A_j)\in\Delta_i$ and consider the constraint $C_j^i$ corresponding to $C_j$:
	\begin{equation}
		\label{eq_ccore_proof3}
		C_j^i:
		\kappaiminus{j}  >  
		\min_{\omega \models A_j B_j}
		\sum_{\substack{k \neq j \\ \omega \models A_k \ol{B_k} \\ (B_k|A_k)\in\Delta_i}} \kappaiminus{k} 
		- 
		\min_{\omega \models A_j \notB_j}
		\sum_{\substack{k \neq j \\ \omega \models A_k \ol{B_k}\\ (B_k|A_k)\in\Delta_i}} \kappaiminus{k} 
	\end{equation}	
	We show that $\hat{C_j}$ and $\hat{C_j^i}$ are equivalent.
	We have $\omega\models A_j\ol{B_j}$ iff $\omega^i\omega^3\models A_j\ol{B_j}$, and $\omega\models A_jB_j$ iff $\omega^i\omega^3\models A_jB_j$. Then due to the generalized safety property, for each $\omega$ with $\omega\models A_j\ol{B_j}$ there is $\omega_2$ with $\omega^i\omega^3=\omega_2^i\omega_2^3$ such that $\omega_2$ falsifies no conditional outside of $\Delta_i$. Because $\omega^i\omega^3=\omega_2^i\omega_2^3$, we have that $\omega$ and $\omega_2$ falsify exactly the same conditionals in $\Delta_i$. 
	\FARBE{Thus $\omega_2$ falsifies only a subset of conditionals that $\omega$ falsifies and therefore 
		\[
		\sum_{\substack{k \neq j \\ \omega_2 \models A_k \ol{B_k} \\ (B_k|A_k)\in\Delta}} \kappaiminus{k}
		\leq
		\sum_{\substack{k \neq j \\ \omega \models A_k \ol{B_k} \\ (B_k|A_k)\in\Delta}} \kappaiminus{k}.
		\]
		Because \eqref{eq_ccore_proof1} utilizes only the minimal worlds and $\omega_2$ only falsifies conditionals in $\Delta_i$, the transformation rules R1 and R2 can be used to transform \eqref{eq_ccore_proof1} into}
	\begin{equation}
		\label{eq_ccore_proof2}
		\kappaiminus{j}^i  >  
		\min_{\omega \models A_j B_j}
		\sum_{\substack{k \neq j \\ \omega \models A_k \ol{B_k} \\ (B_k|A_k)\in\Delta_i}} \kappaiminus{k} 
		- 
		\min_{\omega \models A_j \notB_j}
		\sum_{\substack{k \neq j \\ \omega \models A_k \ol{B_k}\\ (B_k|A_k)\in\Delta_i}} \kappaiminus{k} 
	\end{equation}
	which is the definition of $C_j^i$.
	Therefore $C_j$ %
	can be transformed into %
	$C_j^i$ 
	by applying transformation rules R1 and R2 to each constraint $C_j$. 
	Thus, $\CR(\Delta)\marg_{\Delta_i}$
	can be transformed into 
	$\CR(\Delta_i)$ by applying R1 and R2. 
	\FARBE{Hence, because the set of transformation rules $\{R_1,\dots R_6\}$ is confluent and terminating,
	this means that
	$\hat{\CR}(\Delta)\marg_{\Delta_i}$ and 
	$\hat{\CR}(\Delta_i)$ coincide and also
	that $\CR^+(\Delta)\marg_{\Delta_i}$ and $\CR^+(\Delta_i)$ coincide.}
	
	\FARBE{Because the positive parts of the relevant constraint systems after applying transformation rules $\{R_1,\dots R_6\}$} are the same, $\sigma^{mc}$ assigns the same value to $\eta_j$ and to $\eta_j^i$ (cf. Definition \ref{def:minimal_core_c-representation}) and thus $\sigma^{mc}(\Delta)\marg_{\Delta_{i}}=\sigma^{mc}(\Delta_i)$. Thus, $\sigma^{mc}$ satisfies \ipCSPG.
\end{proof}
\FARBE{Utilizing this selection strategy we can now show that c-core closure fully complies with generalized conditional syntax splitting.}
\begin{proposition}
	c-Core closure satisfies \cRelG\ and \cIndG\ and thus \cSynSplitG.
\end{proposition}
\begin{proof}
	The proposition follows immediately from Propositions~\ref{prop_selstrat_csynsplit} and \ref{prop_coreselstrat_ipcspg}.
\end{proof}
}
Thus we have shown that c-core closure is an example of an inference operator based on a single c-representation that satisfies \ipCSPG\ and thus fully complies with our generalized version of conditional syntax splitting. Next, we will look at an inference operator taking not a single, but all c-representations of a \BB\ into account.

}

\subsection{c-Inference}\label{ssec_cinf}
\emph{c-Inference} was introduced in \citep{BeierleEichhornKernIsberner2016FoIKS,BeierleEichhornKernIsbernerKutsch2018AMAI} as the skeptical inference relation obtained by taking all \hbox{c-representations} of a belief base \Rdelta\ into account.

	\begin{definition}[c-inference, \boldmath{$\skCinfwrt{\Rdelta}$}, \citep{BeierleEichhornKernIsberner2016FoIKS}]
		\label{def_consequence_relations}
		Let \(\Rdelta\) be a belief base and let $A$, $B$ be formulas.
		$B$ is \emph{a (skeptical) c-inference from $A$ in the context of \Rdelta}, denoted by $A\skCinfwrt{\Rdelta} B$, iff $A\nmableit_{\!\!\kappa\,} B$ holds for all c-representations \(\kappa\) of \(\Rdelta\), yielding the inductive inference operator
		\[
		\indreasskcinf:  \Delta \mapsto \skCinfwrt{\Rdelta}
		\]
	\end{definition}
	
	\FARBE{Before proving that c-inference satisfies conditional syntax splitting, we 
}
	\FARBE{recall the following: Consider a safe conditional syntax splitting of $\Delta$ into $\Delta_1$ and $\Delta_2$,
	\FARBE{and a c-representation 
		$\induzierteOCF{\kappaiminusVektor}$
		determined by a solution vector $\kappaiminusVektor\in\solutionsRof{\Delta}$  together with its projections
		$\induzierteOCF{\kappaiminusVektorsubj{1}}$
		and
		$\induzierteOCF{\kappaiminusVektorsubj{2}}$
		to $\Delta_1$
		and $\Delta_2$, respectively.
		Then the rank of any formula $F_i$ over the language
		$\cL(\Sigma_i\cup\Sigma_3)$
		of $\Delta_i$
		under the projection
		$\induzierteOCF{\kappaiminusVektorsubj{i}}$
		coincides with the rank of the formula rank determined by
		$\induzierteOCF{\kappaiminusVektor}$ \citep{BeierleSpiegelHaldimannWilhelmHeyninckKernIsberner2024KR}.
		}\FARBE{We extend this result to} generalized safe conditional syntax splittings in the next proposition.}
\begin{propositionrep}%
	\label{lem_split_crep_formula_null_cond}
	For any \FARBE{\genSafeCondSynSplitGen}, for all 
	$\kappaiminusVektor \in \solutionsRof{\Rdelta}$, \FARBE{and for $i \in \{1,2\}$, $F_i \in \cL(\Sigma_i\cup\Sigma_{3})$,}
	we have
	$\induzierteOCF{\kappaiminusVektor}(F_i) = \induzierteOCF{\kappaiminusVektorsubj{i}}(F_i)$.
\end{propositionrep}
\begin{proof}
	\FARBE{The proof is obtained by adapting the corresponding proof for safe splittings \citep[Proposition 29]{BeierleSpiegelHaldimannWilhelmHeyninckKernIsberner2024KR} by observing that the prerequisites of the steps in the proof are also satisfied by generalized safe conditional syntax splittings.}
\end{proof}
\FARBE{A related proposition for safe splittings additionally states that, for $\j\in\{1,2\}$, $i\neq\j$, \FARBE{it holds that} $\induzierteOCF{\kappaiminusVektorsubj{\j}}(F_i)=0$, i.e., formulas defined over the language of one subbase get assigned the rank 0 in models of the other subbase~\citep{BeierleSpiegelHaldimannWilhelmHeyninckKernIsberner2024KR}. However, this result cannot be extended to \genSafe\ splittings because conditionals in $\DeltaThree$ can be falsified by $F_i$ and thus it is possible that $\induzierteOCF{\kappaiminusVektorsubj{\j}}(F_i)>0$. }

\FARBE{%
	\FARBE{Next we can show} that for every \genSafe\ conditional syntax splitting and every solution vector for $\DeltaI$, we can actually find matching solution vectors for $\DeltaJ$ and $\DeltaThree$.}
\FARBE{
	\begin{propositionrep}
		\label{lemm_thereisd3}
		Let $\Delta$ be a 
		\BB\ with \genSafeCondSynSplitGen. Then \FARBE{for $i \in \{1,2\}$, and} for every $\kappaiminusVektorsubj{i}\in\solutionsRof{\DeltaI}$ there are  $\kappaiminusVektorsubj{i'}\in\solutionsRof{\DeltaJ}$ and $\kappaiminusVektorsubj{3}\in\solutionsRof{\DeltaThree}$, such that $\kappaiminusVektorsubj{i}|_{\DeltaThree}=\kappaiminusVektorsubj{i'}|_{\DeltaThree} = \kappaiminusVektorsubj{3}$.
	\end{propositionrep}
}
	\begin{proof}
		Consider some constraint $\kappaiminus{j}^{i}\in\cspRof{\DeltaThree}$:
		\begin{align}
			\label{eq_proof_lemm_12}
			\kappaiminus{j}^{i}  >  
				\min_{\omega \models A_j B_j}
				\sum_{\substack{k \neq j \\ \omega \models A_k \ol{B_k} \\ (B_k|A_k)\in\DeltaI}} \kappaiminus{k} 
			- 
				\min_{\omega \models A_j \notB_j}
				\sum_{\substack{k \neq j \\ \omega \models A_k \ol{B_k} \\ (B_k|A_k)\in\DeltaI}} \kappaiminus{k} 
		\end{align} 
		because $\DeltaThree\in (\cL(\Sigma_3)|\cL(\Sigma_3))$, the only relevant part for the verification of $(B_j|A_j)$ is $\omega^3$ for any $\omega$. With the general safety property this means that there is an extension $\omega^3\omega^i\omegaJ$ such that $\omega^3\omega^i\omegaJ$ falsifies no conditional in $\Delta\setminus \DeltaThree$. Clearly this must then also hold for the world minimizing this expression. This means, that the solution of $\cspRof{\DeltaThree}$ is independent of the conditionals in $\DeltaIWoThree$ or $\DeltaJWoThree$. Therefore we have $\solutionsRof{\DeltaI}_{|\DeltaThree}=\solutionsRof{\DeltaJ}_{|\DeltaThree}=\solutionsRof{\DeltaThree}$. Since $\Delta$ is (strongly) consistent we can then always find $\kappaiminusVektorsubj{i'}\in\solutionsRof{\DeltaJ}$ and $\kappaiminusVektorsubj{3}\in\solutionsRof{\DeltaThree}$, such that $\kappaiminusVektorsubj{i}|_{\DeltaThree}=\kappaiminusVektorsubj{i'}|_{\DeltaThree} = \kappaiminusVektorsubj{3}$ by choosing $\kappaiminusVektorsubj{3}=\kappaiminusVektorsubj{i}|_{\DeltaThree}$ and $\kappaiminusVektorsubj{i'}$ such that$\kappaiminusVektorsubj{i'}|_{\DeltaThree}=\kappaiminusVektorsubj{i}|_{\DeltaThree}$. 
	\end{proof}
With Propositions \ref{prop_solprop}, \ref{lem_split_crep_formula_null_cond}, and \ref{lemm_thereisd3} we can \FARBE{show:}
\FARBE{
	\begin{propositionrep}
		c-Inference satisfies 
		\FARBE{\cRelG\ and  \cIndG\ and thus
		\cSynSplitG}. 
	\end{propositionrep}
}
\begin{proof}
	Let \safeCondSynSplitGen. W.l.o.g. assume $A,B\in \cL(\Sigma_1), D\in\cL(\Sigma_{2}), \dot{B}\in\{B,\notB\}$ and assume $E\in\cL(\Sigma_3)$ is a complete conjunction with \(\FARBE{DE} \not\equiv \bot\). \FARBE{We show that c-inference satisfies both \cRelG\ \textbf{(I)} and \cIndG\ \textbf{(II)}}.
	
	\textbf{(I)} To prove that $\indreasskcinf$ satisfies \cRelG\ we need to show that $AE\skCinfwrt{\Rdelta} B$ iff $AE \skCinfwrt{\Rdelta_1} B$.
	By applying the definition of $\skCinfwrt{\Rdelta}$ we obtain:
	\begin{align}
		\forall \kappaiminusVektor\in\solutionsR:
		\induzierteOCF{\kappaiminusVektor}(ABE) <  \induzierteOCF{\kappaiminusVektor}(A\ol{B}E)\label{eq_cinf_proof_1}\\
		\text{ iff }
		\forall
		\kappaiminusVektor^1\in\solutionsRof{\DeltaOne}:
		\induzierteOCF{\kappaiminusVektorsubj{1}}(ABE) <  \induzierteOCF{\kappaiminusVektorsubj{1}}(A\ol{B}E)\label{eq_cinf_proof_2}
	\end{align}
	
	\textbf{Direction \(\Rightarrow\):}
	\FARBE{Due to Proposition~\ref{prop_solprop} and \Cref{lemm_thereisd3}} we have that every $\kappaiminusVektor^1\in\solutionsRof{\DeltaI}$ has an extension $\kappaiminusVektorsubj{1}$ such that $(\kappaiminusVektorsubj{1},\kappaiminusVektorsubj{2}|_{\DeltaJWoThree}) \in \solutionsR$. Due to \eqref{eq_cinf_proof_1} we know that $\induzierteOCF{\kappaiminusVektor}(ABE) <  \induzierteOCF{\kappaiminusVektor}(A\ol{B}E)$ holds for $\kappaiminusVektor=(\kappaiminusVektorsubj{1},\kappaiminusVektorsubj{2}|_{\DeltaJWoThree})$. We need to show that  $\induzierteOCF{\kappaiminusVektorsubj{1}}(ABE) <  \induzierteOCF{\kappaiminusVektorsubj{1}}(A\ol{B}E)$.
	With Lemma \ref{lem_split_crep_formula_null_cond} this follows directly because $\induzierteOCF{\kappaiminusVektor}(A\dot{B}E)=\induzierteOCF{\kappaiminusVektorsubj{1}}(A\dot{B}E)$ since $A,B\in \cL(\Sigma_{1})$, $E\in \cL(\Sigma_{3})$ and $\DeltaI\subseteq(\cL(\Sigma_{1}\cup\Sigma_{3})|\cL(\Sigma_{1}\cup\Sigma_{3}))$.
	
	\textbf{Direction \(\Leftarrow\):}
	With Proposition~\ref{prop_solprop} we have that every $\kappaiminusVektor\in\solutionsR$ can be split into $(\kappaiminusVektorsubj{1},\kappaiminusVektorsubj{2}|_{\DeltaJWoThree}) \in \solutionsR$. Due to \eqref{eq_cinf_proof_2} we know that $\induzierteOCF{\kappaiminusVektorsubj{1}}(ABE) <  \induzierteOCF{\kappaiminusVektorsubj{1}}(A\ol{B}E)$ holds. Then with Lemma~\ref{lem_split_crep_formula_null_cond} it follows directly that $\induzierteOCF{\kappaiminusVektorsubj{1}}(ABE) <  \induzierteOCF{\kappaiminusVektorsubj{1}}(A\ol{B}E)$ as above.
	
	\textbf{(II)} Next we prove that $\indreasskcinf$ satisfies {\bf \cindep}. We need to show $AE\skCinfwrt{\Rdelta} B$ iff $ADE\skCinfwrt{\Rdelta} B$. By applying the definition of $\skCinfwrt{\Rdelta}$ we obtain:
	\begin{align}
		\forall \kappaiminusVektor\in\solutionsR:
		\induzierteOCF{\kappaiminusVektor}(ABE) <  \induzierteOCF{\kappaiminusVektor}(A\ol{B}E)\label{eq_cinf_proof_3}\\
		\text{ iff }
		\forall
		\kappaiminusVektor\in\solutionsRof{\Delta}:
		\induzierteOCF{\kappaiminusVektor}(ABDE) <  \induzierteOCF{\kappaiminusVektor}(A\ol{B}DE)\label{eq_cinf_proof_4}
	\end{align}
	
	\textbf{Direction \(\Rightarrow\):} Due to Proposition~\ref{prop_crep_condkind} we know that $\Ind{\induzierteOCF{\kappaiminusVektor}}{\Sigma_1}{\Sigma_2}{\Sigma_3}$. With Lemma~\ref{lemma_ocf_condint_formula} we then have $\induzierteOCF{\kappaiminusVektor}(A\dot{B}DE) = \induzierteOCF{\kappaiminusVektor}(A\dot{B}E) + \induzierteOCF{\kappaiminusVektor}(DE) -\induzierteOCF{\kappaiminusVektor}(E)$. With \eqref{eq_cinf_proof_3} it is clear that \eqref{eq_cinf_proof_4} must also hold.
	
	\textbf{Direction \(\Leftarrow\):} Due to Proposition~\ref{prop_crep_condkind} we know that $\Ind{\induzierteOCF{\kappaiminusVektor}}{\Sigma_1}{\Sigma_2}{\Sigma_3}$. With Lemma~\ref{lemma_ocf_condint_formula} we then have $\induzierteOCF{\kappaiminusVektor}(A\dot{B}DE) = \induzierteOCF{\kappaiminusVektor}(A\dot{B}E) + \induzierteOCF{\kappaiminusVektor}(DE) -\induzierteOCF{\kappaiminusVektor}(E)$. This is equivalent to 
	$\induzierteOCF{\kappaiminusVektor}(A\dot{B}E) =
	\induzierteOCF{\kappaiminusVektor}(A\dot{B}DE) - \induzierteOCF{\kappaiminusVektor}(DE) +\induzierteOCF{\kappaiminusVektor}(E)$. With \eqref{eq_cinf_proof_4} it is clear that \eqref{eq_cinf_proof_3} must also hold. 
\end{proof}
\FARBE{Thus also the inference taking all c-representations into account fully complies with \cSynSplitG.}

\section{(CSynSplit\textsuperscript{g}) properly strengthens (CSynSplit)}
\label{sec_csynsplitg_stronger}
\FARBE{While \cSynSplit\ takes into account all safe conditional syntax splittings of a \BB, \cSynSplitG\ takes into account all generalized safe splittings. Because every safe splitting is generalized safe, but not vice versa (cf. Proposition~\ref{prop_relationship_safeties}), this means that \cSynSplitG\ is harder to satisfy than \cSynSplit.
In this section we formalize this observation by providing a proof showing that \cSynSplitG\ implies \cSynSplit\ but not the other way around.}
First, we first introduce the following lemma, stating that System~Z complies with conditional independence when restricted to \FARBE{simple (non-genuine)} and safe conditional syntax splittings. \FARBE{We will afterwards exploit this lemma to show that there are inductive inference operators satisfying \cSynSplit\ but not \cSynSplitG.}
\begin{lemma}
	\label{lemm_sysZ_simplesafe}
	Let ${\bf C^z}$ be the System Z induced inductive inference operator and $\Delta$ a \BB. Then, for every %
	\FARBE{simple,}
	safe conditional syntax splitting \genSafeCondSynSplitGen, ${\bf C^z}$ satisfies, for all $A,B\in \cL(\Sigma_i)$, $D\in\cL(\Sigma_{i'})$, with $i,i'\in\{1,2\}, i\neq i'$, and a full conjunction $E\in\cL(\Sigma_3)$ with $DE\not\equiv\bot$, that
	\[
	AE\nmableitDS{\Delta}{z} B \quad \mbox{ iff }\quad ADE\nmableitDS{\Delta}{z} B.
	\]
\end{lemma}
\begin{proof}
	W.l.o.g. assume that $i=1,i'=2$, the other case is analogous.
	Because the splitting is not \echtesSplitting, we have that either $\Delta_1\subseteq\Delta_2$ or $\Delta_2\subseteq\Delta_1$. Because the splitting is also safe, we have that $\Delta_1\cap\Delta_2=\emptyset$. Thus, either $\Delta_1=\emptyset$ or $\Delta_2=\emptyset$.
	
	First assume $\Delta_1=\emptyset$. Then $\Delta_2=\Delta$. 
	We deal with the border cases first. Assume either $A\equiv\bot$, or $B\equiv\bot$. If $A\equiv\bot$, then $AE\equiv ADE\equiv \bot$ and the equation holds. If $B\equiv\bot$, then we need to show, that $AE\nmableitDS{\Delta}{z}\bot\text{ iff }AED\nmableitDS{\Delta}{z}\bot.$. Both sides of the iff are false, unless $AE\equiv\bot$ or $AED\equiv\bot$. Neither $E\equiv\bot$ nor $D\equiv\bot$ are allowed as per our assumption and if $A\equiv\bot$ we obtain the first case.
	
	So assume $A\not\equiv\bot$ and $B\not\equiv\bot$. Because $\Delta_1\subseteq\Delta_2$, there exists no signature element in $\Sigma_1$ that appears in any conditional in $\Delta$, because the existence of such an element would mean, that there is some conditional $(B|A)\in\Delta$ with $(B|A)\in\Delta_1$ but $(B|A)\notin\Delta_2$, contrary to our assumption. This means that no formula $F_1\in\cL(\Sigma_1), F_1\not\equiv\bot$ can cause the falsification of any conditional in $\Delta$, i.e., $\kappa_\Delta^z(G)=\kappa_\Delta^z(GF_1)$ for any $G\in\cL(\Sigma)$. With this we have, for $A,B\not\equiv\bot$, $\kappa_\Delta^z(AEB)=\kappa_\Delta^z(AE\ol{B})=\kappa_\Delta^z(E)$ and $\kappa_\Delta^z(AEDB)=\kappa_\Delta^z(AED\ol{B})=\kappa_\Delta^z(ED)$ and thus $AE\nmableitDS{\Delta}{z} B\text{ iff }ADE\nmableitDS{\Delta}{z} B$.
	
	Next assume $\Delta_2=\emptyset$. Then $\Delta_1=\Delta$. The case $D\equiv\bot$ is not possible as per our assumption. Using the same arguments as above, we obtain for any $F_2\in\cL(\Sigma_2)$ and any $G\in\cL(\Sigma)$, that $\kappa_{\Delta}^z(G)=\kappa_{\Delta}^z(GF_2)$. Thus we have $\kappa_{\Delta}^z(AEB)=\kappa_{\Delta}^z(ADEB)$ and $\kappa_{\Delta}^z(AE\ol{B})=\kappa_{\Delta}^z(ADE\ol{B})$ and therefore $AE\nmableitDS{\Delta}{z} B\text{ iff }ADE\nmableitDS{\Delta}{z} B$.
\end{proof}

\FARBE{\FARBE{Now we show that} the postulate \cSynSplitG\ %
	is indeed harder to satisfy than \cSynSplit.}

\begin{propositionrep}
	\label{prop_csynsplitgG_csynsplit}
	The following \FARBE{relationships} hold:
	\begin{enumerate}
		\item \cSynSplitG\ implies \cSynSplit.
		\item \cSynSplit\ does not imply \cSynSplitG.
	\end{enumerate}
\end{propositionrep}
\begin{proof}
	1.: \FARBE{This follows immediately from the fact that every safe conditional syntax splitting is also \genSafe\ (cf. \Cref{prop_relationship_safeties}).}
	
	2.:We construct an inductive inference Operator ${\bf C^{zw}}$ in the following manner:
	\[
	{\bf C^{zw}}(\Delta) =
	\begin{cases}
		\wableitDelta, 
		\text{ if $\Delta$ has a \echtesSplitting\ safe conditional syntax splitting}\\
		\nmableit_{\!\!\Delta}^{\!\!z}, \text{ otherwise}
	\end{cases}
	\]
	
	${\bf C^{zw}}$ \FARBE{satisfies \cRel}, because both ${\bf C^{w}}$ and ${\bf C^{z}}$ satisfy \cRel. Furthermore, \FARBE{whenever ${\bf C^{zw}}(\Delta)={\bf C^{w}}(\Delta)$, then ${\bf C^{zw}}$ satisfies \cInd\ because} ${\bf C^{w}}$ satisfies \cInd. 
	\FARBE{Whenever ${\bf C^{zw}}(\Delta)={\bf C^{z}}(\Delta)$, then all safe splittings of $\Delta$ are simple splittings.}
	Therefore, with \Cref{lemm_sysZ_simplesafe}, \cInd\ is also satisfied \FARBE{whenever ${\bf C^{zw}}(\Delta)={\bf C^{z}}(\Delta)$}, meaning ${\bf C^{zw}}$ satisfies both \cRel\ and \cInd.
	
	We now show that ${\bf C^{zw}}$ violates \cIndG\ by constructing a \BB\ with no \echtesSplitting\ safe conditional syntax splitting, but a \genSafe\ conditional syntax splitting.	
	\FARBE{Consider the signature $\Sigma=\{b,p,f,w,k\}$ meaning that an entity is a \underline{b}ird, is a \underline{p}enguin, \underline{f}lies, has \underline{w}ings, and is a \underline{k}iwi.	
	The \BB\ 
	\[
	\DeltaKiwi=\{(f|b), (\ol{f}|p), (b|p), (\ol{f}|k),(b|k), \FARBE{(w|b), (\ol{w}|k)}\}
	\]
	makes use of this signature.}
\FARBE{The \BB\ $\DeltaKiwi$	
does not have a \echtesSplitting\ safe conditional syntax splitting; 
the list of all conditional syntax splitting of $\DeltaKiwi$ can be found in the appendix. Thus, ${\bf C^{zw}}(\DeltaKiwi)$ coincides with ${\bf C^{z}}(\DeltaKiwi)$.} However, $\DeltaKiwi$ has a \genSafe\ conditional syntax splitting, given by
\begin{align*}
	\begin{split}
		\DeltaKiwi=\{(f|b), (\ol{f}|p), (b|p)\}
		\bigcup_{\{p\},\{k,w\}}^{\sf gs}\{(f|b), (\ol{f}|k),(b|k),(\ol{w}|k),(w|b)\}\mid \{f,b\}.
	\end{split}
\end{align*}
Regarding this splitting however, we have that $bf\nmableitDS{\DeltaKiwi}{z} w$ while $pbf\notnmableitDS{\DeltaKiwi}{z} w$. Since $b,f\in\SigmaThree, p\in\Sigma_1$ and $w\in\Sigma_2$ this is a violation of \cIndG. Thus ${\bf C^{zw}}$ does not satisfy \cSynSplitG.
\end{proof}

\FARBE{Thus, Proposition~\ref{prop_csynsplitgG_csynsplit} shows that, because \cSynSplitG\ covers a broader notion of safety and thus more splittings, \cSynSplitG\ implies \cSynSplit\ but not the other way around.} 
\section{Conclusions and Future Work}
\label{sec_conclusions}
\FARBE{
In this article we 
generalized the notion of safety for conditional syntax splittings, \FARBE{allowing %
the subbases to} share non-trivial conditionals.
\FARBE{This is achieved by introducing a more relaxed notion of safety,
significantly broadening the application scope of the beneficial splitting techniques.}
Moreover, we
identified genuine splittings as the subclass of meaningful %
conditional syntax splittings, separating them from the class of simple splittings which provide no advantage for inductive inference.

\FARBE{Thus, we have made two major steps towards utilizing conditional syntax splitting postulates for inductive inference applications. First, we have significantly broadened the applicability of conditional syntax splitting postulates by adapting them to a more relaxed notion of safety, allowing them to be applied to significantly more splittings. This includes making them applicable to belief bases where previous postulates were not able to be meaningfully applied at all (cf. Example~\ref{exa_contrasafety2}). Second, we have classified splittings \FARBE{beneficial for} inductive inference as genuine splittings, filtering out the large class of simple splittings in the process (cf. Example~\ref{exa_gensafety}). Thus we were able to identify those splittings where conditional syntax splitting postulates can be meaningfully applied as exactly the genuine, generalized safe splittings.}

We introduced adapted conditional syntax splitting postulates to fit with our generalized notion of conditional syntax splitting.
\FARBE{ %
The new postulate \cSynSplitG\
covers more splittings than 
\cSynSplit\
and is thus more \FARBE{relevant}, but harder to satisfy, i.e., there exist inductive inference operators complying with conditional syntax splitting but not with generalized conditional syntax splitting.}
We showed that lexicographic inference, System~W, inductive inference with a single c-representation based on an adequate selection strategy, c-core closure, and c-inference all fully comply with generalized conditional syntax splitting. \FARBE{While System~Z fails to satisfy generalized conditional syntax splitting, we showed that it complies with generalized conditional relevance.}

In future work we will study the exact relationship of our approach to syntactic contextual filtering \citep{dupin2024form} and to propositional forgetting \citep{LangLiberatoreMarquis2003local,SauerwaldKernIsbernerBeckerBeierle2022SUM,SauerwaldBeierleKernIsberner2024FoIKS}. We will exploit the beneficial properties of splitting techniques \FARBE{for} implementations of inductive inference
\citep{BeierleHaldimannSaninSchwarzerSpangSpiegelvonBerg2024SUM,BeierleHaldimannSaninSpangSpiegelvonBerg2025JELIA},
\FARBE{and we will adapt}
the concepts shown here to include also \BBs\ that satisfy a weaker notion of consistency (cf. \citep{HaldimannBeierleKernIsbernerMeyer2023FLAIRSproceedings,HaldimannBeierleKernIsberner2024FoIKS}).}

\clearpage

\section*{Acknowledgments}
\FARBE{

	This work was supported by the Deutsche Forschungsgemeinschaft (DFG, German Research Foundation) - 512363537, grant BE~1700/12-1 awarded to Christoph Beierle. Lars-Phillip Spiegel was supported by this grant.
	
	\FARBE{Jonas Haldimann's work was supported in part by the National Research Foundation of South Africa (REFERENCE NO: SAI240823262612).}

	}

\newcommand{\verzeichnisBibtex}{\string~/BibTeXReferencesSVNlink}

 \bibliographystyle{cas-model2-names} 

\addcontentsline{toc}{section}{References}
\bibliography{references.bib}

\clearpage
\appendix
\section{List of all conditional syntax splitting of belief base $\DeltaRain$ from Example~\ref{exa_simplesplits}.}
All genuine, generalized safe splittings are marked with boxes.
\begin{align*}
		&\safeCondSynSplitMacro{\DeltaRain}{\DeltaRain}{\emptyset}{\Sigma}{\emptyset}{\emptyset}\\
		&\safeCondSynSplitMacro{\DeltaRain}{\DeltaRain}{\emptyset}{\{s,r,o,u\}}{\emptyset}{\{b\}}\\
		&\safeCondSynSplitMacro{\DeltaRain}{\DeltaRain}{\emptyset}{\{b,r,o,u\}}{\emptyset}{\{s\}}\\
		&\safeCondSynSplitMacro{\DeltaRain}{\DeltaRain}{\emptyset}{\{b,s,o,u\}}{\emptyset}{\{r\}}\\
		&\safeCondSynSplitMacro{\DeltaRain}{\DeltaRain}{\emptyset}{\{b,s,r,u\}}{\emptyset}{\{o\}}\\
		&\safeCondSynSplitMacro{\DeltaRain}{\DeltaRain}{\emptyset}{\{b,s,r,o\}}{\emptyset}{\{u\}}\\
		&\safeCondSynSplitMacro{\DeltaRain}{\DeltaRain}{\emptyset}{\{r,o,u\}}{\emptyset}{\{b,s\}}\\
		&\safeCondSynSplitMacro{\DeltaRain}{\DeltaRain}{\emptyset}{\{s,o,u\}}{\emptyset}{\{b,r\}}\\
		&\safeCondSynSplitMacro{\DeltaRain}{\DeltaRain}{\emptyset}{\{s,r,u\}}{\emptyset}{\{b,o\}}\\
		&\safeCondSynSplitMacro{\DeltaRain}{\DeltaRain}{\emptyset}{\{s,r,o\}}{\emptyset}{\{b,u\}}\\
		&\boxed{\genSafeCondSynSplitMacro{\DeltaRain}{\{(\ol{s}|r),(\ol{r}|s),(b|sr)\}}{\{(\ol{s}|r),(\ol{r}|s),(o|s\ol{r}),(\ol{o}|r),(u|ro)\}}{\{b\}}{\{o,u\}}{\{s,r\}}}\\
		&\genSafeCondSynSplitMacro{\DeltaRain}{\DeltaRain}{\{(\ol{s}|r),(\ol{r}|s)\}}{\{b,o,u\}}{\emptyset}{\{s,r\}}\\
		&\CondSynSplitMacro{\DeltaRain}{\DeltaRain}{\emptyset}{\{b,r,u\}}{\emptyset}{\{s,o\}}\\
		&\safeCondSynSplitMacro{\DeltaRain}{\DeltaRain}{\emptyset}{\{b,r,o\}}{\emptyset}{\{s,u\}}\\
		&\boxed{\genSafeCondSynSplitMacro{\DeltaRain}{\{(\ol{s}|r),(\ol{r}|s),(b|sr),(o|s\ol{r}),(\ol{o}|r)\}}{\{(\ol{o}|r),(u|ro)\}}{\{b,s\}}{\{u\}}{\{r,o\}}}\\
		&\genSafeCondSynSplitMacro{\DeltaRain}{\DeltaRain}{\{(\ol{o}|r)\}}{\{b,s,u\}}{\emptyset}{\{r,o\}}\\
		&\safeCondSynSplitMacro{\DeltaRain}{\DeltaRain}{\emptyset}{\{b,s,o\}}{\emptyset}{\{r,u\}}\\
		&\safeCondSynSplitMacro{\DeltaRain}{\DeltaRain}{\{(b|k)\}}{\{b,s,r\}}{\emptyset}{\{o,u\}}\\
		&\genSafeCondSynSplitMacro{\DeltaRain}{\DeltaRain}{\{(\ol{s}|r),(\ol{r}|s),(b|sr)\}}{\{o,u\}}{\emptyset}{\{b,s,r\}}\\
		&\CondSynSplitMacro{\DeltaRain}{\DeltaRain}{\emptyset}{\{r,u\}}{\emptyset}{\{b,s,o\}}\\
		&\safeCondSynSplitMacro{\DeltaRain}{\DeltaRain}{\emptyset}{\{r,o\}}{\emptyset}{\{b,s,u\}}\\
		&\boxed{\genSafeCondSynSplitMacro{\DeltaRain}{\{(\ol{s}|r),(\ol{r}|s),(b|sr),(o|s\ol{r}),(\ol{o}|r)\}}{\{(\ol{o}|r),(u|ro)\}}{\{s\}}{\{u\}}{\{b,r,o\}}}\\
		&\genSafeCondSynSplitMacro{\DeltaRain}{\DeltaRain}{\{(\ol{o}|r)\}}{\{s,u\}}{\emptyset}{\{b,r,o\}}\\
		&\safeCondSynSplitMacro{\DeltaRain}{\DeltaRain}{\emptyset\}}{\{s,o\}}{\emptyset}{\{b,r,u\}}\\
		&\safeCondSynSplitMacro{\DeltaRain}{\DeltaRain}{\emptyset}{\{s,r\}}{\emptyset}{\{b,o,u\}}\\
		&\boxed{\genSafeCondSynSplitMacro{\DeltaRain}{\{(\ol{s}|r),(\ol{r}|s),(b|sr),(o|s\ol{r}),(\ol{o}|r)\}}{\{(\ol{s}|r),(\ol{r}|s),(o|s\ol{r}),(\ol{o}|r),(u|ro)\}}{\{b\}}{\{u\}}{\{s,r,o\}}}\\
		&\genSafeCondSynSplitMacro{\DeltaRain}{\DeltaRain}{\{(\ol{s}|r),(\ol{r}|s),(o|s\ol{r}),(\ol{o}|r)\}}{\{b,u\}}{\emptyset}{\{s,r,o\}}\\
		&\boxed{\genSafeCondSynSplitMacro{\DeltaRain}{\{(\ol{s}|r),(\ol{r}|s),(b|sr)\}}{\{(\ol{s}|r),(\ol{r}|s),(o|s\ol{r}),(\ol{o}|r),(u|ro)\}}{\{b\}}{\{o\}}{\{s,r,u\}}}\\
		&\genSafeCondSynSplitMacro{\DeltaRain}{\DeltaRain}{\{(\ol{s}|r),(\ol{r}|s)\}}{\{b,o\}}{\emptyset}{\{s,r,u\}}\\
		&\CondSynSplitMacro{\DeltaRain}{\DeltaRain}{\emptyset}{\{b,r\}}{\emptyset}{\{s,o,u\}}\\
		&\genSafeCondSynSplitMacro{\DeltaRain}{\DeltaRain}{\{(\ol{o}|r),(u|ro)\}}{\{b,s\}}{\emptyset}{\{r,o,u\}}\\
		&\genSafeCondSynSplitMacro{\DeltaRain}{\DeltaRain}{\{(\ol{s}|r),(\ol{r}|s),(b|sr),(o|s\ol{r}),(\ol{o}|r)\}}{\{u\}}{\emptyset}{\{b,s,r,o\}}\\
		&\genSafeCondSynSplitMacro{\DeltaRain}{\DeltaRain}{\{(\ol{s}|r),(\ol{r}|s),(b|sr)\}}{\{o\}}{\emptyset}{\{b,s,r,u\}}\\
		&\CondSynSplitMacro{\DeltaRain}{\DeltaRain}{\emptyset}{\{r\}}{\emptyset}{\{b,s,o,u\}}\\
		&\genSafeCondSynSplitMacro{\DeltaRain}{\DeltaRain}{\{(\ol{o}|r),(u|or)\}}{\{s\}}{\emptyset}{\{b,r,o,u\}}\\
		&\genSafeCondSynSplitMacro{\DeltaRain}{\DeltaRain}{\{(\ol{s}|r),(\ol{r}|s),(o|s\ol{r}),(\ol{o}|r),(u|ro)\}}{\{b\}}{\emptyset}{\{s,r,o,u\}}\\
		&\genSafeCondSynSplitMacro{\DeltaRain}{\DeltaRain}{\DeltaRain}{\emptyset}{\emptyset}{\Sigma}
\end{align*}

\section{List of all conditional syntax splittings of belief base $\DeltaKiwi$ from the proof of Proposition~\ref{prop_csynsplitgG_csynsplit}.}
All genuine, generalized safe splittings are marked with boxes.		
\begin{align*}
	&\safeCondSynSplitMacro{\DeltaKiwi}{\DeltaKiwi}{\emptyset}{\Sigma}{\emptyset}{\emptyset}\\
	&\genSafeCondSynSplitMacro{\DeltaKiwi}{\DeltaKiwi}{\DeltaKiwi}{\emptyset}{\emptyset}{\Sigma}\\
	&\CondSynSplitMacro{\DeltaKiwi}{\DeltaKiwi}{\emptyset}{\{f,k,b,w\}}{\emptyset}{\{p\}}\\
	&\safeCondSynSplitMacro{\DeltaKiwi}{\DeltaKiwi}{\emptyset}{\{p,f,k,w\}}{\emptyset}{\{b\}}\\
	&\safeCondSynSplitMacro{\DeltaKiwi}{\DeltaKiwi}{\emptyset}{\{p,k,b,w\}}{\emptyset}{\{f\}}\\
	&\CondSynSplitMacro{\DeltaKiwi}{\DeltaKiwi}{\emptyset}{\{p,f,b,w\}}{\emptyset}{\{k\}}\\
	&\safeCondSynSplitMacro{\DeltaKiwi}{\DeltaKiwi}{\emptyset}{\{p,f,k,b\}}{\emptyset}{\{w\}}\\
	&\CondSynSplitMacro{\DeltaKiwi}{\DeltaKiwi}{\{(\ol{f}|p)\}}{\{k,b,w\}}{\emptyset}{\{p,f\}}\\
	&\CondSynSplitMacro{\DeltaKiwi}{\DeltaKiwi}{\emptyset}{\{f,b,w\}}{\emptyset}{\{p,k\}}\\
	&\CondSynSplitMacro{\DeltaKiwi}{\DeltaKiwi}{\{(b|p)\}}{\{f,k,w\}}{\emptyset}{\{p,b\}}\\
	&\CondSynSplitMacro{\DeltaKiwi}{\DeltaKiwi}{\emptyset}{\{f,k,b\}}{\emptyset}{\{p,w\}}\\
	&\CondSynSplitMacro{\DeltaKiwi}{\DeltaKiwi}{\{(\ol{f}|k)\}}{\{p,b,w\}}{\emptyset}{\{f,k\}}\\
	&\CondSynSplitMacro{\DeltaKiwi}{\DeltaKiwi}{\{(\ol{w}|k)\}}{\{p,f,b\}}{\emptyset}{\{k,w\}}\\
	&\genSafeCondSynSplitMacro{\DeltaKiwi}{\DeltaKiwi}{\{(w|b)\}}{\{p,f,k\}}{\emptyset}{\{b,w\}}\\
	&\safeCondSynSplitMacro{\DeltaKiwi}{\DeltaKiwi}{\emptyset}{\{p,k,b\}}{\emptyset}{\{f,w\}}\\
	&\genSafeCondSynSplitMacro{\DeltaKiwi}{\DeltaKiwi}{\{(f|b)\}}{\{p,k,w\}}{\emptyset}{\{f,b\}}\\
	&\CondSynSplitMacro{\DeltaKiwi}{\DeltaKiwi}{\{(b|k)\}}{\{p,f,w\}}{\emptyset}{\{k,b\}}
	\\
	&\boxed{\genSafeCondSynSplitMacro{\DeltaKiwi}{\{(f|b),(\ol{f}|p),(b|p)\}}{\{(f|b),(w|b),(\ol{f}|k),(b|k),(\ol{w}|k)\}}{\{p\}}{\{k,w\}}{\{f,b\}}}\\
	&\CondSynSplitMacro{\DeltaKiwi}{\{(b|k),(f|b),(b|p),(\ol{f}|p),(\ol{f}|k)\}}{\{(b|k),(w|b), (\ol{w}|k)\}}{\{p,f\}}{\{w\}}{\{k,b\}}\\
	&\CondSynSplitMacro{\DeltaKiwi}{\DeltaKiwi}{\{(\ol{f}|k),(\ol{f}|p)\}}{\{b,w\}}{\emptyset}{\{p,f,k\}}\\
	&\genSafeCondSynSplitMacro{\DeltaKiwi}{\DeltaKiwi}{\{(f|b),(b|p),(\ol{f}|p)\}}{\{k,w\}}{\emptyset}{\{p,f,b\}}\\
	&\CondSynSplitMacro{\DeltaKiwi}{\DeltaKiwi}{\{(\ol{f}|p)\}}{\{k,b\}}{\emptyset}{\{p,f,w\}}\\
	&\CondSynSplitMacro{\DeltaKiwi}{\DeltaKiwi}{\{(b|p),(b|k)\}}{\{f,w\}}{\emptyset}{\{p,k,b\}}\\
	&\CondSynSplitMacro{\DeltaKiwi}{\{(b|p),(b|k),(f|b),(\ol{f}|p),(\ol{f}|k)\}}{\{(b|p),(b|k),(w|b),(\ol{w}|k)\}}{\{f\}}{\{w\}}{\{p,k,b\}}\\
	&\CondSynSplitMacro{\DeltaKiwi}{\DeltaKiwi}{\{(\ol{w}|k)\}}{\{f,b\}}{\emptyset}{\{p,k,w\}}\\
	&\CondSynSplitMacro{\DeltaKiwi}{\DeltaKiwi}{\{(b|p),(w|b)\}}{\{f,k\}}{\emptyset}{\{p,b,w\}}\\
	&\CondSynSplitMacro{\DeltaKiwi}{\DeltaKiwi}{\{(\ol{f}|k),(b|k),(f|b)\}}{\{p,w\}}{\emptyset}{\{f,k,b\}}\\
	&\CondSynSplitMacro{\DeltaKiwi}{\{(\ol{f}|k),(b|k),(f|b),(b|p),(\ol{f}|p)\}}{\{(\ol{f}|k),(b|k),(f|b),(w|b),(\ol{w}|k)\}}{\{p\}}{\{w\}}{\{f,k,b\}}\\
	&\CondSynSplitMacro{\DeltaKiwi}{\DeltaKiwi}{\{(\ol{f}|k),(\ol{w}|k)\}}{\{p,b\}}{\emptyset}{\{f,k,w\}}\\
	&\boxed{\genSafeCondSynSplitMacro{\DeltaKiwi}{\{(f|b),(w|b),(\ol{f}|p),(b|p)\}}{\{(f|b),(w|b),(\ol{f}|k),(b|k),(\ol{w}|k)\}}{\{p\}}{\{k\}}{\{f,b,w\}}}\\
	&\CondSynSplitMacro{\DeltaKiwi}{\DeltaKiwi}{\{(\ol{w}|k),(b|k),(w|b)\}}{\{p,f\}}{\emptyset}{\{k,b,w\}}\\
	&\CondSynSplitMacro{\DeltaKiwi}{\DeltaKiwi}{\{(\ol{f}|k),(b|k),(f|b),(b|p),(\ol{f}|p)\}}{\{w\}}{\emptyset}{\{p,f,k,b\}}\\
	&\CondSynSplitMacro{\DeltaKiwi}{\DeltaKiwi}{\{(\ol{f}|k),(\ol{w}|k),(\ol{f}|p)\}}{\{b\}}{\emptyset}{\{p,f,k,w\}}\\
	&\genSafeCondSynSplitMacro{\DeltaKiwi}{\DeltaKiwi}{\{(w|b),(f|b),(b|p),(\ol{f}|p)\}}{\{k\}}{\emptyset}{\{p,f,b,w\}}\\
	&\CondSynSplitMacro{\DeltaKiwi}{\DeltaKiwi}{\{(\ol{w}|k),(b|k),(w|b),(b|p)\}}{\{f\}}{\emptyset}{\{p,k,b,w\}}\\
	&\genSafeCondSynSplitMacro{\DeltaKiwi}{\DeltaKiwi}{\{(\ol{f}|k),(b|k),(f|b),(w|b),(\ol{w}|k)\}}{\{p\}}{\emptyset}{\{f,k,b,w\}}
\end{align*}				

\end{document}